\documentclass[11pt]{article}

\usepackage{thm-restate}
\usepackage{vitercik}
\usepackage{hyperref}
\usepackage{subcaption}
\usepackage{tikz}

\newcommand{\score}{\texttt{score}}
\newcommand{\cost}{\texttt{cost}}

\newcommand{\domain}{\mathcal{Z}}
\newcommand{\fclass}{\mathcal{F}}
\newcommand{\tree}{\mathcal{T}}

\newcommand{\node}{Q}

\newcommand{\reals}{\mathbb{R}}
\newcommand{\sgn}{\textnormal{sign}}
\newcommand{\ddeg}{\textnormal{ddeg}}
\newcommand*\circled[1]{\tikz[baseline=(char.base)]{
            \node[shape=circle,draw,inner sep=1pt] (char) {#1};}}
\newcommand{\pdim}{\textnormal{Pdim}}
\newcommand{\loc}{\texttt{localFathom}}
\newcommand{\glo}{\texttt{globalFathom}}
\newcommand{\fathom}{\textnormal{fathom}}
\newcommand{\explore}{\textnormal{explore}}

\makeatletter
\newcommand{\algmargin}{\the\ALG@thistlm}
\makeatother
\newlength{\whilewidth}
\newlength{\ifwidth}
\algdef{SE}[parWHILE]{parWhile}{EndparWhile}[1]
  {\parbox[t]{\dimexpr\linewidth-\algmargin}{%
     \hangindent\whilewidth\strut\algorithmicwhile\ #1\ \algorithmicdo\strut}}{\algorithmicend\ \algorithmicwhile}%
\algdef{SE}[parIF]{parIf}{EndparIf}[1]
  {\parbox[t]{\dimexpr\linewidth-\algmargin}{%
     \hangindent\ifwidth\strut\algorithmicif\ #1\ \algorithmicdo\strut}}{\algorithmicend\ \algorithmicif}%
\algnewcommand{\parState}[1]{\State%
  \parbox[t]{\dimexpr\linewidth-\algmargin}{\strut #1\strut}}
\algtext*{EndparIf}

\begin{document}

\title{Learning to Branch\footnote{Authors' addresses: \texttt{\{ninamf, tdick, sandholm, vitercik\}@cs.cmu.edu.}}}
\author{Maria-Florina Balcan \and Travis Dick \and Tuomas Sandholm \and Ellen Vitercik}
\maketitle

\begin{abstract}
Tree search algorithms, such as branch-and-bound, are the most widely used tools for solving combinatorial and nonconvex problems. For example, they are the foremost method for solving (mixed) integer programs and constraint satisfaction problems. Tree search algorithms recursively partition the search space to find an optimal solution. In order to keep the tree size small, it is crucial to carefully decide, when expanding a tree node, which question (typically variable) to branch on at that node in order to partition the remaining space. Numerous partitioning techniques (e.g., variable selection) have been proposed, but there is no theory describing which technique is optimal. We show how to use machine learning to determine an optimal weighting of any set of partitioning procedures for the instance distribution at hand using samples from the distribution. We provide the first sample complexity guarantees for tree search algorithm configuration. These guarantees bound the number of samples sufficient to ensure that the empirical performance of an algorithm over the samples nearly matches its expected performance on the unknown instance distribution. This thorough theoretical investigation naturally gives rise to our learning algorithm. Via experiments, we show that learning an optimal weighting of partitioning procedures can dramatically reduce tree size, and we prove that this reduction can even be exponential. Through theory and experiments, we show that learning to branch is both practical and hugely beneficial.
\end{abstract}

\section{Introduction}

Many widely-used algorithms are customizable: they have tunable parameters that have an enormous effect on runtime, solution quality, or both. Tuning parameters by hand is notoriously tedious and time-consuming. In this work, we study algorithm configuration via machine learning, where the goal is to design algorithms that learn the optimal parameter setting for the problem instance distribution at hand.

We study configuration of tree search algorithms. These algorithms are the most widely used tools for solving  combinatorial and nonconvex problems throughout artificial intelligence, operations research, and beyond (e.g., \citep{Russell10:Artificial, Williams13:Model}). For example, branch-and-bound (B\&B) algorithms \citep{Land60:Automatic} solve mixed integer linear programs (MILPs),  and thus have diverse applications, including ones in
machine learning such as MAP estimation~\citep{Kappes13:Towards}, object recognition~\citep{Kokkinos11:Rapid}, clustering~\citep{Komodakis09:Clustering}, and semi-supervised SVMs~\citep{Chapelle07:Branch}.

A tree search algorithm systematically partitions the search space to find an optimal solution. The algorithm organizes this partition via a tree: the original problem is at the root and the children of a given node represent the subproblems formed by partitioning the feasible set of the parent node. A branch is pruned if it is infeasible or it cannot produce a better solution than the best one found so far by the algorithm.
Typically the search space is partitioned by adding an additional constraint on some variable.  For example, suppose the feasible set is defined by the constraint $A\vec{x} \leq \vec{b}$, with $\vec{x} \in \{0,1\}^n$. A tree search algorithm might partition this feasible set into two sets, one where $A \vec{x} \leq \vec{b}$, $x_1 = 0$, and $x_2, \dots, x_n \in \{0,1\}$, and another where $A \vec{x} \leq \vec{b}$, $x_1 = 1$, and $x_2, \dots, x_n \in \{0,1\}$, in which case the algorithm has \emph{branched on} $x_1$.
A crucial question in tree search algorithm design is determining which variable to branch on at each step. An effective variable selection policy can have a tremendous effect on the size of the tree. Currently, there is no known optimal strategy and the vast majority of existing techniques are backed only by empirical comparisons. In the worst-case, finding an approximately optimal branching variable, even at the root of the tree alone, is NP-hard. This is true even in the case of satisfiability, which is a special case of constraint satisfaction and of
MILP \citep{Liberatore00:Complexity}.

In this work, rather than attempt to characterize a branching strategy that is universally optimal, we show empirically and theoretically that it is possible to learn high-performing branching strategies for a given application domain. We model an application domain as a distribution over problem instances, such as a distribution over scheduling problems that an airline solves on a day-to-day basis. This model is standard throughout the algorithm configuration literature (e.g.,~\citep{Hutter09:Paramils, Hutter11:Sequential, Dai17:Learning, Kleinberg17:Efficiency}). The approach has also been used on large-scale problems in industry to configure and select winner determination algorithms for clearing tens of billions of dollars of combinatorial auctions~\citep{Sandholm13:Very-Large-Scale}. The algorithm designer does not know the underlying distribution over problem instances, but has sample access to the distribution. We show how to use samples from the distribution to learn a variable selection policy that will result in as small a search tree as possible in expectation over the underlying distribution.

Our learning algorithm adaptively partitions the parameter space of the variable selection policy into regions where for any parameter in a given region, the resulting tree sizes across the training set are invariant. The learning algorithm returns the empirically optimal parameter over the training set, and thus performs empirical risk minimization (ERM). We prove that the adaptive nature of our algorithm is necessary: performing ERM over a data-independent discretization of the parameter space can be disastrous. In particular, for any discretization of the parameter space, we provide an infinite family of distributions over MILP instances such that every point in the discretization results in a B\&B tree with exponential size in expectation, but there exist infinitely-many parameters outside of the discretized points that result in a tree with constant size with probability 1. A small change in parameters can thus cause a drastic change in the algorithm's behavior. This fact contradicts conventional wisdom. For example, SCIP, the best  open-source MILP solver, sets one of the parameters we investigate to 5/6, regardless of the input MILP's structure. \citet{Achterberg09:SCIP} wrote that 5/6 was empirically optimal when compared against four other data-independent values. In contrast, our analysis shows that a data-driven approach to parameter tuning can have an enormous benefit.

The sensitivity of tree search algorithms to small changes in their parameters is a key challenge that differentiates our sample complexity analysis from those typically found in machine learning. For many well-understood function classes in machine learning, there is a close connection between the distance in parameter space between two parameter vectors and the distance in function space between the two corresponding functions. Understanding this connection is a necessary prerequisite to analyzing how many significantly different functions there are in the class, and thereby quantifying the class's intrinsic complexity. Intrinsic complexity typically translates to VC dimension, Rademacher complexity, or some other metric which allows us to derive learnability guarantees. Since the tree size of a search algorithm as a function of its parameters does not exhibit this predictable behavior, we must carefully analyze the way in which the parameters influence each step of the procedure in order to derive learning algorithms with strong guarantees.
In doing so, we present the first sample complexity guarantees for automated configuration of tree search algorithms. We provide worst-case bounds proving that a surprisingly small number of samples are sufficient for strong learnability guarantees: the sample complexity bound grows quadratically in the size of the problem instance, despite the complexity of the algorithms we study.

In our experiments section, we show that on many datasets based on real-world NP-hard problems, different parameters can result in B\&B trees of vastly different sizes.
Using an optimal parameter for one distribution on problems from a different distribution can lead to a dramatic tree size blowup.
We also provide data-dependent generalization guarantees that allow the algorithm designer to use far fewer samples than in the worst case if the data is well-structured.

\subsection{Related work}

Several works have studied the use of machine learning techniques in the context of B\&B; for an overview, see the summary by \citet{Lodi17:Learning}.

As in this work, \citet{Khalil16:Learning} study variable selection policies. Their goal is to find a variable selection strategy that mimics the behavior of the classic branching strategy known as \emph{strong branching} while running faster than strong branching. \citet{Alvarez17:Machine} study a similar problem, although in their work, the feature vectors in the training set describe nodes from multiple MILP instances. Neither or these works come with any theoretical guarantees, unlike our work.

Several other works study data-driven variable selection from a purely experimental perspective. \citet{DiLiberto16:Dash}
devise an algorithm that learns how to dynamically switch
between different branching heuristics along the branching tree. \citet{Karzan09:Information} propose techniques for choosing problem-specific branching rules based on a partial B\&B tree. Ideally, these branching rules will choose variables that will lead to fast fathoming. They do not rely on any techniques from machine learning. In the context of CSP tree search, \citet{Xia18:Learning} apply existing multi-armed bandit algorithms
to learning variable selection policies during tree search and \citet{Balafrej15:Multi} use a bandit approach to select different levels of propagation during search.

Other works have explored the use of machine learning techniques in the context of other aspects of B\&B beyond variable selection. For example, \citet{He14:Learning} use machine learning to speed up branch-and-bound, focusing on speeding up the node selection policy. Their work does not provide any learning-theoretic guarantees.
Other works that have studied machine learning techniques for branch-and-bound problems other than variable selection include \citet{Sabharwal12:Guiding}, who also study how to devise node selection policies, \citet{Hutter09:Paramils}, who study how to set CPLEX parameters, \citet{Kruber17:Learning}, who study how to detect decomposable model structure, and \citet{Khalil17:Learning}, who study how to determine when to run heuristics.

From a theoretical perspective, \citet{LeBodic17:Abstract} present a theoretical model for the selection of branching variables. It is based upon an abstraction of MIPs to a simpler setting in which it is possible to analytically evaluate the dual bound improvement of choosing a given variable. Based on this model, they present a new variable selection policy which has strong performance on many MIPLIB instances. Unlike our work, this paper is unrelated to machine learning.

The learning-theoretic model of algorithm configuration that we study in this paper was introduced to the theoretical computer science community by \citet{Gupta17:PAC}. Under this model, an application domain is modeled as a distribution over problem instances and the goal is to PAC-learn an algorithm that is nearly optimal over the distribution. This model was later studied by \citet{Balcan17:Learning} as well.
These papers were purely theoretical. In contrast, we show that the techniques proposed in this paper are practical as well, and provide significant benefit.

We provide a more detailed description of several of these papers in Appendix~\ref{app:related}.

\section{Tree search}

Tree search is a broad family of algorithms with diverse applications. 
To exemplify the specifics of tree search, we present a vast family of NP-hard problems --- (mixed) integer linear programs --- and describe how tree search finds optimal solutions to problems from this family. Later on in Section~\ref{sec:CSP}, we provide another example of tree search for constraint satisfaction problems. In Appendix~\ref{app:TS}, we provide a formal, more abstract definition of tree search and generalize our results to this more general algorithm.
\subsection{Mixed integer linear programs}

We study \emph{mixed integer linear programs} (MILPs)
where the objective is to maximize $\vec{c}^\top \vec{x}$ subject to $A\vec{x} \leq \vec{b}$ and where some of the entries of $\vec{x}$ are constrained to be in $\{0,1\}$.
Given a MILP $Q$, we denote an optimal solution to the LP relaxation of $Q$ as $\breve{\vec{x}}_Q = \left(\breve{x}_{Q}[1], \dots \breve{x}_{Q}[n]\right)$. Throughout this work, given a vector $\vec{a}$, we use the notation $a[i]$ to denote the $i^{th}$ component of $\vec{a}$. We also use the notation $\breve{c}_Q$ to denote the optimal objective value of the LP relaxation of $Q$. In other words, $\breve{c}_Q = \vec{c}^{\top} \breve{\vec{x}}_Q$.

\begin{example}[Winner determination]\label{ex:WD}
Suppose there is a set $\{1, \dots, m\}$ of items for sale and a set $\{1, \dots, n\}$ of buyers. In a combinatorial auction, each buyer $i$ submits bids $v_i(b)$ for any number of bundles $b \subseteq \{1, \dots, m\}$. The goal of the winner determination problem is to allocate the goods among the bidders so as to maximize \emph{social welfare}, which is the sum of the buyers' values for the bundles they are allocated. We can model this problem as a MILP by assigning a binary variable $x_{i,b}$ for every buyer $i$ and every bundle $b$ they submit a bid $v_i(b)$ on. The variable $x_{i,b}$ is equal to 1 if and only if buyer $i$ receives the bundle $b$. Let $B_i$ be the set of all bundles $b$ that buyer $i$ submits a bid on. An allocation is feasible if it allocates no item more than once ($\sum_{i = 1}^n \sum_{b \in B_i, j \ni b} x_{i,b} \leq 1$ for all $j \in \{1, \dots, m\}$) and if each bidder receives at most one bundle ($\sum_{b \in B_i} x_{i,b} \leq 1$ for all $i \in \{1, \dots, n\}$).
Therefore, the MILP is:
\[
\begin{array}{lll}
\textnormal{maximize} & \sum_{i = 1}^n \sum_{b \in B_i} v_i(b) x_{i,b} &\\
\textnormal{s.t.} & \sum_{i = 1}^n \sum_{b \in B_i, j \ni b} x_{i,b} \leq 1 &\forall j \in [m]\\
&\sum_{b \in B_i} x_{i,b} \leq 1 &\forall i \in [n]\\
& x_{i,b} \in \{0,1\} &\forall i \in [n], b \in B_i.
\end{array}
\]
\end{example}
\subsubsection{MILP tree search}
MILPs are typically solved using a tree search algorithm called branch-and-bound (B\&B).
Given a MILP problem instance, B\&B relies on two subroutines that efficiently compute upper and lower bounds on the optimal value within a given region of the search space. The lower bound can be found by choosing any feasible point in the region. An upper bound can be found via a linear programming relaxation. The basic idea of B\&B is to partition the search space into convex sets and find upper and lower bounds on the optimal solution within each. The algorithm uses these bounds to form global upper and lower bounds, and if these are equal, the algorithm terminates, since the feasible solution corresponding to the global lower bound must be optimal. If the global upper and lower bounds are not equal, the algorithm refines the partition and repeats.

In more detail, suppose we want to use B\&B to solve a MILP $Q'$. B\&B iteratively builds a search tree $\tree$ with the original MILP $Q'$ at the root. In the first iteration, $\tree$ consists of a single node containing the MILP $Q'$. At each iteration, B\&B uses a \emph{node selection policy} (which we expand on later) to select a leaf node of the tree $\tree$, which corresponds to a MILP $Q$. B\&B then uses a \emph{variable selection policy} (which we expand on in Section~\ref{sec:VSP}) to choose a variable $x_i$ of the MILP $Q$ to branch on. Specifically, let $Q^+_i$ (resp., $Q^-_i$) be the MILP $Q$ except with the additional constraint that $x_i = 1$ (resp., $x_i = 0$). B\&B sets the right (resp., left) child of $\node$ in $\tree$ to be a node containing the MILP $Q^+_i$ (resp., $Q^-_i$). B\&B then tries to ``fathom'' these leafs: the leaf containing $Q^+_i$ (resp., $Q^-_i$) is \emph{fathomed} if:
\begin{enumerate}
\item The optimal solution to the LP relaxation of $Q^+_i$ (resp., $Q^-_i$) satisfies the constraints of the original MILP $Q'$.
\item The relaxation of $Q^+_i$ (resp., $Q^-_i$) is infeasible, so $Q^+_i$ (resp., $Q^-_i$) must be infeasible as well.
\item The objective value of the LP relaxation of $Q^+_i$ (resp., $Q^-_i$) is smaller than the objective value of the best known feasible solution, so the optimal solution to $Q^+_i$ (resp., $Q^-_i$) is no better than the best known feasible solution.
\end{enumerate}
B\&B terminates when every leaf has been fathomed. It returns the best known feasible solution, which is optimal. See Algorithm~\ref{alg:BB} for the pseudocode.
\begin{algorithm}[t]
\caption{Branch and bound}\label{alg:BB}
\begin{algorithmic}[1]
   \Require A MILP instance $Q'$.
   \State Let $\tree$ be a tree that consists of a single node containing the MILP $Q'$.
   \State Let $c^* = -\infty$ be the objective value of the best-known feasible solution.
   \While {there remains an unfathomed leaf in $\tree$}
   \State Use a \emph{node selection policy} to select a leaf of the tree $\tree$, which corresponds to a MILP $Q$.\label{step:while_begin}
\State Use a \emph{variable selection policy} to choose a variable $x_i$ of the MILP $Q$ to branch on.\label{step:VSP}
\State Let $Q^+_i$ (resp., $Q^-_i$) be the MILP $Q$ except with the constraint that $x_i = 1$ (resp., $x_i = 0$).
\State Set the right (resp., left) child of $\node$ in $\tree$ to be a node containing the MILP $Q^+_i$ (resp., $Q^-_i$).
\For {$\tilde{Q} \in \left\{Q_i^+, Q_i^-\right\}$}
\If {the LP relaxation of $\tilde{Q}$ is feasible}
	\State Let $\breve{\vec{x}}_{\tilde{Q}}$ be an optimal solution to the LP and let $\breve{c}_{\tilde{Q}}$ be its objective value.
	\If {the vector $\breve{\vec{x}}_{\tilde{Q}}$ satisfies the constraints of the original MILP $Q'$}
	\State Fathom the leaf containing $\tilde{Q}$.
	\If {$c^* < \breve{c}_{\tilde{Q}}$}
		\State Set $c^* = \breve{c}_{\tilde{Q}}$.
	\EndIf
\ElsIf{$\breve{\vec{x}}_{\tilde{Q}}$ is no better than the best known feasible solution, i.e., $c^* \geq \breve{c}_{\tilde{Q}}$}
	\State Fathom the leaf containing $\tilde{Q}$.
	\EndIf
\Else
	\State Fathom the leaf containing $\tilde{Q}$.\label{step:while_end}
\EndIf
\EndFor
  \EndWhile
\end{algorithmic}
\end{algorithm}

The most common node selection policy is the \emph{best bound policy}. Given a B\&B tree, it selects the unfathomed leaf containing the MILP $Q$ with the maximum LP relaxation objective value. Another common policy is the \emph{depth-first policy}, which selects the next unfathomed leaf in the tree in depth-first order.

\begin{example}\label{ex:BB}
\begin{figure}[t]
\centering
\includegraphics[scale=.8]{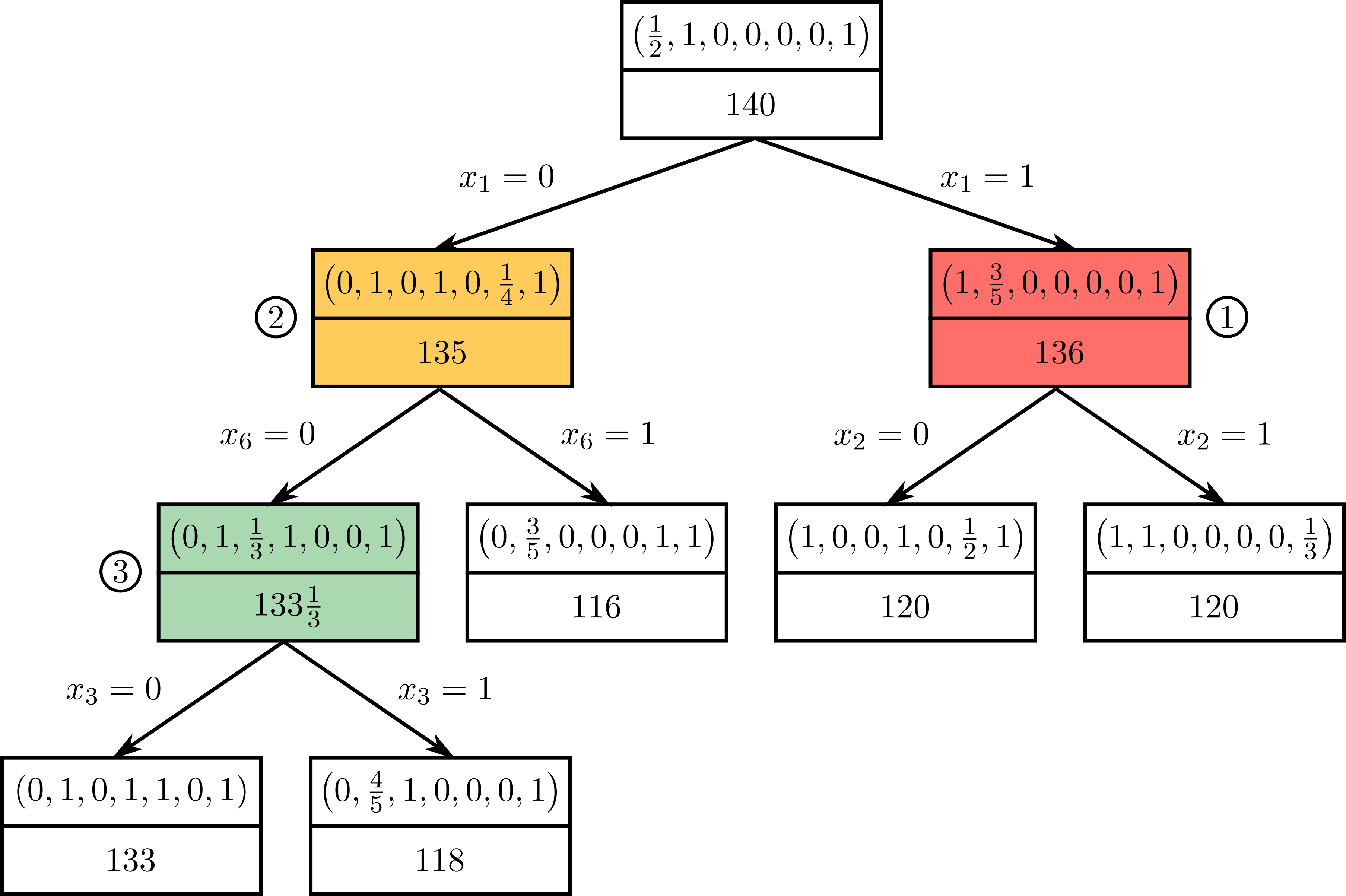}
\caption{Illustration of Example~\ref{ex:BB}.}\label{fig:BB}
\end{figure}
In Figure~\ref{fig:BB}, we show the search tree built by B\&B given as input the following MILP~\citep{Kolesar67:Branch}: \begin{equation}\begin{array}{ll}
\text{maximize} &40x_1 + 60x_2 + 10x_3 + 10x_4 + 3x_5 + 20x_6 + 60x_7\\
\text{subject to} &40x_1 + 50x_2 + 30x_3 + 10x_4 + 10x_5 + 40 x_6 + 30 x_7 \leq 100\\
& x_1, \dots, x_7 \in \{0,1\}.
\end{array}\label{eq:knapsack}\end{equation} Each rectangle denotes a node in the B\&B tree. Given a node $\node$, the top portion of its rectangle displays the optimal solution $\breve{\vec{x}}_Q$ to the LP relaxation of $Q$, which is the MILP \eqref{eq:knapsack} with the additional constraints labeling the edges from the root to $\node$. The bottom portion of the rectangle corresponding to $Q$ displays the objective value $\breve{c}_Q$ of the optimal solution to this LP relaxation, i.e., $\breve{c}_Q = (40, 60, 10, 10, 3, 20, 60) \cdot \breve{\vec{x}}_Q$. In this example, the node selection policy is the best bound policy and the variable selection policy selects the ``most fractional'' variable: the variable $x_i$ such that $\breve{x}_Q[i]$ is closest to $\frac{1}{2}$, i.e., $i = \argmax\left\{\min\left\{1 - \breve{x}_{\node}[i], \breve{x}_{\node}[i]\right\}\right\}$. 

In Figure~\ref{fig:BB}, the algorithm first explores the root. At this point, it has the option of exploring either the left or the right child. Since the optimal objective value of the right child (136) is greater than the optimal objective value of the left child (135), B\&B will next explore the pink node (marked \circled{1}). Next, B\&B can either explore either of the pink node's children or the orange node (marked \circled{2}). Since the optimal objective value of the orange node (135) is greater than the optimal objective values of the pink node's children (120), B\&B will next explore the orange node. After that B\&B can explore either of the orange node's children or either of the pink node's children. The optimal objective value of the green node (marked \circled{3}) is higher than the optimal objective values of the orange node's right child (116) and the pink node's children (120), so B\&B will next explore the green node. At this point, it finds an integral solution, which satisfies all of the constraints of the original MILP \eqref{eq:knapsack}. This integral solution has an objective value of 133. Since all of the other leafs have smaller objective values, the algorithm cannot find a better solution by exploring those leafs. Therefore, the algorithm fathoms all of the leafs and terminates.
\end{example}

\subsubsection{Variable selection in MILP tree search}\label{sec:VSP}

Variable selection policies typically depend on a real-valued \emph{score} per variable $x_i$.
\begin{definition}[Score-based variable selection policy]
Let $\score$ be a deterministic function that takes as input a partial search tree $\tree$, a leaf $Q$ of that tree, and an index $i$ and returns a real value ($\score(\tree, Q, i) \in \R$). For a leaf $Q$ of a tree $\tree$, let $N_{\tree, Q}$ be the set of variables that have not yet been branched on along the path from the root of $\tree$ to $Q$. A score-based variable selection policy selects the variable $\argmax_{x_j \in N_{\tree, Q}} \{\score(\tree, Q, j)\}$ to branch on at the node $Q$.
\end{definition}

We list several common definitions of the function $\score$ below. Recall that for a MILP $Q$ with objective function $\vec{c} \cdot \vec{x}$, we denote an optimal solution to the LP relaxation of $Q$ as $\breve{\vec{x}}_Q = \left(\breve{x}_{Q}[1], \dots \breve{x}_{Q}[n]\right)$. We also use the notation $\breve{c}_Q$ to denote the objective value of the optimal solution to the LP relaxation of $Q$, i.e., $\breve{c}_Q = \vec{c}^{\top} \breve{\vec{x}}_Q$. Finally, we use the notation $Q_i^+$ (resp., $Q_i^-$) to denote the MILP $Q$ with the additional constraint that $x_i = 1$ (resp., $x_i = 0$). If $Q_i^+$ (resp., $Q_i^-$) is infeasible, then we set $\breve{c}_Q - \breve{c}_{Q_i^+}$ (resp., $\breve{c}_Q - \breve{c}_{Q_i^-}$) to be some large number greater than $||\vec{c}||_1$.

\paragraph{Most fractional.} In this case, $\score(\tree, \node,i) = \min\left\{1 - \breve{x}_{\node}[i], \breve{x}_{\node}[i]\right\}.$ The variable that maximizes $\score(\tree, \node,i)$ is the ``most fractional'' variable, since it is the variable such that $\breve{x}_Q[i]$ is closest to $\frac{1}{2}$.

\paragraph{Linear scoring rule \citep{Linderoth99:Computational}.} In this case, $\score(\tree, \node,i) = (1 - \mu)\cdot\max\left\{\breve{c}_Q - \breve{c}_{Q_i^+}, \breve{c}_Q - \breve{c}_{Q_i^-}\right\} + \mu \cdot \min\left\{\breve{c}_Q - \breve{c}_{Q_i^+}, \breve{c}_Q - \breve{c}_{Q_i^-}\right\}$ where $\mu \in [0,1]$ is a user-specified parameter. This parameter balances an ``optimistic'' and a ``pessimistic'' approach to branching: An optimistic approach would choose the variable that maximizes $\max\left\{\breve{c}_Q - \breve{c}_{Q_i^+}, \breve{c}_Q - \breve{c}_{Q_i^-}\right\}$, which corresponds to $\mu = 0$, and a pessimistic approach would choose the variable that maximizes $\min\left\{\breve{c}_Q - \breve{c}_{Q_i^+}, \breve{c}_Q - \breve{c}_{Q_i^-}\right\}$, which corresponds to $\mu = 1$.

\paragraph{Product scoring rule \citep{Achterberg09:SCIP}.} In this case, $\score(\tree, \node,i) = \max \left\{\breve{c}_Q - \breve{c}_{Q_i^-}, \gamma\right\} \cdot \max\left\{\breve{c}_Q - \breve{c}_{Q_i^+}, \gamma\right\}$ where $\gamma = 10^{-6}$. Comparing $\breve{c}_Q - \breve{c}_{Q_i^-}$ and $\breve{c}_Q - \breve{c}_{Q_i^+}$ to $\gamma$ allows the algorithm to compare two variables even if $\breve{c}_Q - \breve{c}_{Q_i^-} = 0$ or $\breve{c}_Q - \breve{c}_{Q_i^+} = 0$. After all, suppose the scoring rule simply calculated the product $\left(\breve{c}_Q - \breve{c}_{Q_i^-}\right)\cdot \left(\breve{c}_Q - \breve{c}_{Q_i^+}\right)$ without comparing to $\gamma$. If $\breve{c}_Q - \breve{c}_{Q_i^-} = 0$, then the score equals 0, canceling out the value of $\breve{c}_Q - \breve{c}_{Q_i^+}$ and thus losing the information encoded by this difference.

\paragraph{Entropic lookahead scoring rule \citep{Gilpin11:Information}.} Let \[e(x) = \begin{cases}-x \log_2(x) - (1-x)\log_2(1-x) &\text{if } x \in (0,1)\\
0 &\text{if } x \in \{0,1\}.\end{cases}\] Set $\score(\tree, \node,i)=-\sum_{j = 1}^n \left(1 - \breve{x}_{\node}[i]\right) \cdot e \left(\breve{x}_{ \node_i^-}[j]\right) + \breve{x}_{\node}[i] \cdot e \left(\breve{x}_{\node_i^+}[j]\right).$

\paragraph{Alternative definitions of the linear and product scoring rules.} In practice, it is often too slow to compute the differences $\breve{c}_Q - \breve{c}_{Q_i^-}$ and $\breve{c}_Q - \breve{c}_{Q_i^+}$ for every variable, since it requires solving as many as $2n$ LPs. A faster option is to partially solve the LP relaxations of $Q_i^-$ and $Q_i^+$, starting at $\breve{\vec{x}}_Q$  and running
a small number of simplex iterations. Denoting the new objective values as $\tilde{c}_{Q_i^-}$ and $\tilde{c}_{Q_i^+}$, we can revise the linear scoring rule to be $\score(\tree, \node,i) = (1 - \mu)\cdot\max\left\{\breve{c}_Q - \tilde{c}_{Q_i^+}, \breve{c}_Q - \tilde{c}_{Q_i^-}\right\} + \mu \cdot \min\left\{\breve{c}_Q - \tilde{c}_{Q_i^+}, \breve{c}_Q - \tilde{c}_{Q_i^-}\right\}$ and we can revise the product scoring rule to be $\score(\tree, \node,i) = \max \left\{\breve{c}_Q - \tilde{c}_{Q_i^-}, \gamma\right\} \cdot \max\left\{\breve{c}_Q - \tilde{c}_{Q_i^+}, \gamma\right\}$. Other popular alternatives to computing $\breve{c}_{Q_i^-}$ and $\breve{c}_{Q_i^+}$ that fit within our framework are \emph{pseudo-cost branching} \citep{Benichou71:Experiments,Gauthier77:Experiments,Linderoth99:Computational} and \emph{reliability branching} \citep{Achterberg05:Branching}.
\section{Guarantees for data-driven learning to branch}\label{sec:theory}

In this section, we begin with our formal problem statement. We then present worst-case distributions over MILP instances demonstrating that learning over any data-independent discretization of the parameter space can be inadequate. Finally, we present sample complexity guarantees and a learning algorithm. Throughout the remainder of this paper, we assume that all aspects of the tree search algorithm except the variable selection policy, such as the node selection policy, are fixed.

\subsection{Problem statement}\label{sec:statement}
Let $\dist$ be a distribution over MILPs $Q$. For example, $\dist$ could be a distribution over clustering problems a biology lab solves day to day, formulated as MILPs. Let $\score_1, \dots, \score_d$ be a set of variable selection scoring rules, such as those in Section~\ref{sec:VSP}. Our goal is to learn a convex combination $\mu_1\score_1+ \cdots + \mu_d\score_d$ of the scoring rules that is nearly optimal in expectation over $\dist$. More formally, let $\cost$ be an abstract cost function that takes as input a problem instance $Q$ and a scoring rule $\score$ and returns some measure of the quality of B\&B using $\score$ on input $Q$. For example, $\cost(Q, \score)$ might be the number of nodes produced by running B\&B using $\score$ on input $Q$.
We say that an algorithm $(\epsilon, \delta)$-learns a convex combination of the
$d$ scoring rules $\score_1, \dots, \score_d$ if for any distribution $\dist$,
with probability at least $1-\delta$ over the draw of a sample
$\left\{Q_1, \dots, Q_m\right\} \sim \dist^m$, the algorithm returns
a convex combination $\score = \hat{\mu}_1\score_1+ \cdots +
\hat{\mu}_d\score_d$ such that $\E_{Q \sim \dist}\bigl[\cost(Q,
\score)\bigr] - \E_{Q \sim \dist}\bigl[\cost(Q, \score^*)\bigr] \leq
\epsilon$, where $\score^*$ is the convex combination of $\score_1, \dots,
\score_d$ with minimal expected cost. In this work, we prove that only a small
number of samples is sufficient to ensure $(\epsilon, \delta)$-learnability.

Following prior work (e.g.,~\citep{Hutter09:Paramils, Kleinberg17:Efficiency}), we assume that there is some cap $\kappa$ on the range of the cost function $\cost$. For example, if $\cost$ is the size of the search tree, we may choose to terminate the algorithm when the tree size grows beyond some bound $\kappa$.  We also assume that the problem instances in the support of $\dist$ are over $n$ binary variables, for some $n \in \N$.

Our results hold for cost functions that are \emph{tree-constant}, which means that for any problem instance $Q$, so long as the scoring rules $\score_1$ and $\score_2$ result in the same search tree, $\cost(Q, \score_1) = \cost(Q, \score_2)$.
For example, the size of the search tree is tree-constant.

\subsection{Impossibility results for data-independent approaches}\label{sec:WC_dist}

In this section, we focus on MILP tree search and prove that it is \emph{impossible} to find a nearly optimal B\&B configuration using a data-independent discretization of the parameters.
Specifically, suppose $\score_1(\tree, \node, i) = \min\left\{\breve{c}_Q - \breve{c}_{Q_i^+}, \breve{c}_Q - \breve{c}_{Q_i^-}\right\}$ and $\score_2(\tree, \node, i) = \max\left\{\breve{c}_Q - \breve{c}_{Q_i^+}, \breve{c}_Q - \breve{c}_{Q_i^-}\right\}$ and suppose the cost function $\cost\left(Q, \mu\score_1 + (1-\mu)\score_2\right)$ measures the size of the tree produced by B\&B when using a fixed but arbitrary node selection policy.
We would like to learn a nearly optimal convex combination $\mu\score_1 + (1-\mu)\score_2$ of these two rules with respect to $\cost$.
\citet{Gauthier77:Experiments} proposed setting $\mu = 1/2$, \citet{Benichou71:Experiments} and \citet{Beale79:Branch} suggested setting $\mu = 1$, and \citet{Linderoth99:Computational} found that $\mu = 2/3$ performs well. \citet{Achterberg09:SCIP} found that experimentally, $\mu = 5/6$ performed best when comparing among $\mu \in \left\{0, 1/2, 2/3, 5/6, 1\right\}$.

We show that for {\em any} discretization of the parameter space
[0, 1], there exists an infinite family of distributions over
MILP problem instances such that for any parameter in the
discretization, the expected tree size is exponential in $n$.
Yet, there exists an infinite number of parameters such
that the tree size is just a constant (with probability 1).
 The proof is in Appendix~\ref{app:theory}.

 \begin{restatable}{theorem}{wcDist}\label{thm:WCdist}
 Let \[\score_1(\tree, \node, i) = \min\left\{\breve{c}_Q - \breve{c}_{Q_i^+}, \breve{c}_Q - \breve{c}_{Q_i^-}\right\} \textnormal{, } \score_2(\tree, \node, i) = \max\left\{\breve{c}_Q - \breve{c}_{Q_i^+}, \breve{c}_Q - \breve{c}_{Q_i^-}\right\},\] and $\cost(Q, \mu\score_1 + (1-\mu)\score_2)$ be the size of the tree produced by B\&B. For every $a,b$ such that $\frac{1}{3} < a < b < \frac{1}{2}$ and for all even $n \geq 6$, there exists an infinite family of distributions $\dist$ over MILP instances with $n$ variables such that if $\mu \in [0,1]\setminus(a,b)$, then \[\E_{Q \sim \dist}\left[\cost\left(Q, \mu\score_1 + (1-\mu)\score_2\right)\right] = \Omega\left(2^{(n-9)/4}\right)\] and if $\mu \in (a,b)$, then with probability 1, $\cost\left(Q, \mu\score_1 + (1-\mu)\score_2\right) = O(1).$ This holds no matter which node selection policy B\&B uses.
 \end{restatable}

 \begin{proof}[Proof sketch]
We populate the support of the distribution $\dist$ by relying on two helpful theorems: Theorem~\ref{thm:families} and \ref{thm:families2}. In Theorem~\ref{thm:families}, we prove that for all $\mu^* \in (0,1)$, there exists an infinite family $\mathcal{F}_{n, \mu^*}$ of MILP instances such that for any $Q \in \mathcal{F}_{n, \mu^*}$, if $\mu \in \left[0,\mu^*\right)$, then the scoring rule $\mu \score_1 + (1-\mu)\score_2$ results in a B\&B tree with $O(1)$ nodes and if $\mu \in \left(\mu^*,1\right]$, the scoring rule results a tree with $2^{(n-4)/2}$ nodes. Conversely, in Theorem~\ref{thm:families2}, we prove that there exists an infinite family $\mathcal{G}_{n, \mu^*}$ of MILP instances such that for any $Q \in \mathcal{G}_{n, \mu^*}$, if $\mu \in \left[0,\mu^*\right)$, then the scoring rule $\mu \score_1 + (1-\mu)\score_2$ results in a B\&B tree with $2^{(n-5)/4}$ nodes and if $\mu \in \left(\mu^*,1\right]$, the scoring rule results a tree with $O(1)$ nodes.
Now, let $Q_a$ be an arbitrary instance in $\mathcal{G}_{n, a}$ and let $Q_b$ be an arbitrary instance in $\mathcal{F}_{n, b}$. The theorem follows by letting $\dist$ be a distribution such that $\Pr_{Q \sim \dist}\left[Q = Q_a\right] = \Pr_{Q \sim \dist}\left[Q = Q_b\right] = 1/2.$ See Figure~\ref{fig:WCdist} for an illustration.
\begin{figure}
\centering
\begin{subfigure}{0.3\textwidth}
\includegraphics{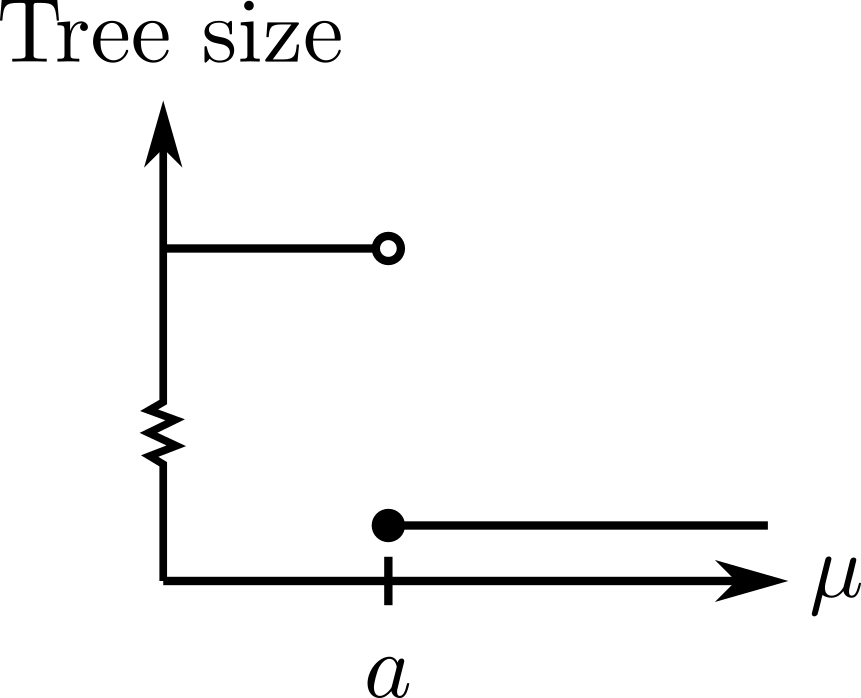} \centering
\caption{The tree size plot for the instance $Q_a$ as a function of $\mu$.\newline}
\end{subfigure}\qquad
\begin{subfigure}{0.3\textwidth}
\includegraphics{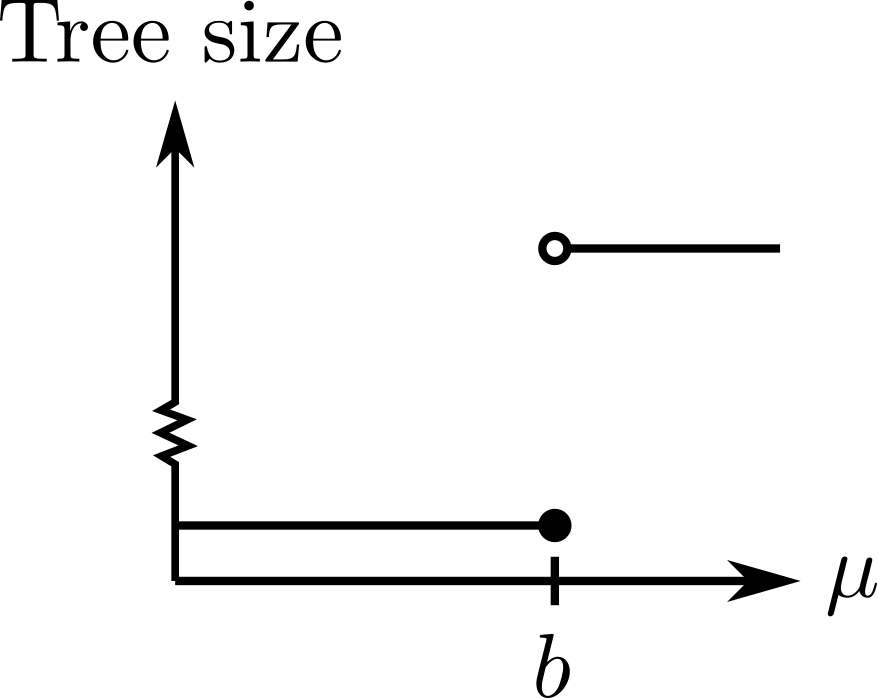}\centering
\caption{The tree size plot for the instance $Q_b$ as a function of $\mu$.\newline}
\end{subfigure}\qquad
\begin{subfigure}{0.3\textwidth}
\includegraphics{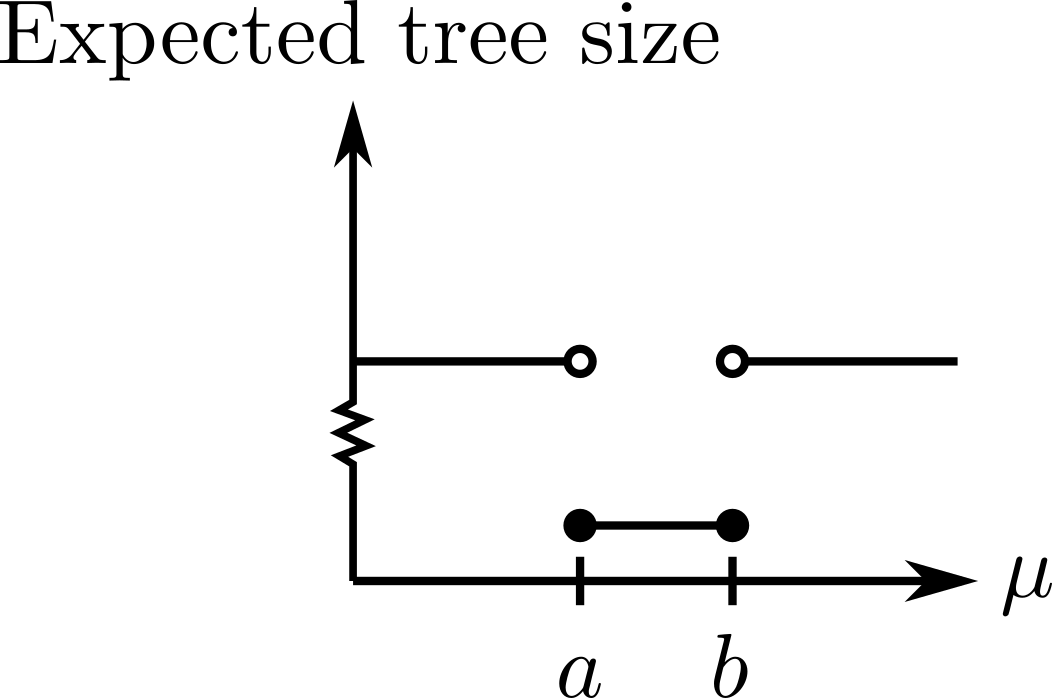}\centering
\caption{The expected tree size plot under the distribution $\dist$ as a function of $\mu$.}
\end{subfigure}
\caption{Illustrations of the proof of Theorem~\ref{thm:WCdist}.}\label{fig:WCdist}
\end{figure}

Throughout the proof of this theorem, we assume the node selection policy is depth-first search. We then prove that for any infeasible MILP, if NSP and NSP' are two node selection policies and $\score = \mu\score_1 + (1-\mu)\score_2$ for any $\mu \in [0,1]$, then tree $\tree$ B\&B builds using NSP and $\score$ equals the tree $\tree'$ it builds using NSP' and $\score$ (see Theorem~\ref{thm:NSP}). Thus, the theorem holds for any node selection policy.
\end{proof}

We now provide a proof sketch of Theorem~\ref{thm:families}, which helps us populate the support of the worst-case distributions in Theorem~\ref{thm:WCdist}. The full proof is in Appendix~\ref{app:theory}.

\begin{restatable}{theorem}{families}\label{thm:families}
 Let \[\score_1(\tree, \node, i) = \min\left\{\breve{c}_Q - \breve{c}_{Q_i^+}, \breve{c}_Q - \breve{c}_{Q_i^-}\right\} \textnormal{, } \score_2(\tree, \node, i) = \max\left\{\breve{c}_Q - \breve{c}_{Q_i^+}, \breve{c}_Q - \breve{c}_{Q_i^-}\right\},\] and $\cost(Q, \mu\score_1 + (1-\mu)\score_2)$ be the size of the tree produced by B\&B. For all even $n \geq 6$ and all $\mu^* \in \left(\frac{1}{3}, \frac{1}{2}\right)$, there exists an infinite family $\mathcal{F}_{n, \mu^*}$ of MILP instances such that for any $Q \in \mathcal{F}_{n, \mu^*}$, if $\mu \in \left[0,\mu^*\right)$, then the scoring rule $\mu \score_1 + (1-\mu)\score_2$ results in a B\&B tree with $O(1)$ nodes and if $\mu \in \left(\mu^*,1\right]$, the scoring rule results a tree with $2^{(n-4)/2}$ nodes.
\end{restatable}

\begin{proof}[Proof sketch]
The MILP instances in $\mathcal{F}_{n, \mu^*}$ are inspired by a worst-case B\&B instance introduced by \citet{Jeroslow74:Trivial}. He proved that for any odd $n'$, every B\&B algorithm will build a tree with $2^{(n'-1)/2}$ nodes before it determines that for any $\vec{c} \in \R^{n'}$, the following MILP is infeasible:
\[\begin{array}{ll}\text{maximize}&\vec{c} \cdot \vec{x}\\
\text{subject to} & 2\sum_{i = 1}^{n'} x_i = n'\\
& \vec{x} \in \{0,1\}^{n'}.
\end{array}\]

We build off of this MILP to create the infinite family $\mathcal{F}_{n, \mu^*}$. Each MILP in $\mathcal{F}_{n, \mu^*}$ combines a hard version of Jeroslow's instance  on $n-3$ variables $\{x_1, \dots, x_{n-3}\}$ and an easy version on 3 variables $\{x_{n-2}, x_{n-1}, x_n\}$. Branch-and-bound only needs to determine that one of these problems is infeasible in order to terminate. The key idea of this proof is that if B\&B branches on all variables in $\{x_{n-2}, x_{n-1}, x_n\}$ first, it will terminate upon making a small tree. However, if B\&B branches on all variables in $\{x_1, \dots, x_{n-3}\}$ first, it will create a tree with exponential size before it terminates. The challenge is to design an objective function that enforces the first behavior when $\mu < \mu^*$ and the second behavior when $\mu > \mu^*$. Proving this is the bulk of the work.

In a bit more detail, every instance in $\mathcal{F}_{n, \mu^*}$ is defined as follows. For any constant $\gamma \geq 1$, let $\vec{c}_1 = \gamma(1, 2, \dots, n-3)$ and let $\vec{c}_2 = \gamma\left(0, \frac{3}{2}, 3-\frac{1}{2\mu^*}\right)$. Let $\vec{c} = \left(\vec{c}_1, \vec{c}_2\right) \in \R^n$ be the concatenation of $\vec{c}_1$ and $\vec{c}_2$. Let $Q_{\gamma,n}$ be the MILP
\[\begin{array}{ll}
\text{maximize}&\vec{c} \cdot \vec{x}\\
\text{subject to} & 2\sum_{i = 1}^{n-3} x_i = n-3\\
&2\left(x_{n-2} + x_{n-1} + x_n\right) = 3\\
& \vec{x} \in \{0,1\}^n.
\end{array}\]
We define $\mathcal{F}_{n, \mu^*} = \left\{Q_{n, \gamma} : \gamma \geq 1\right\}.$

For example, if $\gamma = 1$ and $n = 8$, then $Q_{\gamma, n}$ is \[\begin{array}{ll}
\text{maximize}&\left(1, 2, 3, 4, 5, 0, \frac{3}{2}, 3-\frac{1}{2\mu^*}\right) \cdot \vec{x}\\
\text{subject to} & \begin{pmatrix}
2 & 2 & 2 & 2 & 2 & 0 & 0 & 0\\
0 & 0 & 0 & 0 & 0 & 2 & 2 & 2
\end{pmatrix}\vec{x} = \begin{pmatrix}
5\\ 3
\end{pmatrix}\\
& \vec{x} \in \{0,1\}^8.
\end{array}\]
\begin{figure}
\centering
\begin{subfigure}{0.45\textwidth}
\includegraphics[scale=0.8]{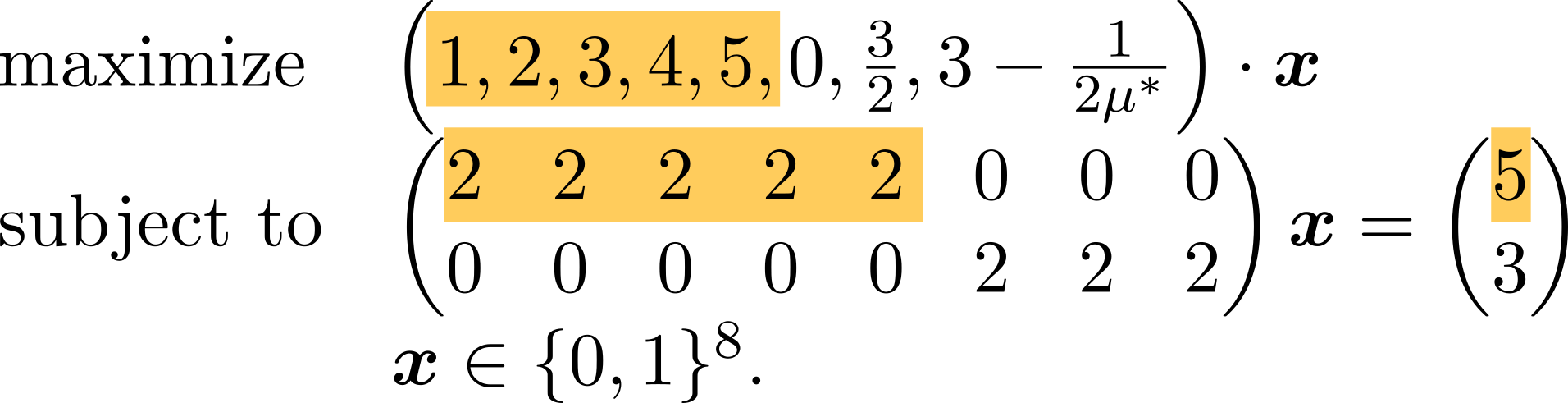} \centering
\caption{A big version of Jeroslow's instance on five variables.}
\label{fig:bigJ}
\end{subfigure}\qquad
\begin{subfigure}{0.45\textwidth}
\includegraphics[scale=0.8]{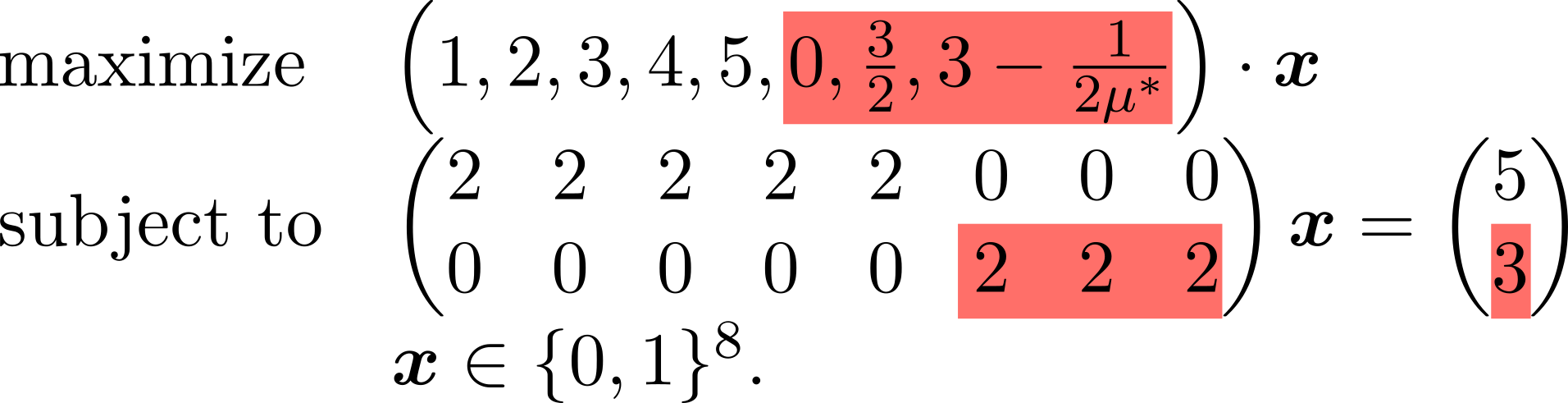}\centering
\caption{A small version of Jeroslow's instance on three variables.}
\label{fig:smallJ}
\end{subfigure}
\caption{Illustrations of the construction from Theorem~\ref{thm:families}.}\label{fig:J}
\end{figure}
As is illustrated in Figure~\ref{fig:J}, we have essentially ``glued together'' two disjoint versions of Jeroslow's instance: the first five variables of $Q_{1, 8}$ correspond to a ``big'' version of Jeroslow's instance and the last three variables correspond to a small version. Since the goal is maximization, the solution to the LP relaxation of $Q_{1,8}$ will try to obtain as much value from the first five variables $\{x_1, \dots, x_5\}$ as it can, but it is constrained to ensure that $2(x_1 + \cdots + x_5) = 5$. Therefore, the first five variables will be set to $\left(0, 0, \frac{1}{2}, 1, 1\right)$. Similarly, the solution to the LP relaxation of $Q_{1, 8}$ will set $(x_6, x_7, x_8) = \left(0, \frac{1}{2}, 1\right)$ because $\frac{3}{2} < 3 - \frac{1}{2\mu^*}$ under our assumption that $\mu^* > \frac{1}{3}$. Thus, the solution to the LP relaxation of $Q_{1, 8}$ is $\left(0, 0, \frac{1}{2}, 1, 1, 0, \frac{1}{2}, 1\right)$. There are only two fractional variables that B\&B might branch on: $x_3$ and $x_7$. Straightforward calculations show that if $\tree$ is the B\&B tree so far, which just consists of the root node, $\mu\score_1\left(\tree, Q_{\gamma, n},3\right) + (1-\mu)\score_2\left(\tree, Q_{\gamma, n},3\right) = \frac{\gamma}{2}$ and $\mu\score_1\left(\tree, Q_{\gamma, n},7\right) + (1-\mu)\score_2\left(\tree, Q_{\gamma, n},7\right) = \frac{3\gamma}{4} -  \frac{\mu\gamma}{4\mu^*}.$ This means that B\&B will branch first on variable $x_7$, which corresponds to the small version of Jeroslow's instance (see Figure~\ref{fig:smallJ}) if and only if $\frac{\gamma}{2} < \frac{3\gamma}{4} -  \frac{\mu\gamma}{4\mu^*}$, which occurs if and only if $\mu < \mu^*$. We show that this first branch sets off a cascade: if B\&B branches first on variable $x_7$, then it will proceed to branch on all variables in $\{x_6, x_7, x_8\}$, thus terminating upon making a small tree. Meanwhile, if it branches on variable $x_3$ first, it will then only branch on variables in $\{x_1, \dots, x_5\}$, creating a larger tree.

In the full proof, we generalize beyond eight variables to $n$, and expand the large version of Jeroslow's instance (as depicted in Figure~\ref{fig:bigJ}) from five variables to $n-3$. When $\mu < \mu^*$, we simply track B\&B's progress to make sure it only branches on
variables from the small version of Jeroslow's instance ($x_{n-2}, x_{n-1}, x_n$) before figuring out the MILP is infeasible. Therefore, the tree will have constant size. When $\mu > \mu^*$, we prove by induction that if B\&B has only branched on variables from the
big version of Jeroslow's instance ($x_1, \dots, x_{n-3}$), it will continue to only branch on those variables.
We also prove it will branch on about half of these variables along each path of the B\&B tree. The tree will thus have exponential size.
 \end{proof}

\subsection{Sample complexity guarantees}\label{sec:sample}

We now provide worst-case guarantees on the number of samples sufficient to $(\epsilon, \delta)$-learn a convex combination of scoring rules.
These results bound the number of samples sufficient to ensure that for any convex combination $\score = \mu_1\score_1 + \cdots + \mu_d\score_d$ of scoring rules, the empirical cost of tree search using $\score$ is close to its expected cost. Therefore, the algorithm designer can optimize $\left(\mu_1, \dots, \mu_d\right)$ over the samples without any further knowledge of the distribution. Moreover these sample complexity guarantees apply for any procedure the algorithm designer uses to tune $\left(\mu_1, \dots, \mu_d\right)$, be it an approximation, heuristic, or optimal algorithm. She can use our guarantees to bound the number of samples she needs to ensure that performance on the sample generalizes to the distribution.

 In Section~\ref{sec:path_scoring} we provide generalization guarantees for a family of  scoring rules we call \emph{path-wise}, which includes many well-known scoring rules as special cases. In this case, the number of samples is surprisingly small given the complexity of these problems: it grows only quadratically with the number of variables. In Section~\ref{sec:general_scoring}, we provide guarantees that apply to any scoring rule, path-wise or otherwise.

\subsubsection{Path-wise scoring rules}\label{sec:path_scoring}
The guarantees in this section apply broadly to a class of scoring rules we call \emph{path-wise} scoring rules. Given a node $Q$ in a search tree $\tree$, we denote the path from the root of $\tree$ to the node $Q$ as $\tree_Q$. The path $\tree_Q$ includes all nodes and edge labels from the root of $\tree$ to $Q$. For example, Figure~\ref{fig:path} illustrates the path $\tree_Q$ from the root of the tree $\tree$ in Figure~\ref{fig:tree1} to the node labeled $Q$. We now state the definition of path-wise scoring rules.
\begin{definition}[Path-wise scoring rule]\label{def:PW}
The function $\score$ is a path-wise scoring rule if for all search trees $\tree$, all nodes $Q$ in $\tree$, and all variables $x_i$, \begin{equation}\score(\tree, \node, i) = \score(\tree_Q, \node, i)\label{eq:PW}\end{equation} where $\tree_Q$ is the path from the root of $\tree$ to $\node$.\footnote{Under this definition, the scoring rule can simulate B\&B for any number of steps starting at any point in the tree and use that information to calculate the score, so long as Equality~\eqref{eq:PW} always holds.} See Figure~\ref{fig:PW} for an illustration.
\end{definition}
\begin{figure}
\centering
\begin{subfigure}{0.3\textwidth}
\includegraphics[scale=.7]{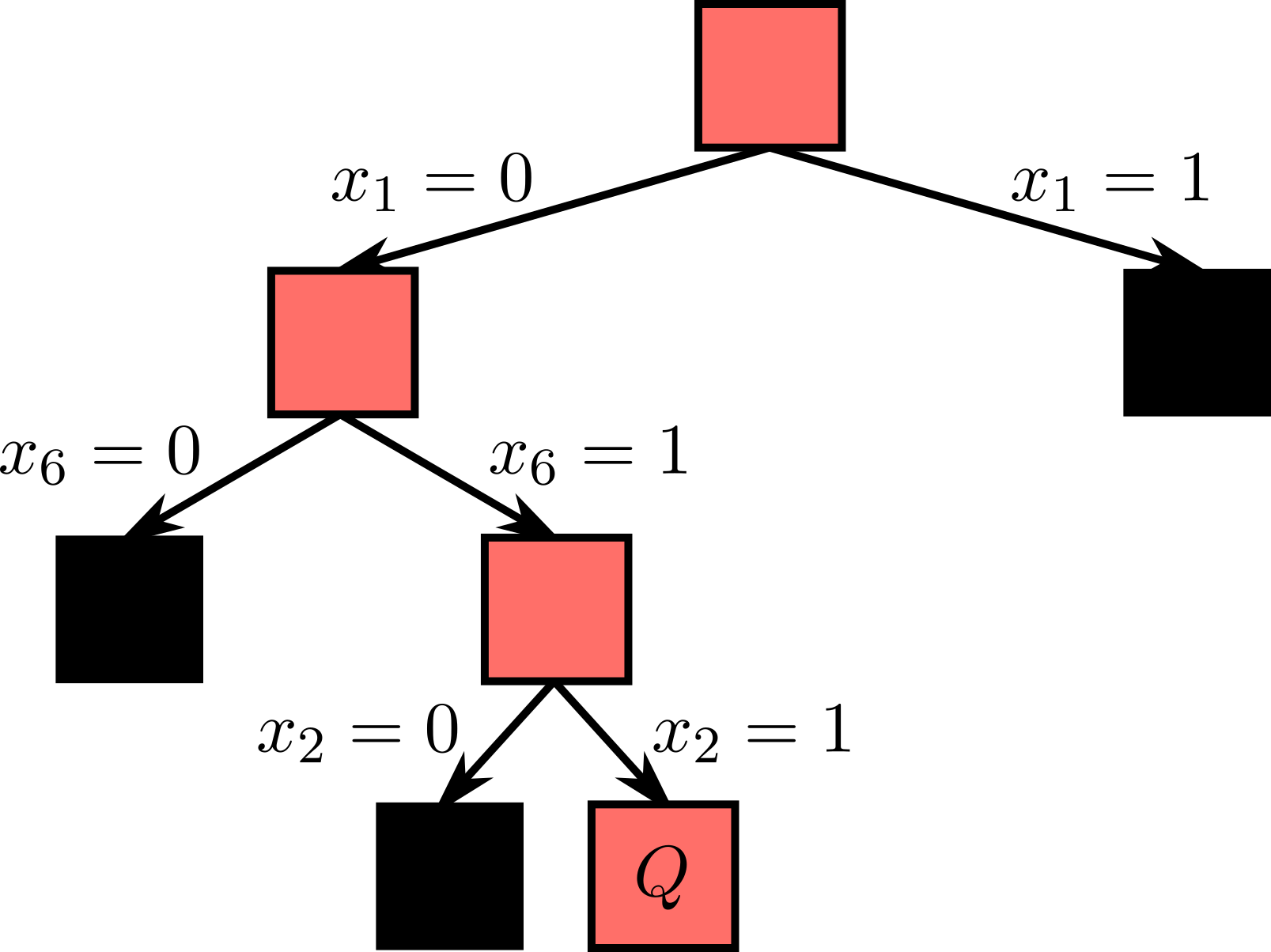} \centering
\caption{A B\&B search tree $\tree$.\newline\newline\newline}\label{fig:tree1}
\end{subfigure}\qquad
\begin{subfigure}{0.2\textwidth}
\includegraphics[scale=.7]{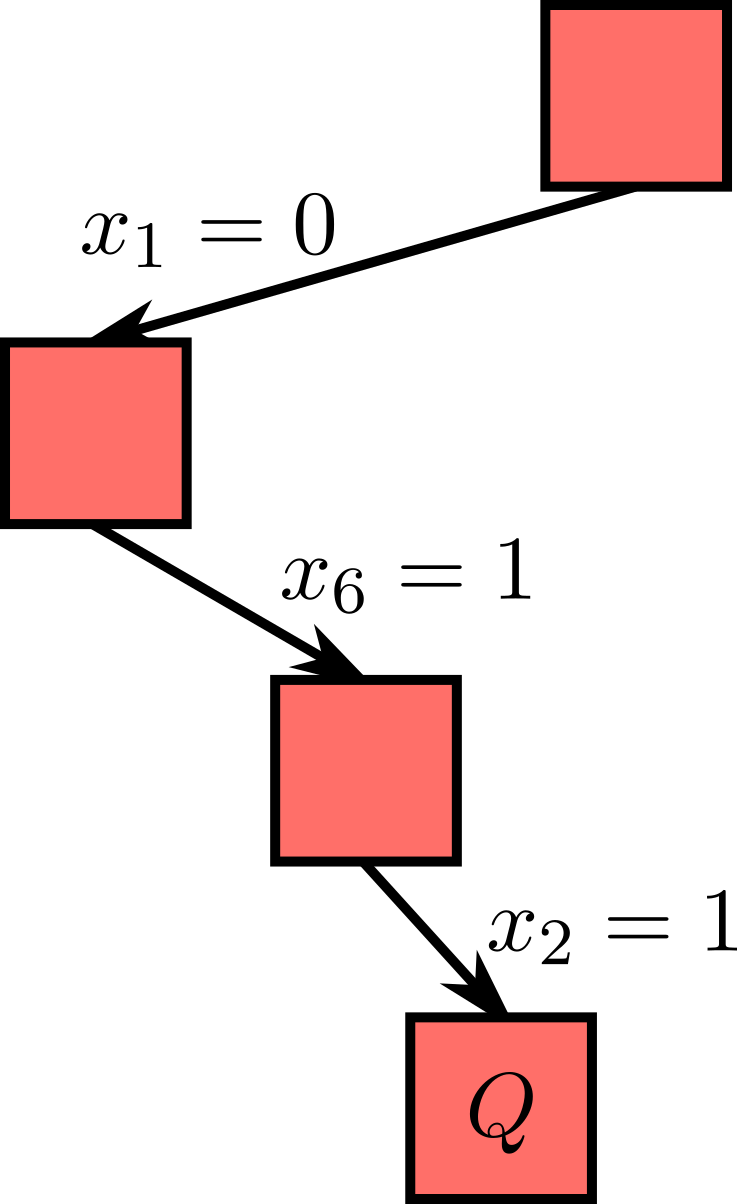}\centering
\caption{The path $\tree_Q$ from the root of the tree $\tree$ in Figure~\ref{fig:tree1} to the node labeled $Q$.}\label{fig:path}
\end{subfigure}\qquad
\begin{subfigure}{0.4\textwidth}
\includegraphics[scale=.7]{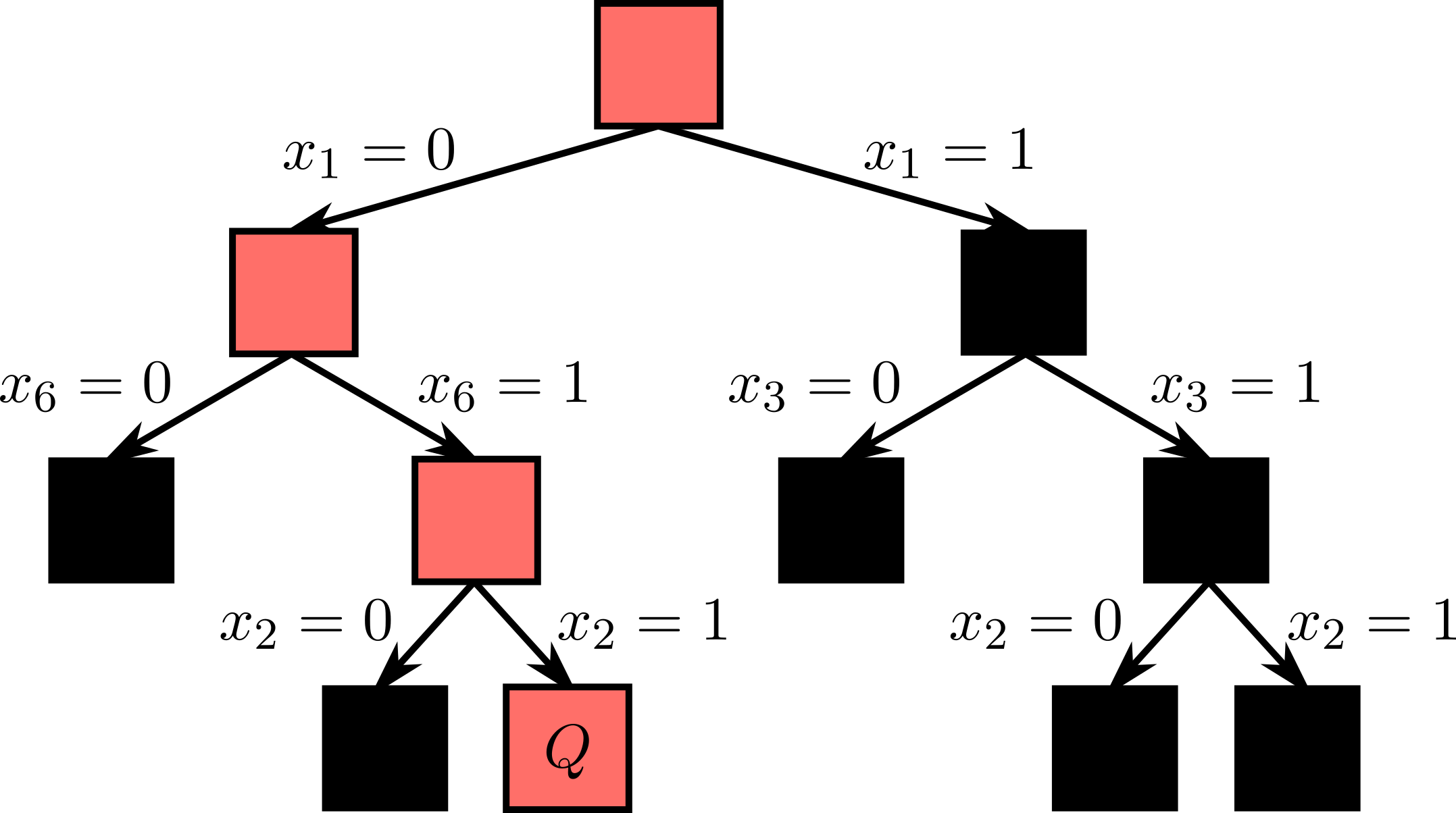}\centering
\caption{Another tree $\tree'$ that has the path $\tree_Q$ as a rooted subtree.\newline\newline}
\end{subfigure}
\caption{Illustrations to accompany the definition of a path-wise scoring rule (Definition~\ref{def:PW}). If the scoring rule $\score$ is path-wise, then for any variable $x_i$, $\score(\tree, Q, i) = \score(\tree', Q, i) = \score(\tree_Q, Q, i)$.}\label{fig:PW}
\end{figure}

Definition~\ref{def:PW} requires that if the node $Q$ appears at the end of the same path in two different B\&B trees, then any path-wise scoring rule must assign every variable the same score with respect to $Q$ in both trees.

Path-wise scoring rules include many well-studied rules as special cases, such as the \emph{most fractional}, \emph{product}, and \emph{linear} scoring rules, as defined in Section~\ref{sec:VSP}. The same is true when B\&B only partially solves the LP relaxations of $Q_i^-$ and $Q_i^+$ for every variable $x_i$ by running a small number of simplex iterations, as we describe in Section~\ref{sec:VSP} and as is our approach in our experiments. In fact, these scoring rules depend only on the node in question, rather than the path from the root to the node. We present our sample complexity bound for the more general class of path-wise scoring rules because this class captures the level of generality the proof holds for. On the other hand, \emph{pseudo-cost branching} \citep{Benichou71:Experiments,Gauthier77:Experiments,Linderoth99:Computational} and \emph{reliability branching} \citep{Achterberg05:Branching}, two widely-used branching strategies, are not path-wise, but our more general results from Section~\ref{sec:general_scoring} do apply to those strategies.

In order to prove our generalization guarantees, we make use of the following key structure which bounds the number of search trees branch-and-bound will build on a given instance over the entire range of parameters. In essence, this is a bound on the intrinsic complexity of the algorithm class defined by the range of parameters, and this bound on algorithm class's intrinsic complexity implies strong generalization guarantees.

\begin{lemma}\label{lem:induction}
Let $\cost$ be a tree-constant cost function, let $\score_1$ and $\score_2$ be two path-wise scoring rules, and let $Q$ be an arbitrary problem instance over $n$ binary variables. There are $T \leq 2^{n(n-1)/2}n^n$ intervals $I_1, \dots, I_T$ partitioning $[0,1]$ where for any interval $I_j$, across all $\mu \in I_j$, the scoring rule $\mu\score_1 + (1-\mu)\score_2$ results in the same search tree.
\end{lemma}

\begin{proof}
We prove this lemma first by considering the actions of an alternative algorithm $A'$ which runs exactly like B\&B, except it only fathoms nodes if they are integral or infeasible. We then relate the behavior of $A'$ to the behavior of B\&B to prove the lemma.

First, we prove the following bound on the number of search trees $A'$ will build on a given instance over the entire range of parameters. This bound matches that in the lemma statement.

\begin{restatable}{claim}{claimIntervals}\label{claim:simple_BnB_intervals}
There are $T \leq 2^{n(n-1)/2}n^n$ intervals $I_1, \dots, I_T$ partitioning $[0,1]$ where for any interval $I_j$, the search tree $A'$ builds using the scoring rule $\mu\score_1 + (1-\mu)\score_2$ is invariant across all $\mu \in I_j$.\footnote{This claim holds even when $\score_1$ and $\score_2$ are members of the more general class of \emph{depth-wise} scoring rules, which we define as follows. For any search tree $\tree$ of depth $\text{depth}(\tree)$ and any $j \in [n]$, let $\tree[j]$ be the subtree of $\tree$ consisting of all nodes in $\tree$ of depth at most $j$. We say that $\score$ is a \emph{depth-wise} scoring rule if for all search trees $\tree$, all $j \in [\text{depth}(\tree)]$, all nodes $Q$ of depth $j$, and all variables $x_i$, $\score(\tree, \node, i) = \score(\tree[j], \node, i)$.}
\end{restatable}

\begin{proof}[Proof sketch of Claim~\ref{claim:simple_BnB_intervals}]
We prove this claim by induction. For a tree $\tree$, let $\tree[i]$ be the nodes of depth at most $i$. We prove that for $i \in \{1,\dots, n\}$, there are $T \leq 2^{i(i-1)/2}n^i$ intervals $I_1, \dots, I_T$ partitioning $[0,1]$ where for any interval $I_j$ and any two parameters $\mu, \mu' \in I_j$, if $\tree$ and $\tree'$ are the trees $A'$ builds using the scoring rules $\mu\score_1 + (1-\mu)\score_2$ and $\mu'\score_1 + (1-\mu')\score_2$, respectively, then $\tree[i] = \tree'[i]$. Suppose that this is indeed the case for some $i \in \{1,\dots, n\}$ and consider an arbitrary interval $I_j$ and any two parameters $\mu, \mu' \in I_j$. Consider an arbitrary node $\node$ in $\tree[i-1]$ (or equivalently, $\tree'[i-1]$) at depth $i-1$. If $Q$ is integral or infeasible, then it will be fathomed no matter which parameter $\mu\in I_j$ the algorithm $A'$ uses. Otherwise, for all $\mu \in I_j$, let $\tree_{\mu}$ be the state of the search tree $A'$ builds using the scoring rule $\mu\score_1 + (1-\mu)\score_2$ at the point when it branches on $Q$. By the inductive hypothesis, we know that across all $\mu \in I_j$, the path from the root to $Q$ in $\tree_{\mu}$ is invariant, and we refer to this path as $\tree_Q$. Given a parameter $\mu \in I_j$, the variable $x_k$ will be branched on at node $\node$ so long as $k= \argmax_{\ell}\left\{\mu\score_1(\tree_{\mu},\node,\ell) + (1-\mu)\score_2(\tree_{\mu},\node,\ell)\right\},$ or equivalently, so long as $k= \argmax_{\ell}\left\{\mu\score_1(\tree_Q,\node,\ell) + (1-\mu)\score_2(\tree_Q,\node,\ell)\right\}$. In other words, the decision of which variable to branch on is determined by a convex combination of the constant values $\score_1(\tree_Q,\node,\ell)$ and $\score_2(\tree_Q,\node,\ell)$ no matter which parameter $\mu \in I_j$ the algorithm $A'$ uses. Here, we critically use the fact that the scoring rule is path-wise.

Since $\mu\score_1(\tree_Q,\node,\ell) + (1-\mu)\score_2(\tree_Q,\node,\ell)$ is a linear function of $\mu$ for all $\ell$, there are at most $n$ intervals subdividing the interval $I_j$ such that the variable branched on at node $\node$ is fixed.
\begin{figure}
\centering
\includegraphics{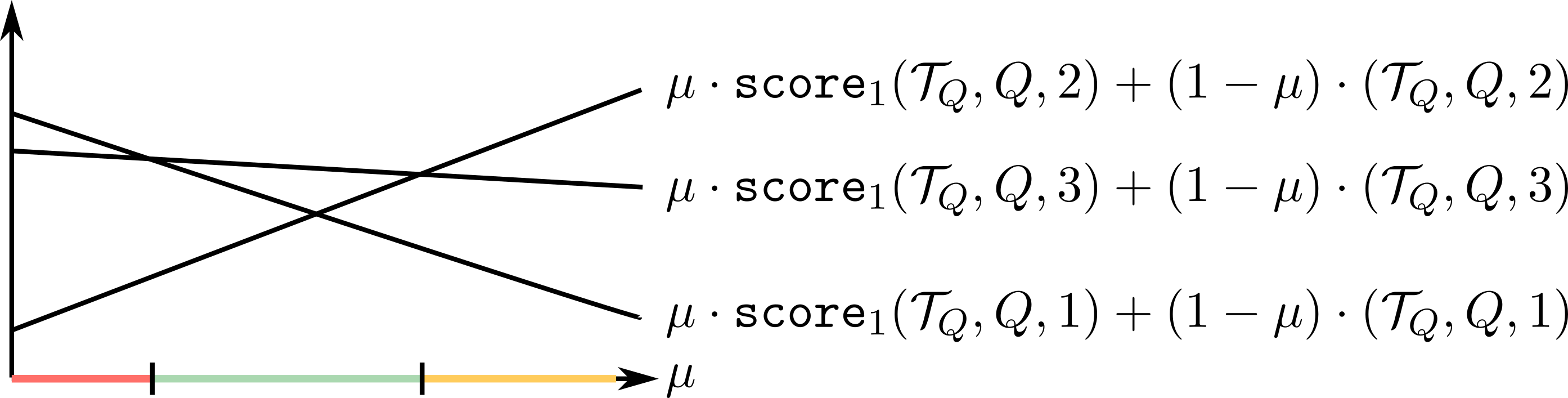}
\caption{Illustrations of the proof of Claim~\ref{claim:simple_BnB_intervals} for a hypothetical MIP $Q$ where the algorithm can either branch on $x_1$, $x_2$, or $x_3$ next. In the left-most interval (colored pink), $x_1$ will be branched on next, in the central interval (colored green), $x_3$ will be branched on next, and in the right-most interval (colored orange), $x_2$ next will be branched on next.}\label{fig:linear}
\end{figure}
Moreover, there are at most $2^{i-1}$ nodes at depth $i-1$, and each node similarly contributes a subpartition of $I_j$ of size $n$. If we merge all $2^{i-1}$ partitions, we have $T' \leq 2^{i-1}(n-1)+1$ intervals $I_1', \dots, I_{T'}'$ partitioning $I_j$ where for any interval $I_p'$ and any two parameters $\mu, \mu' \in I_p'$, if $\tree$ and $\tree'$ are the trees $A'$ builds using the scoring rules $\mu\score_1 + (1-\mu)\score_2$ and $\mu'\score_1 + (1-\mu')\score_2$, respectively, then $\tree[i] = \tree'[i]$. We can similarly subdivide each interval $I_1, \dots, I_T$. The claim then follows from counting the number of subdivisions.
\end{proof}

Next, we explicitly relate the behavior of B\&B to $A'$, proving that the search tree B\&B builds is a rooted subtree of the search tree $A'$ builds.

\begin{claim}\label{claim:simple_BnB_subtree}
Given a parameter $\mu \in [0,1]$, let $\tree$ and $\tree'$ be the trees B\&B and $A'$ build respectively using the scoring rule $\mu\score_1 + (1-\mu)\score_2$. For any node $Q$ of $\tree$, let $\tree_Q$ be the path from the root of $\tree$ to $Q$. Then $\tree_Q$ is a rooted subtree of $\tree'$.
\end{claim}

\begin{proof}[Proof of Claim~\ref{claim:simple_BnB_subtree}]
The path $\tree_Q$ can be labeled by a sequence of indices from $\{0,1\}$ and a sequence of variables from $\{x_1, \dots, x_n\}$ describing which variable is branched on and which value it takes on along the path $\tree_Q$. Let $((j_1, x_{i_1}), \dots, (j_t, x_{i_t}))$ be this sequence of labels, where $t$ is the number of edges in $\tree_Q$. We can similarly label every edge in $\tree'$. We claim that there exists a path beginning at the root of $\tree'$ with the labels $((j_1, x_{i_1}), \dots, (j_t, x_{i_t}))$.

For a contradiction, suppose no such path exists. Let $(j_{\tau}, x_{i_{\tau}})$ be the earliest label in the sequence $((j_1, x_{i_1}), \dots, (j_t, x_{i_t}))$ where there is a path beginning at the root of $\tree'$ with the labels $((j_1, x_{i_1}), \dots, (j_{\tau-1}, x_{i_{\tau - 1}}))$, but there is no way to continue the path using an edge labeled $(j_{\tau}, x_{i_{\tau}})$. There are exactly two reasons why this could be the case:
\begin{enumerate}
\item The node at the end of the path with labels $((j_1, x_{i_1}), \dots, (j_{\tau-1}, x_{i_{\tau - 1}}))$ was fathomed by $A'$.
\item The algorithm $A'$ branched on a variable other than $x_{i_{\tau}}$ at the end of the path labeled $((j_1, x_{i_1}), \dots, (j_{\tau-1}, x_{i_{\tau - 1}}))$.
\end{enumerate}

In the first case, since $A'$ only fathoms a node if it is integral or infeasible, we know that B\&B will also fathom the node at the end of the path with labels $((j_1, x_{i_1}), \dots, (j_{\tau-1}, x_{i_{\tau - 1}}))$. However, this is not the case since B\&B next branches on the variable $x_{i_\tau}$.

The second case is also not possible since the scoring rules are both path-wise. In a bit more detail, let $Q'$ be the node at the end of the path with labels $((j_1, x_{i_1}), \dots, (j_{\tau-1}, x_{i_{\tau - 1}}))$. We refer to this path as $\tree_{Q'}$. Let $\bar{\tree}$ (respectively, $\bar{\tree}'$) be the state of the search tree B\&B (respectively, $A$') has built at the point it branches on $Q'$. We know that $\tree_{Q'}$ is the path from the root to $Q'$ in both of the trees $\bar{\tree}$ and $\bar{\tree}'$. Therefore, for all variables $x_k$, $\mu \score_1(\bar{\tree}, Q', k) + (1-\mu)\score_2(\bar{\tree}, Q', k) = \mu \score_1(\tree_{Q'}, Q', k) + (1-\mu)\score_2(\tree_{Q'}, Q', k) = \mu \score_1(\bar{\tree}', Q', k) + (1-\mu)\score_2(\bar{\tree}', Q', k)$. This means that B\&B and $A'$ will choose the same variable to branch on at the node $Q'$.

 Therefore, we have reached a contradiction, so the claim holds.
\end{proof}

Next, we use Claims~\ref{claim:simple_BnB_intervals} and \ref{claim:simple_BnB_subtree} to prove Lemma~\ref{lem:induction}. Let $I_1, \dots, I_T$ be the intervals guaranteed to exist by Claim~\ref{claim:simple_BnB_intervals} and let $I_t$ be an arbitrary one of the intervals. Let $\mu'$ and $\mu''$ be two arbitrary parameters from $I_t$. We will prove that the scoring rules $\mu'\score_1 + (1-\mu')\score_2$ and $\mu''\score_1 + (1-\mu'')\score_2$ result in the same B\&B search tree. For a contradiction, suppose that this is not the case.
Consider the first iteration where B\&B using the scoring rule $\mu'\score_1 + (1-\mu')\score_2$ differs from B\&B using the scoring rule $\mu''\score_1 + (1-\mu'')\score_2$. By iteration, we mean lines~\ref{step:while_begin} through \ref{step:while_end} of Algorithm~\ref{alg:BB}. Up until this iteration, B\&B has built the same partial search tree $\tree$. Since the node selection policy does not depend on $\mu'$ or $\mu''$, B\&B will choose the same leaf $Q$ of the B\&B search tree to branch on no matter which scoring rule it uses.

Suppose B\&B chooses different variables to branch on in Step~\ref{step:VSP} of Algorithm~\ref{alg:BB} depending on whether it uses the scoring rule $\mu'\score_1 + (1-\mu')\score_2$ or $\mu''\score_1 + (1-\mu'')\score_2$. Let $\tree_Q$ be the path from the root of $\tree$ to $Q$. By Claim~\ref{claim:simple_BnB_intervals}, we know that the algorithm $A'$ builds the same search tree using the two scoring rules. Let $\bar{\tree}'$ (respectively, $\bar{\tree}''$) be the state of the search tree $A'$ has built using the scoring rule $\mu'\score_1 + (1-\mu')\score_2$ (respectively, $\mu''\score_1 + (1-\mu'')\score_2$) by the time it branches on the node $Q$. By Claims~\ref{claim:simple_BnB_intervals} and \ref{claim:simple_BnB_subtree}, we know that $\tree_Q$ is the path of from the root to $Q$ of both $\bar{\tree}'$ and $\bar{\tree}''$. By Claim~\ref{claim:simple_BnB_intervals},
we know that $A'$ will branch on the same variable $x_i$ at the node $Q$ in both the trees $\bar{\tree}'$ and $\bar{\tree}''$, so $i = \argmax_{j} \left\{\mu'\score_1(\bar{\tree}', Q, j)  + (1-\mu')\score_2(\bar{\tree}', Q, j)\right\}$, or equivalently, \begin{equation}i = \argmax_{j} \left\{\mu'\score_1(\tree_Q, Q, j)  + (1-\mu')\score_2(\tree_Q, Q, j)\right\},\label{eq:one_prime}\end{equation} and $i = \argmax_{j} \left\{\mu''\score_1(\bar{\tree}'', Q, j)  + (1-\mu'')\score_2(\bar{\tree}'', Q, j)\right\}$, or equivalently, \begin{equation}i = \argmax_{j} \left\{\mu''\score_1(\tree_Q, Q, j)  + (1-\mu'')\score_2(\tree_Q, Q, j)\right\}.\label{eq:two_primes}\end{equation} Returning to the search tree $\tree$ that B\&B is building, Equation~\eqref{eq:one_prime} implies that \[i = \argmax_{j} \left\{\mu'\score_1(\tree, Q, j)  + (1-\mu')\score_2(\tree, Q, j)\right\}\] and Equation~\eqref{eq:two_primes} implies that $i = \argmax_{j} \left\{\mu''\score_1(\tree, Q, j)  + (1-\mu'')\score_2(\tree, Q, j)\right\}$.
Therefore, B\&B will branch on $x_i$ at the node $Q$ no matter which scoring rule it uses. 

Finally, since the trees B\&B has built so far are identical, the choice of whether or not to fathom the children $Q_i^+$ and $Q_i^-$ does not depend on the scoring rule, so B\&B will fathom the same nodes no matter whether it uses the scoring rule $\mu'\score_1 + (1-\mu')\score_2$ or $\mu''\score_1 + (1-\mu'')\score_2$. Therefore, we have reached a contradiction: the iterations were identical. We conclude that the lemma holds.
\end{proof}

We now show how this structure implies a generalization guarantee. We formulate our guarantees in terms of \emph{pseudo-dimension}. By classic results from learning theory, pseudo-dimension immediately implies generalization guarantees. Pseudo-dimension is defined as follows.
\begin{definition}[Pseudo-dimension \citep{Pollard84:Convergence}]
Let $\mathcal{F}$ be a class of functions mapping an abstract domain $\domain$ to the set $[-\kappa,\kappa]$. Let $\sample= \left\{z_1, \dots, z_m\right\}$ be a subset of $\domain$ and let $r_1, \dots, r_m \in \R$ be a set of \emph{targets}. We say that $r_1, \dots, r_m$ \emph{witness} the shattering of $\sample$ by $\mathcal{F}$ if for all $\sample' \subseteq \sample$, there exists some function $f_{\sample'} \in \mathcal{F}$ such that for all $z_i \in {\sample'}$, $f_{\sample'}\left(z_i\right) \leq r_i$ and for all $z_i \not\in {\sample'}$, $f_{\sample'}\left(z_i\right) > r_i$. If there exists some $\vec{r} \in \R^m$ that witnesses the shattering of $\sample$ by $\mathcal{F}$, then we say that $\sample$ is \emph{shatterable} by $\mathcal{F}$. Finally, the pseudo-dimension of $\mathcal{F}$, denoted $\pdim\left(\fclass\right)$, is the size of the largest set that is shatterable by $\mathcal{F}$.
\end{definition}
Theorem~\ref{thm:gen_guarantee} provides generalization bounds in terms of pseudo-dimension.
\begin{theorem}[\cite{Pollard84:Convergence}]\label{thm:gen_guarantee}
For any distribution $\dist$ over $\domain$, with probability at least $1-\delta$ over the draw of $\sample = \left\{z_1, \dots, z_m\right\} \sim \dist^m$, for all $f \in \mathcal{F}$, \[\left|\E_{z \sim \dist}[f(z)] - \frac{1}{m} \sum_{i = 1}^m f\left(z_i\right)\right| = O\left(\kappa\sqrt{\frac{\pdim\left(\fclass\right)}{m}} + \kappa \sqrt{\frac{\ln\left(1/\delta\right)}{m}}\right).\]
\end{theorem}

\begin{theorem}\label{thm:MILP_WCpdim}
Let $\cost$ be a tree-constant cost function, let $\score_1$ and $\score_2$ be two path-wise scoring rules, and let $\mathcal{C}$ be the set of functions $\left\{\cost\left(\cdot, \mu\score_1 + (1-\mu)\score_2\right) : \mu \in [0,1]\right\}$. Then $\pdim(\mathcal{C}) = O\left(n^2\right).$
\end{theorem}

\begin{proof} Suppose that $\pdim(\mathcal{C}) = m$ and let $\sample = \left\{Q_1, \dots, Q_m\right\}$ be a shatterable set of problem instances. We know there exists a set of targets $r_1, \dots, r_m \in \R$ that witness the shattering of $\sample$ by $\mathcal{C}$. This means that for every $\sample' \subseteq \sample$, there exists a parameter $\mu_{\sample'}$ such that if $Q_i \in \sample$, then $\cost\left(Q_i, \mu_{\sample'}\score_1 + \left(1-\mu_{\sample'}\right)\score_2\right) \leq r_i$. Otherwise $\cost\left(Q_i, \mu_{\sample'}\score_1 + \left(1-\mu_{\sample'}\right)\score_2\right) > r_i$. Let $M = \left\{\mu_{\sample'} : \sample' \subseteq \sample\right\}$. We will prove that  $|M| \leq m2^{n(n-1)/2}n^n + 1$, and since $2^m = |M|$, this means that $\pdim(\mathcal{C}) = m = O\left(\log\left(2^{n(n-1)/2}n^n\right)\right) = O\left(n^2\right)$ (see Lemma~\ref{lem:log_ineq} in Appendix~\ref{app:theory}).

To prove that $|M| \leq m2^{n(n-1)/2}n^n+1$, we rely on Lemma~\ref{lem:induction}, which tells us that for any problem instance $Q$, there are $T \leq 2^{n(n-1)/2}n^n$ intervals $I_1, \dots, I_T$ partitioning $[0,1]$ where for any interval $I_j$, across all $\mu \in I_j$, the scoring rule $\mu\score_1 + (1-\mu)\score_2$ results in the same search tree. If we merge all $T$ intervals for all samples in $\sample$, we are left with $T' \leq m2^{n(n-1)/2}n^n+1$ intervals $I_1', \dots, I_{T'}'$ where for any interval $I_j'$ and any $Q_i \in \sample$, $\cost\left(Q_i, \mu\score_1 + \left(1-\mu\right)\score_2\right)$ is constant for all $\mu \in I_j'$. Therefore, at most one element of $M$ can come from each interval, meaning that $|M| \leq T' \leq m2^{n(n-1)/2}n^n+1$, as claimed.
\end{proof}

\subsubsection{General scoring rules}\label{sec:general_scoring}

In this section, we provide generalization guarantees that apply to learning convex combinations of any set of scoring rules. Unlike the guarantees in Section~\ref{sec:path_scoring}, they depend on the size of the search trees B\&B is allowed to build. For example, we may choose to terminate the algorithm when the tree size grows beyond some bound $\bar{\kappa}$. The following lemma corresponds to Lemma~\ref{lem:induction} for this setting. For the full proof, see Lemma~\ref{lem:induction_general_TS} in Appendix~\ref{app:general_TS} which proves the lemma for a more general tree search algorithm.

\begin{lemma}\label{lem:induction_general}
Let $\cost$ be a tree-constant cost function, let $\score_1, \dots, \score_d$ be $d$ arbitrary scoring rules, and let $Q$ be an arbitrary MILP over $n$ binary variables. Suppose we limit B\&B to producing search trees of
  size $\bar{\kappa}$. There is a set $\mathcal{H}$ of at most $n^{2(\bar{\kappa} + 1)}$ hyperplanes such that for any connected component $R$ of $[0,1]^d \setminus \mathcal{H}$, the search tree B\&B builds using the scoring rule $\mu_1\score_1 + \cdots + \mu_d\score_d$ is invariant across all $(\mu_1, \dots, \mu_d) \in R$.
\end{lemma}

\begin{proof}[Proof sketch]
  The proof has two steps. In Claim~\ref{claim:sequences_TS}, we show that there are at most $n^{\bar{\kappa}}$
  different search trees that B\&B might produce for the instance $Q$ as we
  vary the mixing parameter vector $(\mu_1, \dots, \mu_d)$. In Claim~\ref{claim:hyperplanes_TS}, for each of the possible search trees
  $\mathcal{T}$ that might be produced, we show that the set of parameter values
  $(\mu_1, \dots, \mu_d)$ which give rise to that tree lie in the intersection of $n^{\bar{\kappa} + 2}$ halfspaces. Of course, each of these halfspaces is defined by a hyperplane. Let $\mathcal{H}$ be the union of all $n^{\bar{\kappa} + 2}$ hyperplanes over all $n^{\bar{\kappa}}$ trees. We know that for any connected component $R$ of $[0,1]^d \setminus \mathcal{H}$, the search tree B\&B builds using the scoring rule $\mu_1\score_1 + \cdots + \mu_d\score_d$ is invariant across all $(\mu_1, \dots, \mu_d) \in R$, so the lemma statement holds.
\end{proof}

In the same way Lemma~\ref{lem:induction} implies the pseudo-dimension bound in Theorem~\ref{thm:MILP_WCpdim}, Lemma~\ref{lem:induction_general} also implies a pseudo-dimension bound. The proof is similar to that of Theorem~\ref{thm:MILP_WCpdim}. See Theorem~\ref{thm:MILP_WCpdim_arbitrary} in Appendix~\ref{app:general_TS} which proves the theorem for a more general tree search algorithm.

\begin{theorem}\label{thm:MILP_WCpdim_arbitrary_BB}
Let $\cost$ be a tree-constant cost function and let $\score_1, \dots, \score_d$ be $d$ arbitrary scoring rules. Suppose we limit B\&B to producing search trees of
  size $\bar{\kappa}$. Let $\mathcal{C}$ be the set of functions $\left\{\cost\left(\cdot, \mu\score_1 + \cdots + \mu_d\score_d\right) : (\mu_1, \dots, \mu_d) \in [0,1]^d\right\}$. Then $\pdim(\mathcal{C}) = O\left(d\bar{\kappa}\log n + d\log d\right).$
\end{theorem}

Naturally, if the function $\cost$ measures the size of the search tree B\&B returns capped at some value $\kappa$ as described in Section~\ref{sec:statement}, we obtain the following corollary.

\begin{cor}
Let $\cost$ be a tree-constant cost function, let $\score_1, \dots, \score_d$ be $d$ arbitrary scoring rules, and let $\kappa \in \N$ be an arbitrary tree size bound. Suppose that for any problem instance $Q$, $\cost(Q, \mu\score_1 + \cdots + \mu_d\score_d)$ equals the minimum of the following two values: 1) The number of nodes in the search tree B\&B builds using the scoring rule $\mu\score_1 + \cdots + \mu_d\score_d$ on input $Q$; and 2) $\kappa$. For any distribution $\dist$ over problem instances $Q$ with at most $n$ variables, with probability at least $1-\delta$ over the draw $\{Q_1, \dots, Q_m\} \sim \dist^m,$ for any $\mu \in [0,1]$, \begin{align*}&\left|\E_{Q \sim \dist} [\cost(Q, \mu\score_1 + \cdots + \mu_d\score_d)] - \frac{1}{m} \sum_{i = 1}^m \cost(Q_i, \mu\score_1 + \cdots + \mu_d\score_d)\right|\\
= \text{ }&O\left(\sqrt{\frac{d\kappa^2(\kappa\log n + \log d)}{m}} + \kappa \sqrt{\frac{\ln(1/\delta)}{m}}\right).\end{align*}
\end{cor}

\subsubsection{Learning algorithm}

For the case where we wish to learn the optimal tradeoff between two scoring rules, we provide an algorithm in Appendix~\ref{app:theory} that finds the empirically optimal parameter $\hat{\mu}$ given a sample of $m$ problem instances. By Theorems~\ref{thm:MILP_WCpdim}, \ref{thm:MILP_WCpdim_arbitrary_BB}, and \ref{thm:gen_guarantee}, we know that so long as $m$ is sufficiently large, $\hat{\mu}$ is nearly optimal in expectation.
 Our algorithm stems from the observation that for any tree-constant cost function $\cost$ and any problem instance $Q$, $\cost\left(Q, \mu\score_1 + (1-\mu)\score_2\right)$ is simple: it is a piecewise-constant function of $\mu$ with a finite number of pieces. This is the same observation that we prove in Lemma~\ref{lem:induction_general}. Given a sample of problem instances $\sample = \left\{Q_1, \dots, Q_m\right\}$, our ERM algorithm constructs all $m$ piecewise-constant functions, takes their average, and finds the minimizer of that function. In practice, we find that the number of pieces making up this piecewise-constant function is small, so our algorithm can learn over a training set of many problem instances.
\section{Experiments}
\begin{figure}
\centering
\begin{subfigure}{.3\textwidth}
\includegraphics[width=\textwidth]{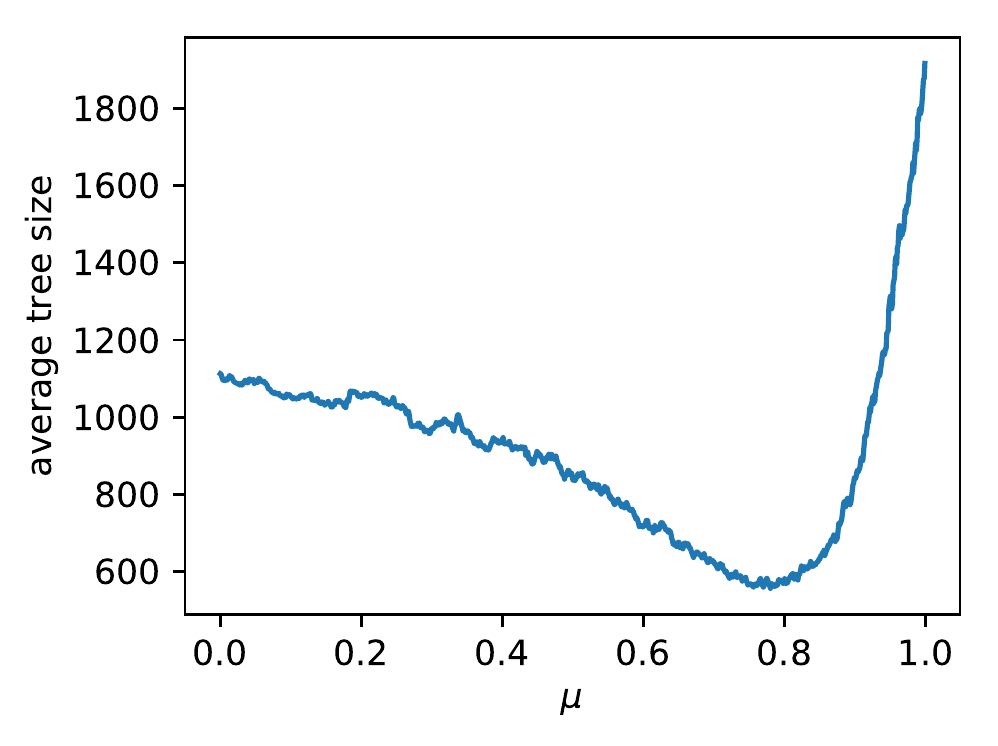}\centering
\caption{CATS ``arbitrary''}
\end{subfigure}\qquad
\begin{subfigure}{.3\textwidth}
\includegraphics[width=\textwidth]{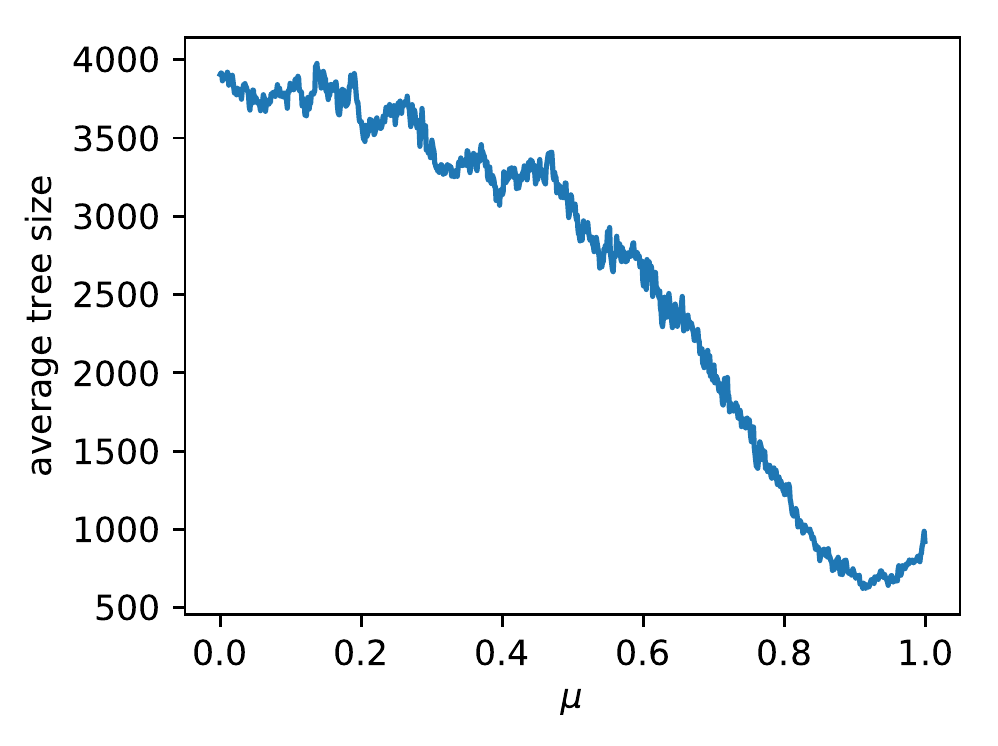}\centering
\caption{CATS ``regions''}
\end{subfigure}\qquad
\begin{subfigure}{.3\textwidth}
\includegraphics[width=\textwidth]{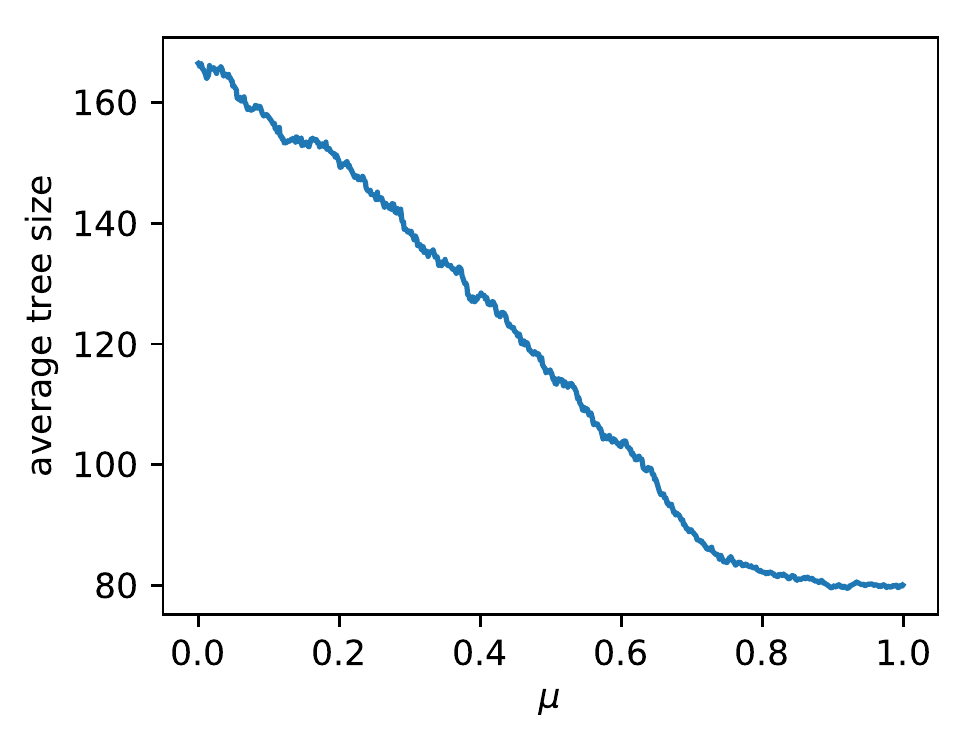}\centering
\caption{Facility location}
\end{subfigure}\qquad
\begin{subfigure}{.3\textwidth}
\includegraphics[width=\textwidth]{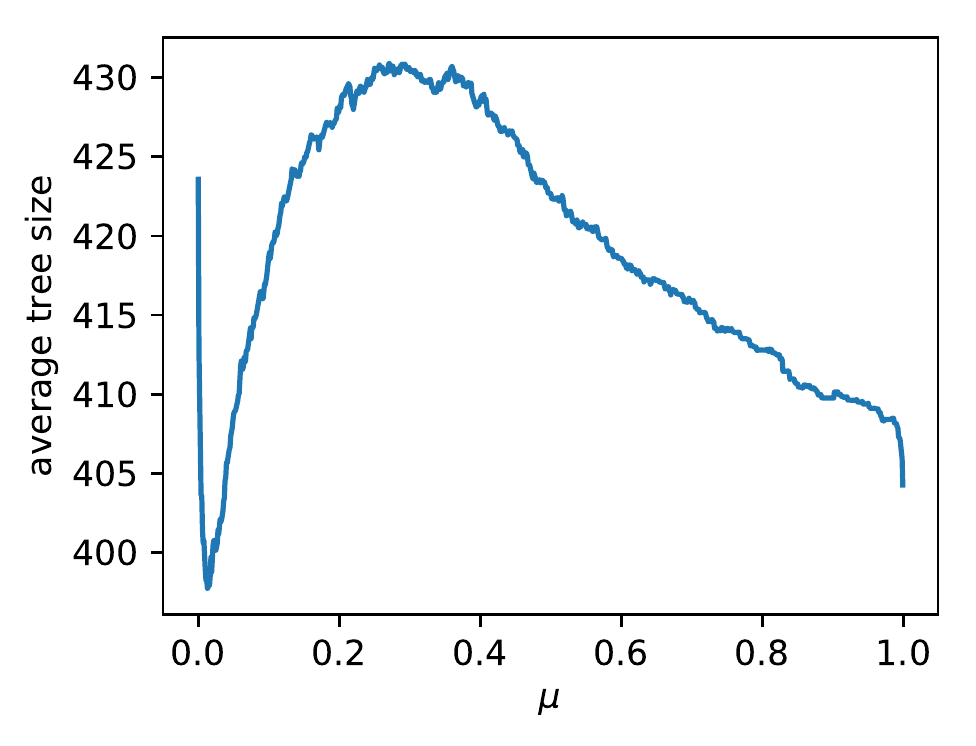}\centering
\caption{Linear separators}
\end{subfigure}\qquad
\begin{subfigure}{.3\textwidth}
\includegraphics[width=\textwidth]{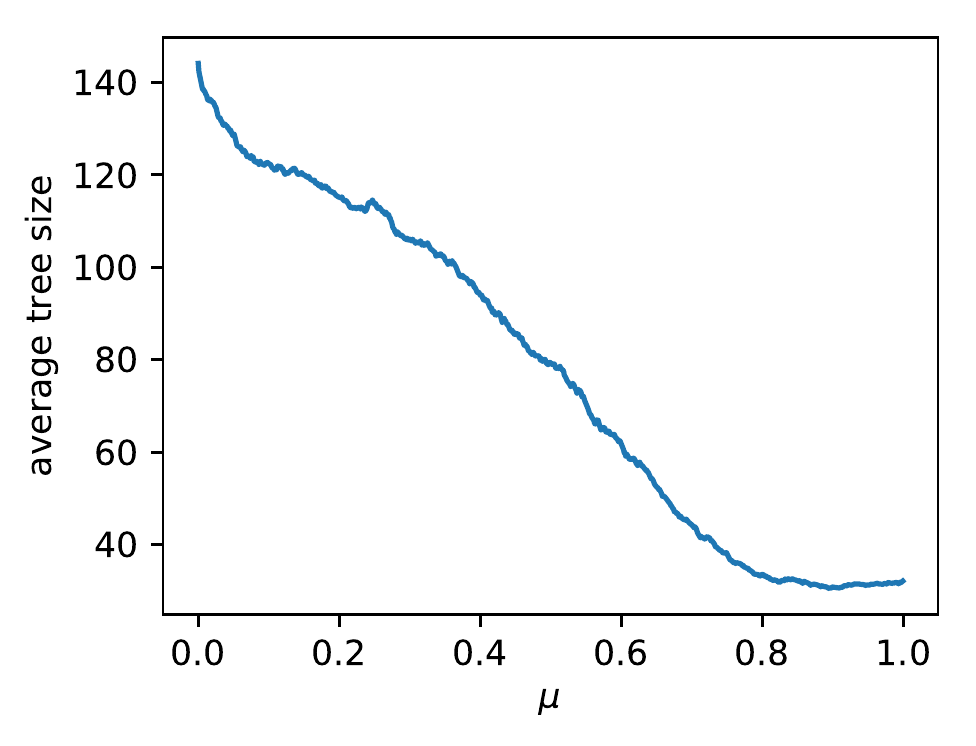}\centering
\caption{Clustering}
\end{subfigure}
\caption{The average tree size produced by B\&B when run with the linear scoring rule with parameter $\mu$.}
\label{fig:treesize}
\end{figure}

In this section, we show that the parameter of the variable selection rule in
B\&B algorithms for MILP can have a dramatic effect on the average tree size
generated for several domains, and no parameter value is effective across the
multiple natural distributions.

\paragraph{Experimental setup.} We use the C API of IBM ILOG CPLEX
12.8.0.0 to override the default variable selection rule using a branch
callback.
Additionally, our callback
performs extra book-keeping to determine a finite set of
values for the parameter that give rise to {\em all} possible B\&B trees for a given instance (for the given choice of branching rules that our algorithm is learning to weight). This ensures that there are no good or bad values for the parameter that get skipped; such skipping could be problematic according to our theory in Section~\ref{sec:WC_dist}.
We run CPLEX exactly once for each possible B\&B tree on
each instance. Following \citet{Khalil16:Learning, Fischetti12:Branching}, and \citet{Karzan09:Information}, we disable CPLEX's cuts and primal heuristics, and we also disable its root-node preprocessing.
The CPLEX node selection policy
is set to ``best bound'' (aka. $A^*$ in AI), which is the most typical choice in MILP. All experiments are run on a cluster of 64
c3.large Amazon AWS instances.

For each of the following application domains, Figure~\ref{fig:treesize} shows
the average B\&B tree size produced for each possible value of the $\mu$ parameter
for the linear scoring rule, averaged over $m$ independent samples from the
distribution.

\subparagraph{Combinatorial auctions.} We generate $m=100$
instances of the combinatorial auction winner determination problem under the OR-bidding language~\citep{Sandholm02:Algorithm}, which makes this problem equivalent to weighted set packing. The problem is NP-complete. We encode each instance as a binary MILP (see Example~\ref{ex:WD}). We use the Combinatorial
Auction Test Suite (CATS)~\citep{Leyton00:Toward} to generate these instances.
We use the ``arbitrary'' generator with 200 bids and 100 goods and ``regions'' generator with 400
bids and 200 goods.

\subparagraph{Facility location.} Suppose there is a set $I$ of customers
and a set $J$ of facilities that have not yet been built. The facilities each
produce the same good, and each consumer demands one unit of that good. Consumer
$i$ can obtain some fraction $y_{ij}$ of the good from facility $j$, which costs
them $d_{ij}y_{ij}$. Moreover, it costs $f_j$ to construct facility $j$. The
goal is to choose a subset of facilities to construct while minimizing total
cost. In Appendix~\ref{app:experiments} we show how to formulate facility
location as a MILP. We generate $m=500$ instances with 70 facilities and 70
customers each. Each cost $d_{ij}$ is uniformly sampled from $\left[0,
10^4\right]$ and each cost $f_j$ is uniformly sampled from $\left[0, 3\cdot
10^3\right]$. This distribution has regularly been used to generate benchmark
sets for facility location~\citep{Gilpin11:Information, Alekseeva15:Exact, Goldengorin11:Optimal, Kochetov05:Computationally, Homberger08:Two}.

\subparagraph{Clustering.} Given $n$ points $P = \left\{p_1, \dots,
p_n\right\}$ and pairwise distances $d(p_i, p_j)$ between each pair of points
$p_i$ and $p_j$, the goal of $k$-means clustering is to find $k$ \emph{centers}
$C = \left\{c_1, \dots, c_k\right\} \subseteq P$ such that the following
objective function is minimized: $\sum_{i = 1}^n \min_{j \in [k]} d\left(p_i,
c_j\right)^2.$ In Appendix~\ref{app:experiments}, we show how to formulate this
problem as a MILP. We generate $m=500$ instances with 35 points each and $k =
5$. We set $d(i,i) = 0$ for all $i$ and choose $d(i,j)$ uniformly at random from
$[0,1]$ for $i \neq j$. These distances do not satisfy the triangle inequality
and they are not symmetric (i.e., $d(i,j) \neq d(j,i)$), which tends to lead
to harder MILP instances than using Euclidean distances between randomly chosen
points in $\reals^d$.

\subparagraph{Agnostically learning linear separators.} Let $\vec{p}_1,
\dots, \vec{p}_N$ be $N$ points in $\R^d$ labeled by $z_1, \dots, z_N \in
\{-1,1\}$. Suppose we wish to learn a linear separator $\vec{w} \in \R^d$ that
minimizes 0-1 loss, i.e., $\sum_{i = 1}^N \mathbf{1}_{\left\{z_i\left\langle
\vec{p}_i, \vec{w}\right\rangle < 0\right\}}$. In
Appendix~\ref{app:experiments}, we show how to formulate this problem as a MILP.
We generate $m=500$ problem instances with 50 points $\vec{p}_1, \dots,
\vec{p}_{50}$ from the 2-dimensional standard normal distribution. We sample the
true linear separator $\vec{w}^*$ from the 2-dimensional standard Gaussian
distribution and label point $\vec{p}_i$ by $z_i = \sgn(\langle \vec{w}^*,
\vec{p}_i \rangle )$. We then choose $10$ random points and flip their labels so
that there is no consistent linear separator.

\paragraph{Experimental results.} The relationship between the variable
selection parameter and the average tree size varies greatly from application to
application. This implies that the parameters should be tuned on a
per-application basis, and that no parameter value is universally effective. In
particular, the optimal parameter for the ``regions'' combinatorial auction
problem, facility location, and clustering is close to 1. However, that value is
severely suboptimal for the ``arbitrary'' combinatorial auction domain,
resulting in trees that are
three times the size of the trees obtained under the optimal parameter value.

\paragraph{Data-dependent guarantees.}
We now explore data-dependent generalization guarantees. To prove our worst-case guarantee
Theorem~\ref{thm:MILP_WCpdim}, we show in Lemma~\ref{lem:induction} that for any
MILP instance over $n$ binary variables, there are $T \leq
2^{n(n-1)/2}n^n$ intervals $I_1, \dots, I_T$ partitioning $[0,1]$
where for any interval $I_j$, across all $\mu \in I_j$, the scoring rule
$\mu\score_1 + (1-\mu)\score_2$ results in the same B\&B tree. Our generalization guarantees grow logarithmically in
the number of intervals. In practice, we find that the number of
intervals partitioning $[0,1]$ is much smaller than $2^{n(n-1)/2}n^n$. In this section, we take advantage of this
data-dependent simplicity to derive stronger generalization guarantees when the number of
intervals partitioning $[0,1]$ is small. To do so, we move from pseudo-dimension to Rademacher complexity \citep{Bartlett02:Rademacher,
Koltchinskii01:Rademacher} since Rademacher complexity implies distribution-dependent guarantees whereas pseudo-dimension implies generalization guarantees that are worst-case over the distribution.
We now define empirical Rademacher complexity.
\begin{definition}[Empirical Rademacher complexity]\label{def:erad}
Let $\mathcal{F}$ be a class of functions mapping an abstract domain $\domain$ to the set $[-\kappa,\kappa]$. The \emph{empirical Rademacher complexity} of $\mathcal{F}$ with respect to a sample $\sample = \left\{z_1, \dots, z_m\right\} \subseteq \domain$ of size $m$ is defined as \[\erad(\mathcal{F}) = \frac{1}{m} \E_{\vec{\sigma} \sim \{-1,1\}^m}\left[\sup_{f \in \mathcal{F}}\sum_{i = 1}^m \sigma[i] f\left(z_i\right)\right].\]
\end{definition}
The following \emph{generalization guarantee} based on Rademacher complexity is well-known.
\begin{theorem}\label{thm:gen_guarantee_rad}
For any distribution $\dist$ over $\domain$, with probability at least $1-\delta$ over the draw of $\sample = \left\{z_1, \dots, z_m\right\} \sim \dist^m$, for all $f \in \mathcal{F}$, $\left|\E_{z \sim \dist}[f(z)] - \frac{1}{m} \sum_{i = 1}^m f\left(z_i\right)\right| \leq 2\erad(\mathcal{F}) + 4\kappa\sqrt{\frac{2}{m}\ln\frac{4}{\delta}}.$
\end{theorem}

To obtain data-dependent generalization guarantees, we rely on the following theorem, which follows from truncating the proof of Massart's finite lemma~\citep{Massart00:Some}, as observed by~\citet{Riondato15:Mining}. For a function class $\fclass$ and a  set $\sample  = \left\{z_1, \dots, z_m\right\} \subseteq \domain$, we use the notation $\fclass(\sample)$ to denote the set of vectors $\left\{\left(f\left(z_1\right), \dots, f\left(z_m\right)\right) : f \in \fclass \right\}$.

\begin{theorem}\label{thm:main_convex}
For a sample $\sample$ of size $m$, suppose that $\fclass(\sample)$ is finite.
Then \[\erad(\fclass) \leq \inf_{\lambda > 0}\frac{1}{\lambda} \log \left( \sum_{\vec{a} \in \fclass(\sample)}\exp \left(\frac{1}{2}\left(\frac{\lambda ||\vec{a}||_2}{m}\right)^2\right)\right).\]
\end{theorem}

Figure~\ref{fig:generalization} shows a comparison of the worst-case versus
data-dependent guarantees. The latter are significantly better. For ease of comparison, we compare the data-dependent bound provided by Theorem~\ref{thm:main_convex} with a variation of Theorem~\ref{thm:MILP_WCpdim} in terms of Rademacher complexity. See Theorem~\ref{thm:MILP_WCrad} in Appendix~\ref{app:experiments} for details.
\begin{figure}[t]\centering
\begin{subfigure}{.3\textwidth}
\includegraphics[width=\textwidth]{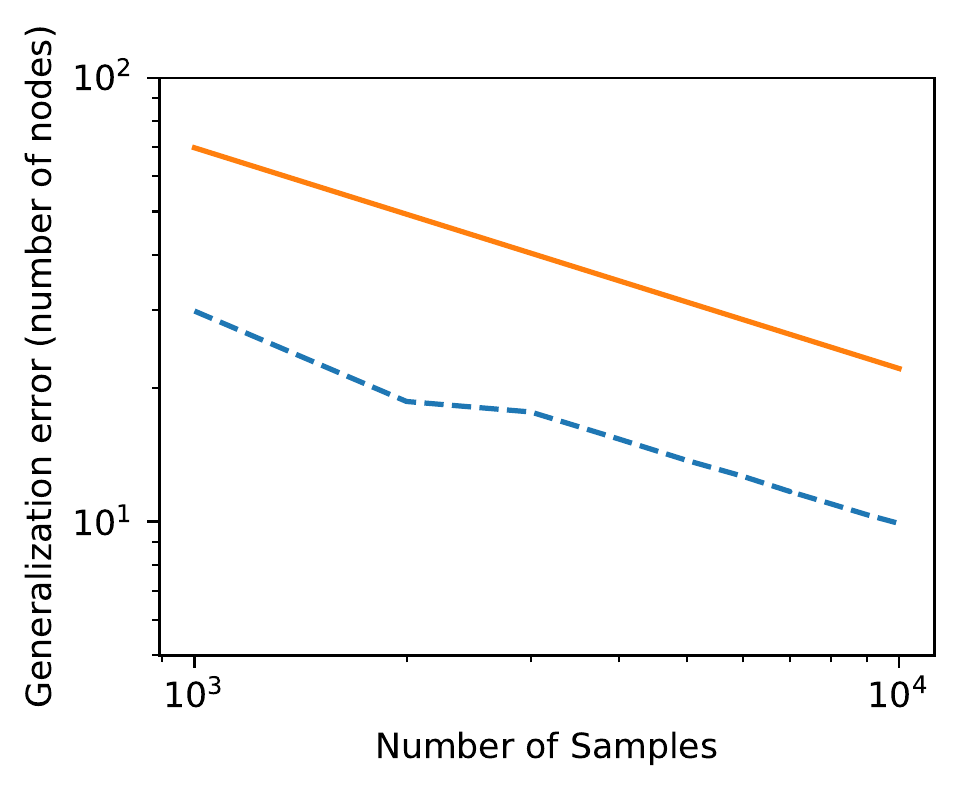}\centering
\caption{Linear separators}\label{fig:10loss10}
\end{subfigure}\qquad
\begin{subfigure}{.3\textwidth}
\includegraphics[width=\textwidth]{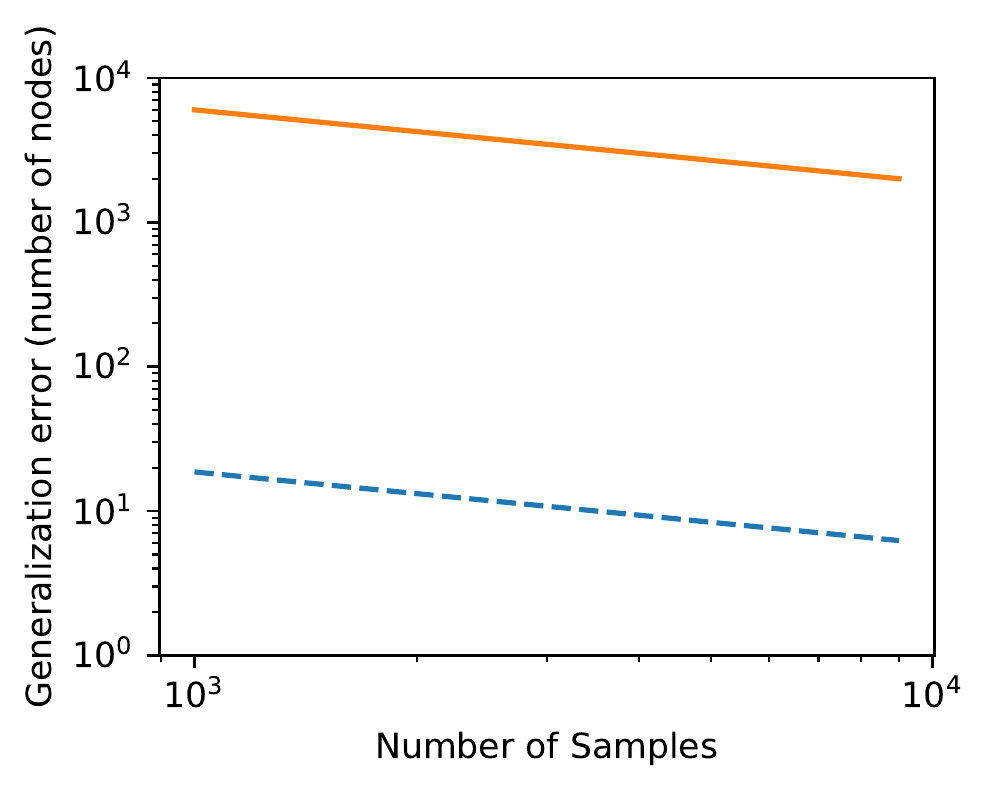}\centering
\caption{Clustering}\label{fig:kmeans_k3_35_5d}
\end{subfigure}
\caption{Data-dependent generalization guarantees (the dotted lines) and worst-case generalization guarantees (the solid lines). Figure~\ref{fig:10loss10} displays generalization guarantees for linear separators when
there are 10 points drawn from the 2-dimensional standard normal distribution, the true classifier is drawn from the 2-dimensional standard normal distribution, and 4 random points have their label flipped. Figure~\ref{fig:kmeans_k3_35_5d} shows generalization guarantees for $k$-means clustering MILPs when there are 35 points drawn from the 5-dimensional standard normal distribution and $k = 3$. In both plots, we cap the tree size at 150 nodes.}
\label{fig:generalization}
\end{figure}

\section{Constraint satisfaction problems}\label{sec:CSP}

In this section, we describe tree search for constraint satisfaction problems. The generalization guarantee from Section~\ref{sec:sample} also applies to tree search in this domain, as we describe in Appendix~\ref{app:TS}.

A constraint satisfaction problem (CSP) is a tuple $(X, D, C)$, where $X = \left\{x_1, \dots, x_n\right\}$ is a set of variables, $D = \left\{D_1, \dots, D_n\right\}$ is a set of domains where $D_i$ is the set of values variable $x_i$ can take on, and $C$ is a set of constraints between variables. Each constraint in $C$ is a pair $\left(\left(x_{i_1}, \dots, x_{i_r}\right), \psi\right)$ where $\psi$ is a function mapping $D_{i_1} \times \cdots \times D_{i_r}$ to $\{0,1\}$ for some $r \in [n]$ and some $i_1, \dots, i_r \in [n]$.
Given an assignment $\left(y_1, \dots, y_n\right) \in D_1 \times \cdots \times D_n$  of the variables in $X$, a constraint $\left(\left(x_{i_1}, \dots, x_{i_r}\right), \psi\right)$ is satisfied if $\psi\left(y_{i_1}, \dots, y_{i_r}\right) = 1$.
The goal is to find an assignment that maximizes the number of satisfied constraints.

The \emph{degree} of a variable $x$, denoted $\deg(x)$, is the number of constraints involving $x$. The \emph{dynamic degree} of (an unassigned variable) $x$ given a partial assignment $\vec{y}$, denoted $\ddeg(x, \vec{y})$ is the number of constraints involving $x$ and at least one other unassigned variable.

\begin{example}[Graph $k$-coloring]\label{ex:graph}
\begin{figure}
\centering
\begin{subfigure}{0.4\textwidth}
\includegraphics[scale=.8]{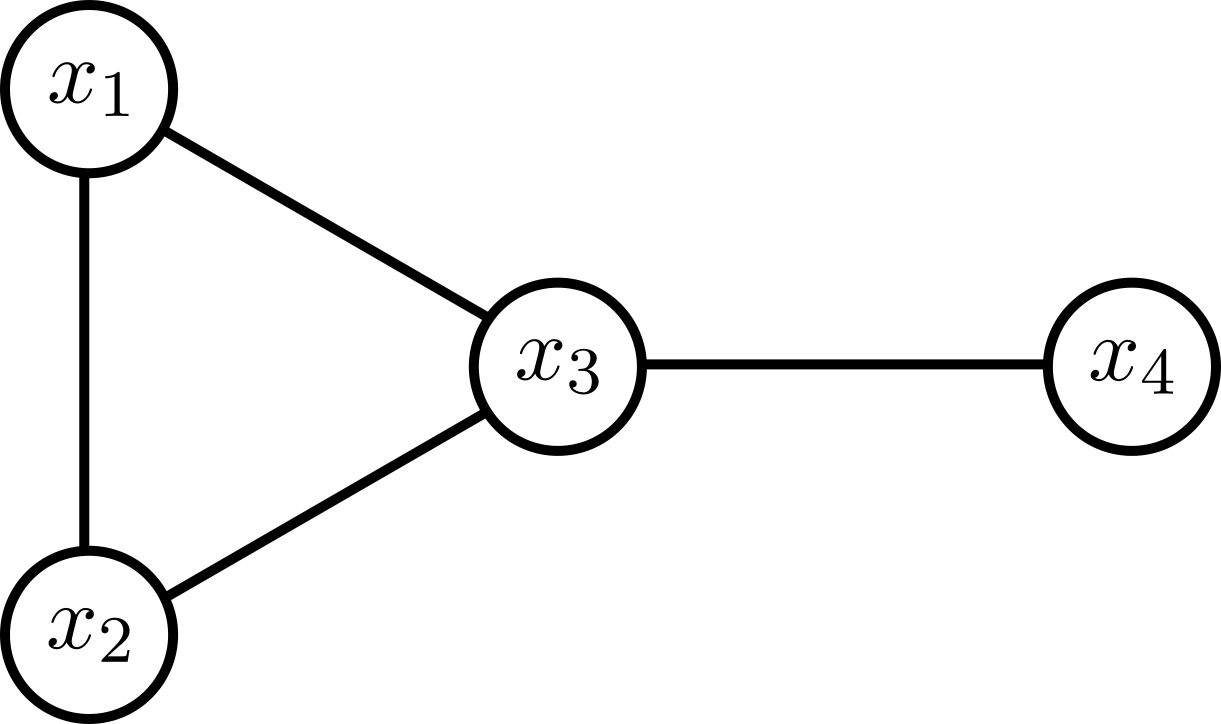} \centering
\caption{}
\label{fig:uncolored}
\end{subfigure}\qquad
\begin{subfigure}{0.4\textwidth}
\includegraphics[scale=.8]{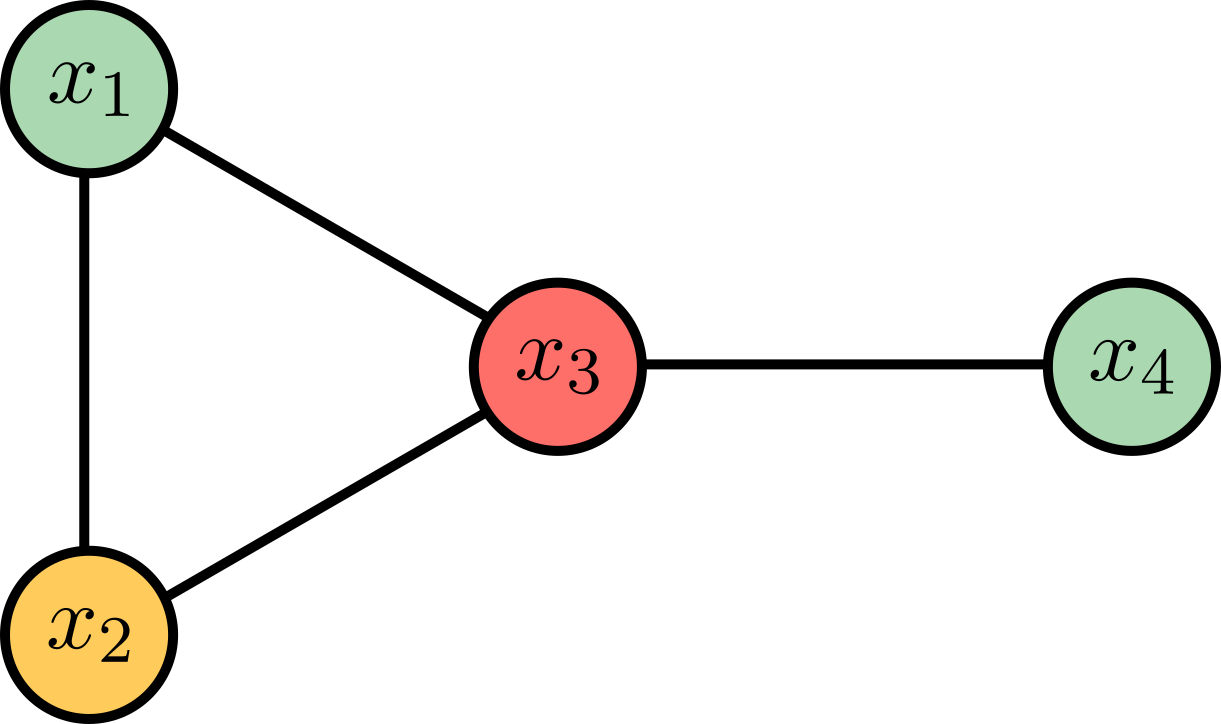}\centering
\caption{}
\label{fig:colored}
\end{subfigure}\qquad
\caption{Illustrations of Example~\ref{ex:graph}.}
\end{figure}
Given a graph, the goal of this problem is to color its vertices using at most $k$ colors such that no two adjacent vertices share the same color. This problem can be formulated as a CSP, as illustrated by the following example.
Suppose we want to 3-color the graph in Figure~\ref{fig:uncolored} using pink, green, and orange. The four vertices correspond to the four variables $X= \{x_1, \dots, x_4\}$. The domain $D_1 = \cdots = D_4 = \{\text{pink, green, orange}\}$. The only constraints on this problem are that no two adjacent vertices share the same color. Therefore we define $\psi$ to be the ``not equal'' relation mapping $\{\text{pink, green, orange}\} \times \{\text{pink, green, orange}\} \to \{0,1\}$ such that $\psi(\omega_1, \omega_2) = \textbf{1}_{\{\omega_1 \not= \omega_2\}}$. Finally, we define the set of constraints to be \[C = \left\{((x_1, x_2), \psi), ((x_1, x_3), \psi), ((x_2, x_3), \psi), ((x_3, x_4), \psi)\right).\] See Figure~\ref{fig:colored} for a coloring that satisfies all constraints ($y_1 = y_4 = \text{green, } y_2 = \text{orange, and } y_3 = \text{pink}).$ 
\end{example}
\subsubsection{CSP tree search}

CSP tree search begins by choosing a variable $x_i$ with domain $D_i$ and building $|D_i|$ branches, each one corresponding to one of the $|D_i|$ possible value assignments of $x$. Next, a node $\node$ of the tree is chosen, another variable $x_j$ is chosen, and $|D_j|$ branches from $\node$ are built, each corresponding to the possible assignments of $x_j$. The search continues and a branch is pruned if any of the constraints are not feasible given the partial assignment of the variables from the root to the leaf of that branch.

\subsubsection{Variable selection in CSP tree search}\label{sec:CSP_VSP}

As in MILP tree search, there are many variable selection policies researchers have suggested for choosing which variable to branch on at a given node. Typically, algorithms associate a score for branching on a given variable $x_i$ at node $\node$ in the tree $\tree$, as in B\&B. The algorithm then branches on the variable with the highest score.
We provide several examples of common variable selection policies below.

\paragraph{deg/dom and ddeg/dom \citep{Bessiere96:MAC}:} deg/dom corresponds to the scoring rule $\score(\tree, \node, i) = \deg(x_i)/|D_i|$ and ddeg/dom corresponds to the scoring rule $\score(\tree, \node, i) = \ddeg(x_i, \vec{y})/|D_i|$, where $\vec{y}$ is the assignment of variables from the root of $\tree$ to $\node$.

\paragraph{Smallest domain \citep{Haralick80:Increasing}:} In this case, $\score(\tree, \node, i) = 1/|D_i|$.

\bigskip

Our theory is for tree search and applies to both MILPs and CSPs.
It applies both to lookahead approaches that require learning the weighting of the two children (the more promising and less promising child) and to approaches that require learning the weighting of several different scoring rules.
\section{Conclusions and broader applicability}

In this work, we studied machine learning for tree search algorithm configuration. We showed how to learn a nearly optimal mixture of branching (e.g., variable selection) rules. Through experiments, we showed that using the optimal parameter for one application domain when solving problem instances from a different application domain can lead to a substantial tree size blow up. We proved that this blowup can even be exponential. 
We provided the first sample complexity guarantees for tree search algorithm configuration.
With only a small number of samples, the empirical cost incurred by using any mixture of two scoring rules will be close to its expected cost, where cost is an abstract measure such as tree size. We showed that using empirical Rademacher complexity, these bounds can be significantly tightened further.
Through theory and experiments, we showed that learning to branch is practical and hugely beneficial. 

While we presented the theory in the context of tree search, it also applies to other tree-growing applications. For example, it could be used for learning rules for selecting variables to branch on in order to construct small \emph{decision trees} that correctly classify training examples. Similarly, it could be used for learning to branch in order to construct a desirable
{\em taxonomy} of items represented as a tree---for example for representing  customer segments in advertising day to day.

\paragraph{Acknowledgments.}

This work was supported in part by the National Science Foundation under grants CCF-1422910, CCF-1535967,  IIS-1618714, IIS-1718457, IIS-1617590, CCF-1733556, a Microsoft Research Faculty Fellowship, an Amazon Research Award, a NSF Graduate Research Fellowship, and the ARO under award W911NF-17-1-0082.

\bibliography{../../dairefs}
\bibliographystyle{plainnat}
\newpage
\appendix
\section{Additional related work}\label{app:related}

As in this work, \citet{Khalil16:Learning} study variable selection policies. Their goal is to find a variable selection strategy that mimics the behavior of the classic branching strategy known as \emph{strong branching} while running faster than strong branching. They split a single MILP instance into a training and test set. The beginning of the B\&B algorithm is the training phase: their algorithm assigns features to each node and observes the variable selected to branch on by strong branching using the product scoring rule (see the definition in Section~\ref{sec:VSP}). Collecting these features and variable observations for a pre-specified number of nodes, they thus amass a training set. Their algorithm then uses this training set together with a classic ranking algorithm from the machine learning literature~\citep{Joachims02:Optimizing} to determine a function $f$ that, ideally, will imitate strong branching. In a bit more detail, on each node the algorithm encounters as it explores the rest of the tree, the algorithm assigns features $\vec{\phi}$ to that node, evaluates the function $f(\vec{\phi})$ which returns a variable to branch on, and the algorithm branches on that variable. The hope is that $f(\vec{\phi})$ matches the variable strong branching would choose to branch on at that node.

Other works have explored the use of machine learning techniques in the context of other aspects of B\&B beyond variable selection. For example, \citet{He14:Learning} use machine learning to speed up branch-and-bound, focusing on speeding up the node selection policy. Their work does not provide any learning-theoretic guarantees. Beyond node selection, \citet{He14:Learning} also aim to learn a \emph{pruning policy}, which determines whether to fathom a given node. This policy may occasionally fathom a node even when the branch contains the optimal solution, but ideally, this will happen infrequently or the resulting solution will be nearly optimal. Their learning algorithm takes as input a set of MIP instances and their optimal solutions. For each MIP in this set, the algorithm constructs a simple \emph{oracle} which returns, given a set of B\&B leaf nodes as input, the leaf whose feasible set contains the optimal solution. The learning algorithm uses the oracle to train the node selection and pruning policies using \emph{imitation learning}. Ideally, the resulting node selection policy should choose the nodes the oracle selects, and the pruning policy should prune any node that does not contain the optimal solution in its feasible set. 
\section{Variable selection policies}\label{app:VSP}

\begin{algorithm}
\caption{Generic variable selection}\label{alg:generic}
\begin{algorithmic}[1]
   \Require Current subproblem $Q$.
   \State Let $F = \{j \in I \ | \ \breve{x}_{Q}[j] \not\in \Z\}$ be the set of candidate variables.
\State For all candidates $i \in F$, calculate a score value $\score(Q, i)$.
   \Ensure $\argmax_{i} \{\score(Q, i)\}$.
\end{algorithmic}
\end{algorithm}
\section{Proofs from Section~\ref{sec:theory}}\label{app:theory}

\paragraph{Notation.} In the proofs in this section, we will use the following notation. Let $Q$ be a MILP instance. Suppose that we branch on $x_i$ and $x_j$, setting $x_i = 0$ and $x_j = 1$. We use the notation $Q_{i,j}^{-,+}$ to denote the resulting MILP. Similar, if we set $x_i = 1$ and $x_j = 0$, we denote the resulting MILP as $Q_{i,j}^{+,-}$.

\wcDist*

 \begin{proof}
We populate the support of the distribution $\dist$ by relying on two helpful theorems: Theorem~\ref{thm:families} and \ref{thm:families2}. In Theorem~\ref{thm:families}, we prove that for all $\mu^* \in \left(\frac{1}{3}, \frac{2}{3}\right)$, there exists an infinite family $\mathcal{F}_{n, \mu^*}$ of MILP instances such that for any $Q \in \mathcal{F}_{n, \mu^*}$, if $\mu \in \left[0,\mu^*\right)$, then the scoring rule $\mu \score_1 + (1-\mu)\score_2$ results in a B\&B tree with $O(1)$ nodes and if $\mu \in \left(\mu^*,1\right]$, the scoring rule results a tree with $2^{(n-4)/2}$ nodes. Conversely, in Theorem~\ref{thm:families2}, we prove that there exists an infinite family $\mathcal{G}_{n, \mu^*}$ of MILP instances such that for any $Q \in \mathcal{G}_{n, \mu^*}$, if $\mu \in \left[0,\mu^*\right)$, then the scoring rule $\mu \score_1 + (1-\mu)\score_2$ results in a B\&B tree with $2^{(n-5)/4}$ nodes and if $\mu \in \left(\mu^*,1\right]$, the scoring rule results a tree with $O(1)$ nodes.

Now, let $Q_a$ be an arbitrary instance in $\mathcal{G}_{n, a}$ and let $Q_b$ be an arbitrary instance in $\mathcal{F}_{n, b}$. The theorem follows by letting $\dist$ be a distribution such that $\Pr_{Q \sim \dist}\left[Q = Q_a\right] = \Pr_{Q \sim \dist}\left[Q = Q_b\right] = 1/2.$ We know that if $\mu \in [0,1] \setminus (a,b)$, then the expected value of $\cost\left(Q, \mu\score_1 + (1-\mu)\score_2\right)$ is
$\frac{1}{2}\left(O(1) + 2^{(n-5)/4}\right) \geq 2^{(n-9)/4}.$ Meanwhile, if $\mu \in (a,b)$, then with probability 1, \[\cost\left(Q, \mu\score_1 + (1-\mu)\score_2\right) = O(1).\]

Throughout the proof of this theorem, we assume the node selection policy is depth-first search. We then prove that for any infeasible MILP, if NSP and NSP' are two node selection policies and $\score = \mu\score_1 + (1-\mu)\score_2$ for any $\mu \in [0,1]$, then tree $\tree$ B\&B builds using NSP and $\score$ equals the tree $\tree'$ it builds using NSP' and $\score$ (see Theorem~\ref{thm:NSP}). Thus, the theorem holds for any node selection policy.\end{proof}

\families*

\begin{proof}
For ease of notation in this proof, we will drop $\tree$ from the input of the functions $\score_1$ and $\score_2$ since the scoring rules do not depend on $\tree$, they only depend on the input MILP instance and variable.

For any constant $\gamma \geq 1$, let $\vec{c}_1 = \gamma(1, 2, \dots, n-3)$ and let $\vec{c}_2 = \gamma\left(0, 1.5, 3-\frac{1}{2\mu^*}\right)$. Let $\vec{c} = \left(\vec{c}_1, \vec{c}_2\right) \in \R^n$ be the concatenation of $\vec{c}_1$ and $\vec{c}_2$. Next, define the $n$-dimensional vectors $\vec{a}_1 = 2\sum_{i = 1}^{n-3} \vec{e}_i$ and $\vec{a}_2 = 2\sum_{i = 1}^3 \vec{e}_{n-3+i}$, and let $A$ be a matrix whose first row is $\vec{a}_1$ and second row is $\vec{a}_2$. Let $Q_{\gamma,n}$ be the MILP
\[\begin{array}{ll}
\text{maximize}&\vec{c} \cdot \vec{x}\\
\text{subject to} & A\vec{x} = (n-3, 3)^{\top}\\
& \vec{x} \in \{0,1\}^n.
\end{array}\] We define $\mathcal{F}_{n, \mu^*} = \left\{Q_{n, \gamma} : \gamma \geq 1\right\}.$

\begin{example}\label{ex:wc}
If $\gamma = 1$ and $n = 8$, then $Q_{\gamma, n}$ is \[\begin{array}{ll}
\text{maximize}&\left(1, 2, 3, 4, 5, 0, 1.5, 3-\frac{1}{2\mu^*}\right) \cdot \vec{x}\\
\text{subject to} & \begin{pmatrix}
2 & 2 & 2 & 2 & 2 & 0 & 0 & 0\\
0 & 0 & 0 & 0 & 0 & 2 & 2 & 2
\end{pmatrix}\vec{x} = \begin{pmatrix}
5\\ 3
\end{pmatrix}\\
& \vec{x} \in \{0,1\}^8.
\end{array}\]\end{example}

For every even $n \geq 6$, both of the constraints $\vec{a}_1 \cdot \vec{x} = 2\sum_{i = 1}^{n-3} x_i = n-3$ and $\vec{a}_2 \cdot \vec{x} = 2\sum_{i = 1 }^3 x_{n-3+i} = 3$ are infeasible for $\vec{x} \in \{0,1\}^n$ since $2\sum_{i = 1 }^{n-3} x_i$ and $2\sum_{i = 1}^3 x_{n-3+i}$ are even numbers but 3 and $n-3$ are odd numbers. The key idea of this proof is that if branch-and-bound branches on all variables in $\{x_{n-2}, x_{n-1}, x_n\}$ first, it will terminate upon making a tree of size at most $2^3 = 8$, since at most three branches are necessary to determine that $\vec{a}_1 \cdot \vec{x} = 2\sum_{i = 1}^3 x_{n-3+i} = 3$ is infeasible. However, if branch-and-bound branches on all variables in $\{x_1, \dots, x_{n-3}\}$ first, it will create a tree with exponential size before it terminates.

\begin{lemma}\label{lem:mu_big}
Suppose $\mu < \mu^*$. Then for any MILP $Q_{\gamma, n} \in \fclass_{n, \mu^*}$, $\mu \score_1 + (1-\mu)\score_2$ branches on all variables in $\{x_{n-2}, x_{n-1}, x_n\}$ before branching on variables in $\{x_1, \dots, x_{n-3}\}$.
\end{lemma}

\begin{proof}[Proof of Lemma~\ref{lem:mu_big}]
For ease of notation, for the remainder of this proof, we drop the subscript $(\gamma, n)$ from $Q_{\gamma, n}$ and denote this MILP as $Q$. We first need to determine the form of $\breve{\vec{x}}_Q$, which is the optimal solution to the LP relaxation of $Q$.
 It is easiest to see how the LP relaxation will set the variables $x_{n-2}$, $x_{n-1}$, and $x_n$. The only constraints on these variables are that $2(x_{n-2} + x_{n-1} + x_n) = 3$ and that $x_{n-2}, x_{n-1}, x_n \in [0,1]$. Recall that $\vec{c} = \left(1, 2, \dots, n-3, 0, 1.5, 3-\frac{1}{2\mu^*}\right)$ is the vector defining the objective value of $Q$. Since $\mu^* > 1/3$, we know that $1.5 < 3-\frac{1}{2\mu^*}$, which means that $c[n-2] < c[n-1] < c[n]$. Since the goal is to maximize $\vec{c} \cdot \vec{x}$, the LP relaxation will set $x_{n-2} = 0$, $x_{n-1} = \frac{1}{2}$, and $x_n = 1$. The logic for the first $n-3$ variables is similar. In this case, $c[1] < c[2] < \cdots < c[n-3]$, so the LP relaxation's solution will put as much weight as possible on the variable $x_{n-3}$, then as much weight as possible on the variable $x_{n-4}$, and so on, putting as little weight as possible on the variable $x_1$ since it has the smallest corresponding objective coefficient $c[1]$. Since the only constraints on these variables are that $2\sum_{i = 1}^{n-3} x_i = n-3$ and $x_1, \dots, x_{n-3} \in [0,1]$, the LP objective value can set $\left \lfloor \frac{n-3}{2} \right\rfloor$ of the variables to 1, it can set one variable to $\frac{1}{2}$, and it has to set the rest of the variables to 0. Letting $i = \left \lceil \frac{n-3}{2} \right\rceil$, this means the LP relaxation will set the first $i - 1$ variables $x_{1}, \cdots x_{i-1}$ to zero, it will set $x_i = \frac{1}{2}$, and it will set $x_{i + 1}, \dots, x_{n-3}$ to 1. In other words, \[\breve{x}_{Q}[j]
= \begin{cases} 0 &\text{if } j \leq \left\lfloor (n-3)/2 \right\rfloor \text{ or } j = n-2\\
\frac{1}{2} &\text{if } j = \left\lceil (n-3)/2 \right\rceil \text{ or } j = n-1\\
1 &\text{if } \left\lceil (n-3)/2 \right\rceil \leq j \leq n-3 \text{ or } j = n.
\end{cases}\]
For example, if $n = 8$, then $\breve{\vec{x}}_{Q} = \left(0, 0, \frac{1}{2}, 1, 1, 0, \frac{1}{2}, 1\right)$. (See Figure~\ref{fig:lb_Q}.) Therefore, the only candidate variables to branch on are $x_{n-1}$ and $x_i$ where again, $i = \left\lceil (n-3)/2 \right\rceil$.

\begin{figure}[t]
\centering
\begin{subfigure}{\textwidth}
\includegraphics[scale=.8]{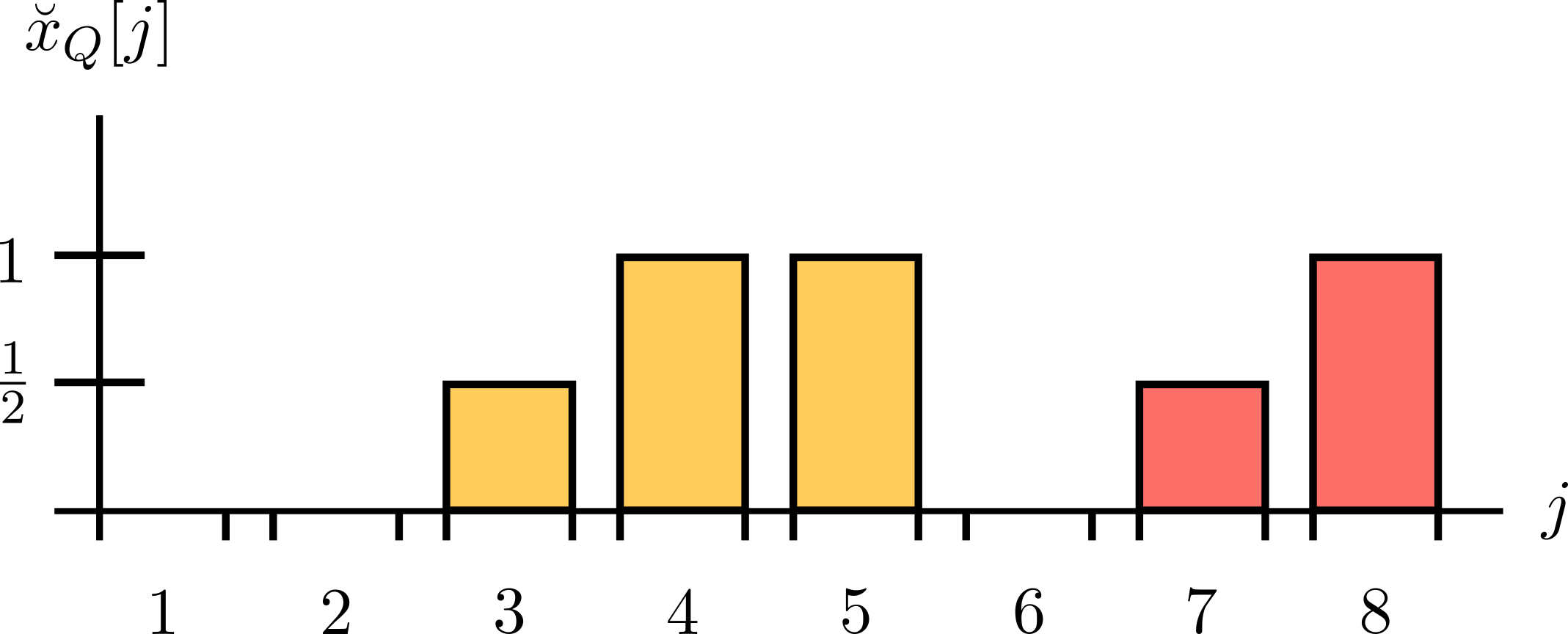} \centering
\caption{}\label{fig:lb_Q}
\end{subfigure}\newline
\begin{subfigure}{0.45\textwidth}
\includegraphics[scale=.8]{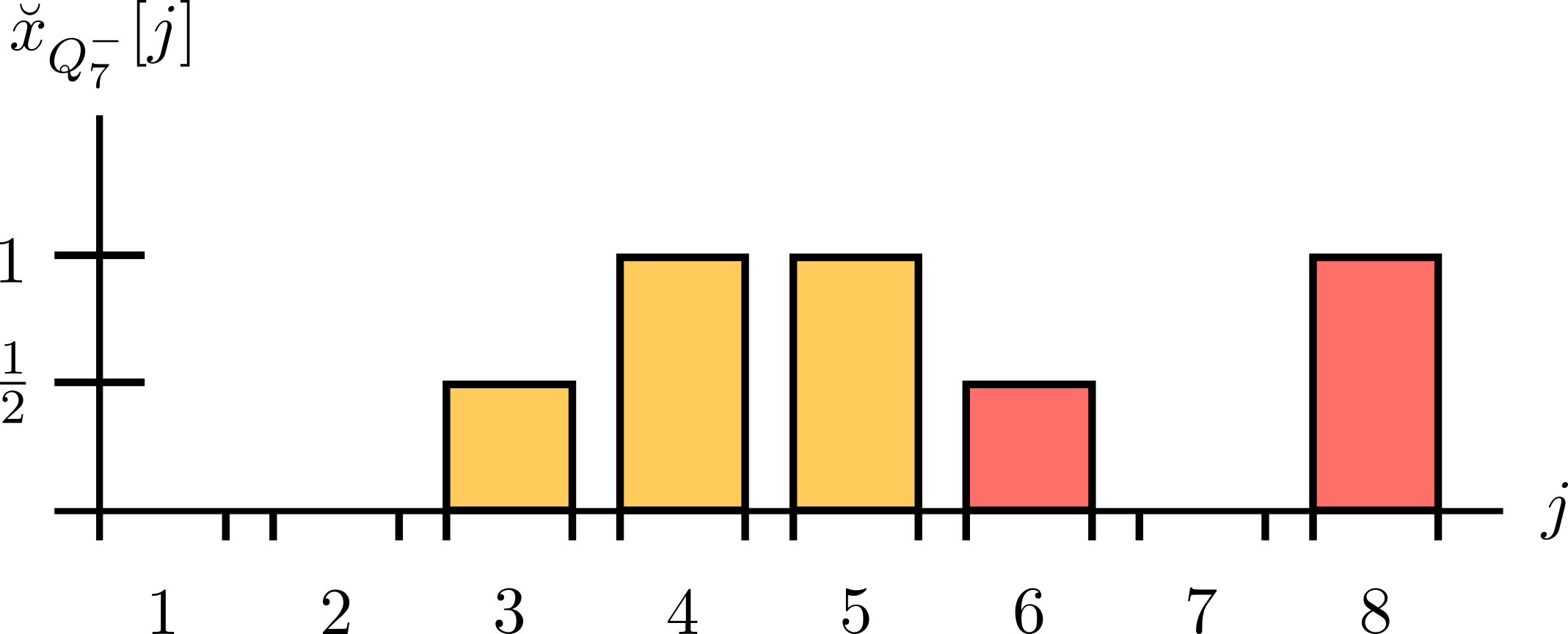}\centering
\caption{}\label{fig:lb_Q7m}
\end{subfigure}\qquad
\begin{subfigure}{0.45\textwidth}
\includegraphics[scale=.8]{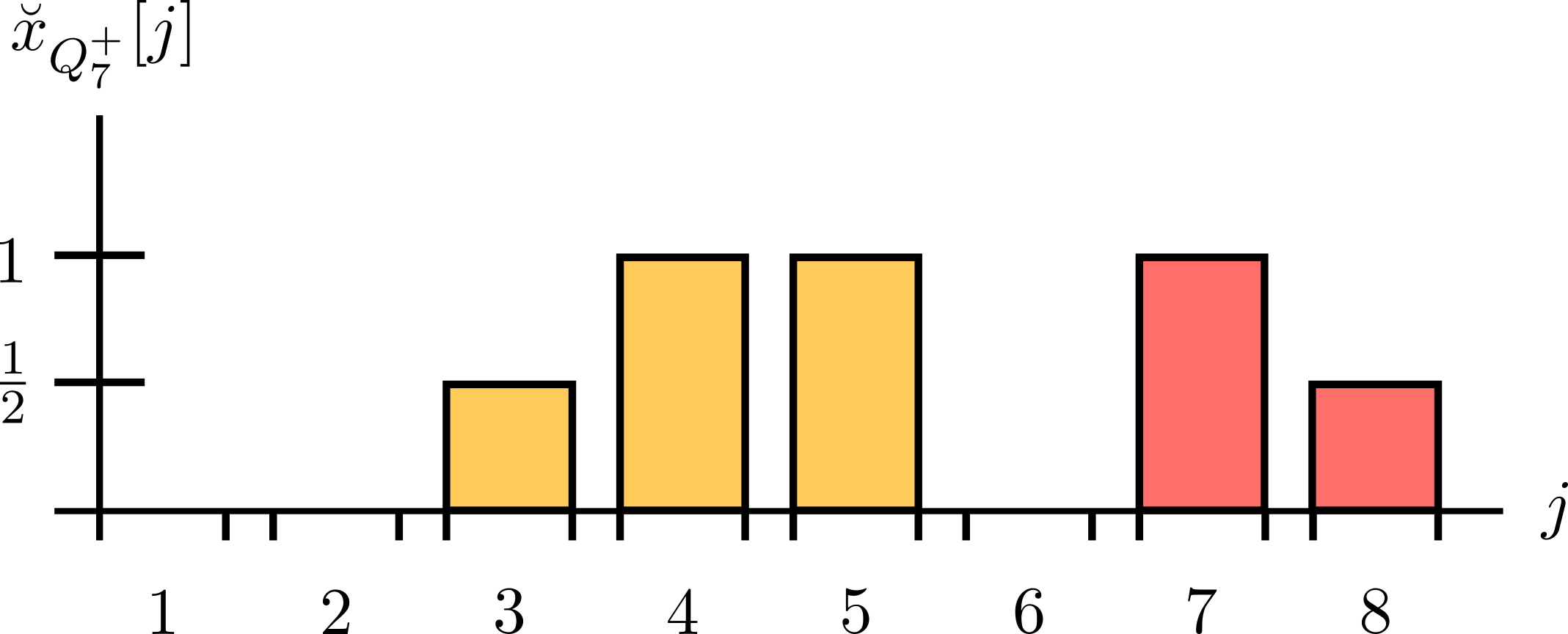}\centering
\caption{}\label{fig:lb_Q7p}
\end{subfigure}
\begin{subfigure}{0.45\textwidth}
\includegraphics[scale=.8]{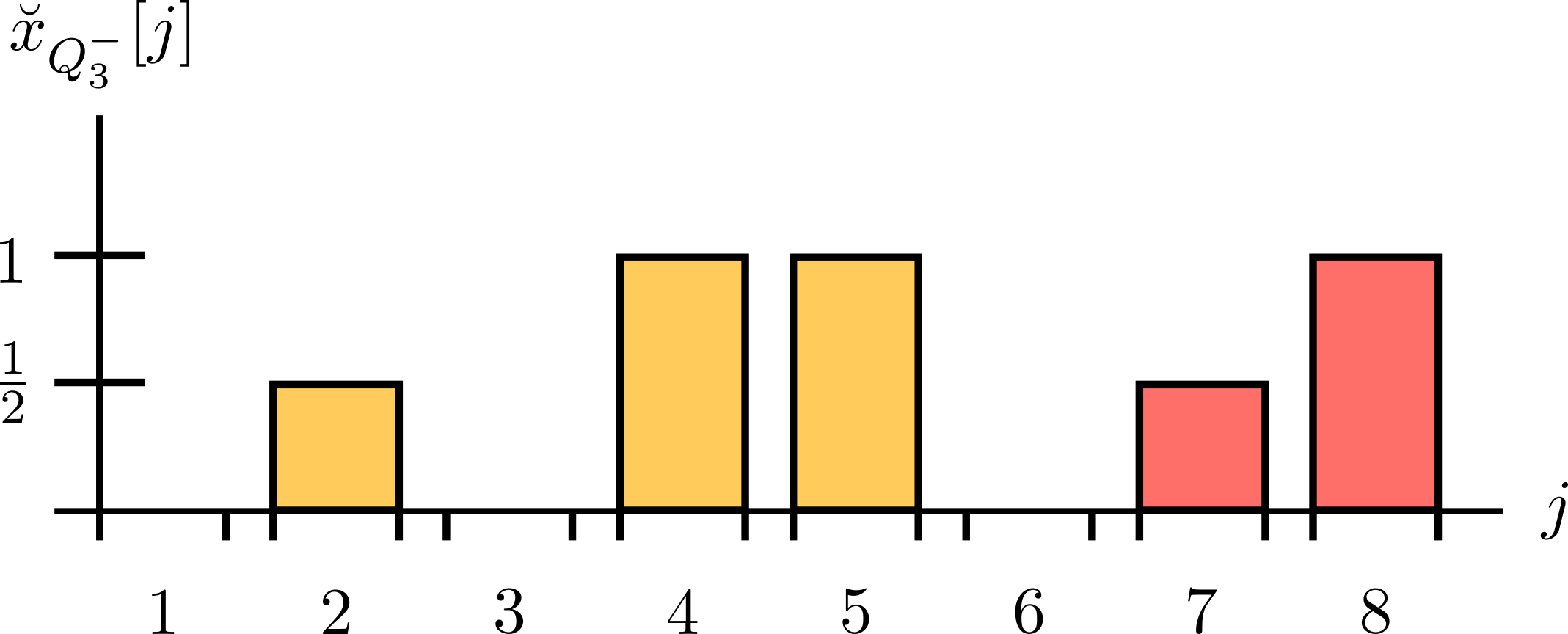}\centering
\caption{}\label{fig:lb_Q3m}
\end{subfigure}\qquad
\begin{subfigure}{0.45\textwidth}
\includegraphics[scale=.8]{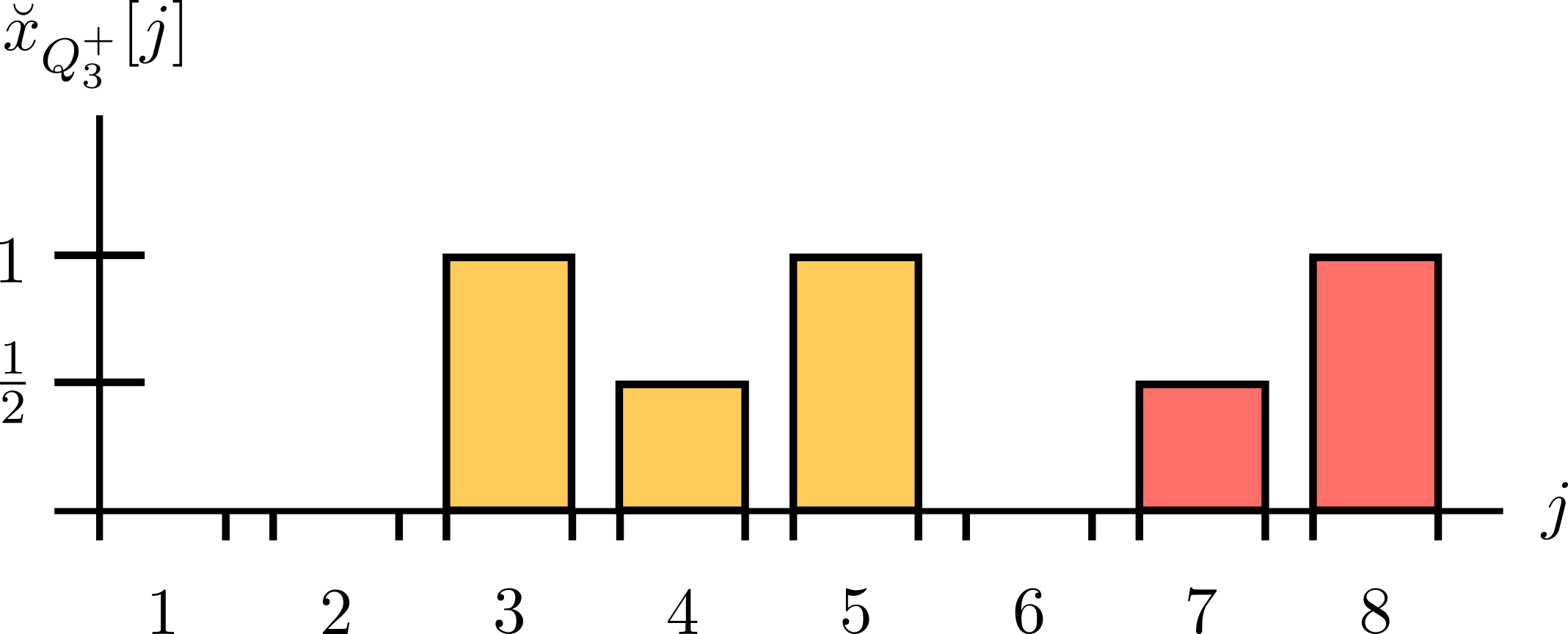}\centering
\caption{}\label{fig:lb_Q3p}
\end{subfigure}
\caption{Illustrations to accompany the proof of Lemma~\ref{lem:mu_big} when $n = 8$. For each $j$ on the $x$-axis, the histogram gives the value of either $\breve{x}_Q[j]$ (Figure~\ref{fig:lb_Q}), $\breve{x}_{Q_7^-}[j]$ (Figure~\ref{fig:lb_Q7m}), $\breve{x}_{Q_7^+}[j]$ (Figure~\ref{fig:lb_Q7p}), $\breve{x}_{Q_3^-}[j]$ (Figure~\ref{fig:lb_Q3m}), or $\breve{x}_{Q_3^+}[j]$ (Figure~\ref{fig:lb_Q3p}). In this case, $i = 3$.}\label{fig:lb_graphs}
\end{figure}

To determine when variable B\&B will branch on, we need to calculate $\breve{\vec{x}}_{Q_i^-}$ which is the solution to the LP relaxation of $Q$ with the additional constraint that $x_i = 0$, as well as $x_{Q_i^+}$ which is the solution to the LP relaxation of $Q$ with the additional constraint that $x_i = 1$, and $x_{Q_{n-1}^-}$ and $x_{Q_{n-1}^+}$.
First, we will determine the form of the vectors $\breve{\vec{x}}_{Q_{n-1}^-}$ and $\breve{\vec{x}}_{Q_{n-1}^+}$. Suppose we set $x_{n-1}=0$. We now need to ensure that $2(x_{n-2} + 0 + x_n) = 3$ and that $x_{n-2}, x_n \in [0,1]$. Since $c[n-2] = 0 < 3 - \frac{1}{2\mu^*} = c[n]$, the solution to the LP relaxation of $Q_{n-1}^-$ will set $x_{n-2} = \frac{1}{2}$ and $x_n = 1$. (See Figure~\ref{fig:lb_Q7m}.) Similarly, if we set $x_{n-1}=1$, we now need to ensure that $2(x_{n-2} + 1 + x_n) = 3$ and that $x_{n-2}, x_n \in [0,1]$. Therefore, the solution to the LP relaxation of $Q_{n-1}^-$ will set $x_{n-2} = 0$ and $x_n = \frac{1}{2}$. (See Figure~\ref{fig:lb_Q7p}.) In other words, $\breve{\vec{x}}_{Q_{n-1}^-} = \breve{\vec{x}}_{Q} - \frac{1}{2}\vec{e}_{n-1} + \frac{1}{2}\vec{e}_{n-2}$ and $\breve{\vec{x}}_{Q_{n-1}^+} = \breve{\vec{x}}_{Q} + \frac{1}{2}\vec{e}_{n-1} - \frac{1}{2}\vec{e}_{n}$.

The argument for $\breve{\vec{x}}_{Q_i^-}$ and $\breve{\vec{x}}_{Q_i^-}$ is similar. Recall that in $\breve{\vec{x}}_Q$, the solution to the LP relaxation of the original MIP $Q$, we have that $\breve{x}_Q[i] = \frac{1}{2}$ since it is the median of the variables $x_1, \dots, x_{n-3}$. For all $j > i$, we have that $\breve{x}_Q[j] = 1$ and for all $j < i$, we have that $\breve{x}_Q[j] = 0$. Suppose we set $x_i = 0$. As before, the LP relaxation's solution will put as much weight as possible on the variable $x_{n-3}$, then as much weight as possible on the variable $x_{n-4}$, and so on, putting as little weight as possible on the variable $x_1$ since it has the smallest corresponding objective coefficient $c[1]$. Since it cannot set $x_i = \frac{1}{2}$, it will set the next-best variable to $\frac{1}{2}$, which is $x_{i-1}$. (See Figure~\ref{fig:lb_Q3m}.) In other words, $\breve{\vec{x}}_{Q_i^-} = \breve{\vec{x}}_{Q} - \frac{1}{2}\vec{e}_i + \frac{1}{2}\vec{e}_{i-1}$. If we set $x_i = 1$, the LP relaxation's solution will have to take some weight away from the variables $x_{i+1}, \dots, x_{n-3}$ since it needs to ensure that $2\sum_{i = 1}^{n-3} x_i = n-3$. Therefore, it will set $x_{i+1}$ to $\frac{1}{2}$ and $x_j$ to 1 for all $j > i+1$. (See Figure~\ref{fig:lb_Q3p}.)
In other words, $\breve{\vec{x}}_{Q_i^+} = \breve{\vec{x}}_{Q} + \frac{1}{2}\vec{e}_i - \frac{1}{2}\vec{e}_{i+1}$.

 Therefore, \begin{align*}\breve{c}_{Q_i^-} &= \breve{\vec{c}}_{Q} - \frac{1}{2}\left\lceil \frac{n-3}{2} \right\rceil + \frac{1}{2}\left\lfloor \frac{n-3}{2} \right\rfloor = \breve{c}_{Q} - \frac{\gamma}{2},\\
\breve{\vec{c}}_{Q_i^+} &= \breve{\vec{c}}_{Q} +  \frac{1}{2}\left\lceil \frac{n-3}{2} \right\rceil - \frac{1}{2}\left(\left\lceil \frac{n-3}{2} \right\rceil + 1\right) = \breve{c}_{Q} - \frac{\gamma}{2},\\
\breve{c}_{Q_{n-1}^-} &=\breve{\vec{c}}_{Q} - \frac{3\gamma}{4}, \text{ and }\\
\breve{c}_{Q_{n-1}^+} &=\breve{\vec{c}}_{Q} +\frac{3}{4} - \frac{1}{2}\left(3 - \frac{1}{2\mu^*}\right) = \breve{\vec{c}}_{Q} - \frac{\gamma}{4}\left(3 - \frac{1}{\mu^*}\right).
\end{align*} This means that $\breve{c}_Q - \breve{c}_{Q_i^-} = \breve{c}_Q - \breve{c}_{Q_i^+} = \gamma/2$, $\breve{c}_Q - \breve{c}_{Q_{n-1}^-} = 3\gamma/4$, and $\breve{c}_Q - \breve{c}_{Q_{n-1}^+} = \frac{\gamma}{4}\left(3 - \frac{1}{\mu^*}\right)$. Therefore, $\mu\score_1(Q,i) + (1-\mu)\score_2(Q,i) = \gamma/2$ and $\mu\score_1(Q,n-1) + (1-\mu)\score_2(Q,n-1) = \frac{\mu\gamma}{4}\left(3 - \frac{1}{\mu^*}\right) + \frac{3\gamma(1-\mu)}{4} = \frac{3\gamma}{4} -  \frac{\mu\gamma}{4\mu^*}.$ This means that $\mu\score_1(Q,i) + (1-\mu)\score_2(Q,i) = \gamma/2 <\frac{3\gamma}{4} -  \frac{\mu\gamma}{4\mu^*} = \mu\score_1(Q,n-1) + (1-\mu)\score_2(Q,n-1)$ so long as $\mu < \mu^*$.

The next node branch-and-bound will explore is $Q_{n-1}^-$. The vector $\breve{\vec{x}}_{Q_{n-1}^-}$ has fractional values only in positions $i$ and $n-2$. Branching on $i$, we again have that $\breve{c}_{Q_{n-1}^-} - \breve{c}_{Q_{n-1,i}^{-,-}} = \breve{c}_{Q_{n-1}^-} - \breve{c}_{Q_{n-1,i}^{-,+}} = \gamma/2$.
Branching on $n-2$, $Q_{n-1, n-2}^{-,-}$ is infeasible, so $\breve{c}_{Q_{n-1}^-} - \breve{c}_{Q_{n-1,n-2}^{-,-}}$ equals some large number $B \geq ||\vec{c}||_1$. Next, $\breve{\vec{x}}_{Q_{n-1, n-2}^{-,+}} = \breve{\vec{x}}_{Q_{n-1}^-} + \frac{1}{2}\vec{e}_{n-2} - \frac{1}{2}\vec{e}_n,$ so $\breve{c}_{Q_{n-1, n-2}^{-,+}} = \breve{c}_{Q_{n-1}^-} - \frac{\gamma}{2}\left(3 - \frac{1}{2\mu^*}\right).$ Therefore, $\mu\score_1\left(Q_{n-1}^-,i\right) + (1-\mu)\score_2\left(Q_{n-1}^-,i\right) = \gamma/2$ and \begin{align*}\mu\score_1\left(Q_{n-1}^-,n-2\right) + (1-\mu)\score_2\left(Q_{n-1}^-,n-2\right) &= \frac{\mu\gamma}{2}\left(3 - \frac{1}{2\mu^*}\right) + (1 - \mu)B\\
&= B + \mu\left(\frac{3\gamma}{2} - \frac{\gamma}{4\mu^*} - B\right)\\
&\geq B + \mu^*\left(\frac{3\gamma}{2} - \frac{\gamma}{4\mu^*} - B\right)\\
&= B - \frac{\gamma}{4} + \mu^*\left(\frac{3\gamma}{2} - B\right)\\
&> B - \frac{\gamma}{4} + \frac{3\gamma/4 - B}{3\gamma/2 - B}\left(\frac{3\gamma}{2} - B\right)\\
&= B - \frac{\gamma}{4} + \frac{3\gamma}{4} - B\\
&= \frac{\gamma}{2},\\
\end{align*} where the final inequality holds because $\mu^* < 1 < \frac{B-3\gamma/4}{B - 3\gamma/2}$. Therefore, $x_{n-2}$ will be branched on next.

Since $Q_{n-1, n-2}^{-,-}$ is infeasible, the next node branch-and-bound will explore is $Q_{n-1, n-2}^{-,+}$. The vector $\breve{\vec{x}}_{Q_{n-1, n-2}^{-, +}}$ has fractional values only in positions $i$ and $n$. Both MILP instances $Q_{n-1, n-2, n}^{-, +, -}$ and $Q_{n-1, n-2, n}^{-, +, +}$ are infeasible, so $\mu\score_1\left(Q_{n-1, n-2}^{-, +},n\right) + (1-\mu)\score_2\left(Q_{n-1, n-2}^{-, +},n\right) = B$ whereas $\mu\score_1\left(Q_{n-1, n-2}^{-, +},i\right) + (1-\mu)\score_2\left(Q_{n-1, n-2}^{-, +},i\right) = \gamma/2$, as before. Therefore, branch-and-bound will branch on $x_n$ and fathom both children.

The next node branch-and-bound will explore is $Q_{n-1}^+$. The vector $\breve{\vec{x}}_{Q_{n-1}^+}$ has fractional values only in positions $i$ and $n$. Branching on $i$, we again have that $\breve{c}_{Q_{n-1}^+} - \breve{c}_{Q_{n-1,i}^{+,-}} = \breve{c}_{Q_{n-1}^+} - \breve{c}_{Q_{n-1,i}^{+,+}} = \gamma/2$.
Branching on $x_n$, $\breve{\vec{x}}_{Q_{n-1, n}^{+,-}} = \breve{\vec{x}}_{Q_{n-1}^+} - \frac{1}{2}\vec{e}_{n} + \frac{1}{2}\vec{e}_{n-2},$ so $\breve{c}_{Q_{n-1, n}^{+,-}} = \breve{c}_{Q_{n-1}^+} - \frac{\gamma}{2}\left(3 - \frac{1}{2\mu^*}\right).$
Meanwhile, $Q_{n-1, n}^{+,+}$ is infeasible, so $\breve{c}_{Q_{n-1}^+} - \breve{c}_{Q_{n-1,n-2}^{+,+}} = B$.
 Therefore, $\mu\score_1\left(Q_{n-1}^+,i\right) + (1-\mu)\score_2\left(Q_{n-1}^+,i\right) = \gamma/2$ and $\mu\score_1\left(Q_{n-1}^+,n\right) + (1-\mu)\score_2\left(Q_{n-1}^+,n\right) = \frac{\mu\gamma}{2}\left(3 - \frac{1}{2\mu^*}\right) + (1 - \mu)B > \gamma/2$. Therefore, $x_n$ will be branched on next.

The next node branch-and-bound will explore is $Q_{n-1, n}^{+,-}$. The vector $\breve{\vec{x}}_{Q_{n-1, n}^{+,-}}$ has fractional values only in positions $i$ and $n-2$. Both MILP instances $Q_{n-1, n, n-2}^{+, -, -}$ and $Q_{n-1, n, n-2}^{+,-,-}$ are infeasible, so $\mu\score_1\left(Q_{n-1, n}^{-, +},n-2\right) + (1-\mu)\score_2\left(Q_{n-1, n}^{-, +},n-2\right) = B$ whereas $\mu\score_1\left(Q_{n-1, n-2}^{-, +},i\right) + (1-\mu)\score_2\left(Q_{n-1, n-2}^{-, +},i\right) = \gamma/2$, as before. Therefore, branch-and-bound will branch on $x_{n-2}$ and fathom both children.

At this point, all children have been fathomed, so branch-and-bound will terminate.
\end{proof}

\begin{lemma}\label{lem:mu_small}
Suppose $\mu > \mu^*$. Then for any MILP $Q_{\gamma, n} \in \fclass_{n, \mu^*}$, branch-and-bound with the scoring rule $\mu \score_1 + (1-\mu)\score_2$ will create a tree of depth at least $2^{(n-5)/4}$.
\end{lemma}

\begin{proof}[Proof of Lemma~\ref{lem:mu_small}]
To prove this lemma, we use induction to show that on any path from the root of the branch-and-bound tree to a node of depth $i = \left\lfloor (n-3)/2 \right\rfloor$, if $J$ are the set of indices branched on along that path, then $J \subseteq \{x_1, \dots, x_{n-3}\}$. Even after branching on $i$ nodes from $\{x_1, \dots, x_{n-3}\}$, the MILP will still be feasible, so the branch will not yet have been fathomed (since the original MILP is infeasible, a node will be fathomed only when it is infeasible). Therefore, branch-and-bound will continue down every branch to depth $\left\lfloor (n-3)/2 \right\rfloor$, thus creating a tree with $2^{(n-4)/2}$ nodes.

\begin{claim}\label{claim:path}
On any path from the root of the branch-and-bound tree to a node of depth $i = \left\lfloor (n-3)/2 \right\rfloor$, if $J$ are the set of indices branched on along that path, then $J \subseteq \{x_1, \dots, x_{n-3}\}$.
\end{claim}

\begin{proof}[Proof of Claim~\ref{claim:path}]We prove this claim by induction.

\bigskip
\noindent\textbf{Inductive hypothesis.} For $j \leq \left\lfloor (n-3)/2 \right\rfloor$, let $J$ be the set of indices branched on along an arbitrary path of the branch-and-bound tree from the root to a node of depth $j$. Then $J \subseteq \{x_1, \dots, x_{n-3}\}$.

\bigskip
\noindent\textbf{Base case ($j=0$).} As we saw in the proof of Lemma~\ref{lem:mu_big}, if $\mu > \mu^*$, then branch-and-bound will first branch on $x_i$ where $i = \left\lceil (n-3)/2 \right\rceil$.

\bigskip
\noindent\textbf{Inductive step.} Let $j$ be an arbitrary index such that $0 \leq j \leq i-1$. Let $J$ be the set of indices branched on along an arbitrary path of the branch-and-bound tree from the root to a node of depth $j$. We know from the inductive hypothesis that $J \subseteq \{x_1, \dots, x_{n-3}\}$. Let $Q'$ be the MILP at that node. Since $j \leq \left\lfloor (n-3)/2 \right\rfloor - 1$, we know that the LP relaxation of $Q'$ is feasible. Let $z$ be the number of variables set to zero in $J$ and let $x_{p_1}, x_{p_2}, \dots, x_{p_t}$ be $\left\{x_1, \dots, x_{n-3}\right\} \setminus J$ ordered such that $p_k < p_{k'}$ for $k < k'$. We know that the solution to the LP relaxation of $Q'$ will have the first $i' := \left\lfloor (n-3)/2 \right\rfloor - z$ variables $x_{p_1}, \dots, x_{p_{i'}}$ set to 0, it will set $x_{p_{i'+1}}$ to 1/2, and it will set the remaining variables in $\left\{x_1, \dots, x_{n-3}\right\} \setminus J$ to 1. Thus, the fractional variables are $x_{p_{i'+1}}$ and $x_{n-1}$. Note that since $z \leq  |J| \leq \left\lfloor (n-3)/2 \right\rfloor - 1$, $i' = \left\lfloor (n-3)/2 \right\rfloor - z \geq 1$.

Suppose we branch on $x_{p_{i'+1}}$. If we set $x_{p_{i'+1}} = 0$, then the LP relaxation of $(Q')_{p_{i'+1}}^-$ will set $x_{p_{i'}}$ to be 1/2 and otherwise the optimal solution will remain unchanged. Thus, $\breve{c}_{Q'} - \breve{c}_{(Q')_{p_{i' + 1}}^-} = 
\breve{c}_{Q'} - \left(\breve{c}_{Q'} - \gamma p_{i'+1}/2 + \gamma p_{i'}/2\right) = \frac{\gamma\left(p_{i'+1} - p_{i'}\right)}{2}$. Meanwhile, if we set $x_{p_{i'+1}} = 1$, then the LP relaxation of $(Q')_{p_{i'+1}}^-$ will set $x_{p_{i'+2}}$ to be 0 and otherwise the optimal solution will remain unchanged. Thus, $\breve{c}_{Q'} - \breve{c}_{(Q')_{p_{i' + 1}}^+} = \breve{c}_{Q'} - \left(\breve{c}_{Q'} + \gamma p_{i'+1}/2 - \gamma p_{i'+2}/2\right) = \frac{\gamma\left(p_{i'+2} - p_{i'+1}\right)}{2}$. Suppose that $\breve{c}_{Q'} - \breve{c}_{(Q')_{p_{i' + 1}}^+} > \breve{c}_{Q'} - \breve{c}_{(Q')_{p_{i' + 1}}^-}$. Then \begin{align*}\mu\score_1\left(Q', p_{i'+1}\right) + (1-\mu)\score_2\left(Q', p_{i'+1}\right) &= \frac{\gamma}{2} \left(\mu \left(p_{i'+1} - p_{i'}\right) + (1-\mu)\left(p_{i'+2} - p_{i'+1}\right) \right)\\
&\geq \frac{\gamma}{2} \left(\mu + (1-\mu)\right)\\
&= \frac{\gamma}{2}.
\end{align*}

 Meanwhile, suppose that $\breve{c}_{Q'} - \breve{c}_{(Q')_{p_{i' + 1}}^+} \leq \breve{c}_{Q'} - \breve{c}_{(Q')_{p_{i' + 1}}^-}$. Then \begin{align*}\mu\score_1\left(Q', p_{i'+1}\right) + (1-\mu)\score_2\left(Q', p_{i'+1}\right) &=  \frac{\gamma}{2} \left(\mu \left(p_{i'+2} - p_{i'+1}\right) + (1-\mu)\left( p_{i'+1} - p_{i'}\right) \right)\\
&\geq \frac{\gamma}{2} \left(\mu + (1-\mu)\right)\\
&= \frac{\gamma}{2}.
\end{align*} Meanwhile, as in the proof of Lemma~\ref{lem:mu_big}, $\mu\score_1(Q',n-1) + (1-\mu)\score_2(Q',n-1) = \frac{3\gamma}{4} -  \frac{\mu\gamma}{4\mu^*} < \frac{\gamma}{2}$ so long as $\mu > \mu^*$. Thus, branch-and-bound will branch next on $x_{p_{i'}}$.
\end{proof}
\end{proof}

\end{proof}

\begin{theorem}\label{thm:families2}
Let \[\score_1(Q, i) = \min\left\{\breve{c}_Q - \breve{c}_{Q_i^+}, \breve{c}_Q - \breve{c}_{Q_i^-}\right\} \textnormal{ and } \score_2(Q, i) = \max\left\{\breve{c}_Q - \breve{c}_{Q_i^+}, \breve{c}_Q - \breve{c}_{Q_i^-}\right\}.\] For all even $n \geq 6$ and all $\mu^* \in \left(\frac{1}{3}, \frac{2}{3}\right)$, there exists an infinite family $\mathcal{G}_{n, \mu^*}$ of MILP instances such that for any $Q \in \mathcal{G}_{n, \mu^*}$, if $\mu \in \left[0,\mu^*\right)$, then the scoring rule $\mu \score_1 + (1-\mu)\score_2$ results in a branch-and-bound tree with $\Omega\left(2^{(n-5)/4}\right)$ nodes and if $\mu \in \left(\mu^*,1\right]$, the scoring rule results a tree with $O(1)$ nodes.
\end{theorem}

\begin{proof}
For any constant $\gamma \geq 1$, let $\vec{c}_1 \in \R^{n-3}$ be a vector such that \[c_1[i] = \begin{cases} 0 &\text{if } i < (n-3)/2\\
1.5 &\text{if } i = \left\lceil (n-3)/2\right\rceil\\
3 - \frac{1}{2\mu^*} &\text{if } i > (n-3)/2 + 1
\end{cases}\] and let $\vec{c}_2 = \left(1, 2, 3\right)$. Let $\vec{c} = \gamma\left(\vec{c}_1, \vec{c}_2\right) \in \R^n$ be the concatenation of $\vec{c}_1$ and $\vec{c}_2$ multiplied with $\gamma$. For example, if $\gamma = 1$ and $n = 8$, then $\vec{c} = \left(0,0, 1.5, 3 - \frac{1}{2\mu^*}, 3 - \frac{1}{2\mu^*}, 1, 2, 3\right)$ Next, define the $n$-dimensional vectors $\vec{a}_1 = 2\sum_{i = 1}^{n-3} \vec{e}_i$ and $\vec{a}_2 = 2\sum_{i = 1}^3 \vec{e}_{n-3+i}$, and let $A$ be a matrix whose first row is $\vec{a}_1$ and second row is $\vec{a}_2$. Let $Q_{\gamma,n}$ be the MILP
\[\begin{array}{ll}
\text{maximize}&\vec{c} \cdot \vec{x}\\
\text{subject to} & A\vec{x} = (n-3, 3)^{\top}\\
& \vec{x} \in \{0,1\}^n.
\end{array}\]

We define $\mathcal{G}_{n, \mu^*} = \left\{Q_{n, \gamma} : \gamma \geq 1\right\}.$

For every even $n \geq 6$, $Q_{\gamma, n}$, both of the constraints $\vec{a}_1 \cdot \vec{x} = 2\sum_{i = 1}^{n-3} x_i = n-3$ and $\vec{a}_2 \cdot \vec{x} = 2\sum_{i = 1 }^3 x_{n-3+i} = 3$ are infeasible for $\vec{x} \in \{0,1\}^n$ since $2\sum_{i = 1 }^{n-3} x_i$ and $2\sum_{i = 1}^3 x_{n-3+i}$ are even numbers but 3 and $n-3$ are odd numbers. The key idea of this proof is that if branch-and-bound branches on all variables in $\{x_{n-2}, x_{n-1}, x_n\}$ first, it will terminate upon making a tree of size at most $2^3 = 8$, since at most three branches are necessary to determine that $\vec{a}_1 \cdot \vec{x} = 2\sum_{i = 1}^3 x_{n-3+i} = 3$ is infeasible. However, if branch-and-bound branches on all variables in $\{x_1, \dots, x_{n-3}\}$ first, it will create a tree with exponential size before it terminates.

\begin{lemma}\label{lem:mu_big_2}
Suppose $\mu > \mu^*$. Then for any MILP $Q_{\gamma, n} \in \mathcal{G}_{n, \mu^*}$, $\mu \score_1 + (1-\mu)\score_2$ branches on all variables in $\{x_{n-2}, x_{n-1}, x_n\}$ before branching on variables in $\{x_1, \dots, x_{n-3}\}$.
\end{lemma}

\begin{proof}[Proof of Lemma~\ref{lem:mu_big_2}]
For ease of notation, for the remainder of this proof, we drop the subscript $(\gamma, n)$ from $Q_{\gamma, n}$ and denote this MILP as $Q$.
The optimal solution to the LP relaxation of $Q$ has the following form: \[\breve{\vec{x}}_{Q}[j]
= \begin{cases} 0 &\text{if } j \leq \left\lfloor (n-3)/2 \right\rfloor \text{ or } j = n-2\\
\frac{1}{2} &\text{if } j = \left\lceil (n-3)/2 \right\rceil \text{ or } j = n-1\\
1 &\text{if } \left\lceil (n-3)/2 \right\rceil \leq j \leq n-3 \text{ or } j = n.
\end{cases}\]
For example, if $n = 8$, then $\breve{\vec{x}}_{Q} = \left(0, 0, \frac{1}{2}, 1, 1, 0, \frac{1}{2}, 1\right)$. Therefore, the only candidate variables to branch on are $x_{n-1}$ and $x_i$ where $i = \left\lceil (n-3)/2 \right\rceil$. Branching on $x_i$, we have $\breve{\vec{x}}_{Q_i^-} = \breve{\vec{x}}_{Q} - \frac{1}{2}\vec{e}_i + \frac{1}{2}\vec{e}_{i-1}$ and $\breve{\vec{x}}_{Q_i^+} = \breve{\vec{x}}_{Q} + \frac{1}{2}\vec{e}_i - \frac{1}{2}\vec{e}_{i+1}$. Branching on $x_{n-1}$, we have $\breve{\vec{x}}_{Q_{n-1}^-} = \breve{\vec{x}}_{Q} - \frac{1}{2}\vec{e}_{n-1} + \frac{1}{2}\vec{e}_{n-2}$ and $\breve{\vec{x}}_{Q_{n-1}^+} = \breve{\vec{x}}_{Q} + \frac{1}{2}\vec{e}_{n-1} - \frac{1}{2}\vec{e}_{n}$. Therefore, \begin{align*}\breve{c}_{Q_i^-} &= \breve{\vec{c}}_{Q} - \frac{3\gamma}{4},\\
\breve{\vec{c}}_{Q_i^+} &= \breve{\vec{c}}_{Q} +  \frac{3\gamma}{4} - \frac{\gamma}{2}\left(3 - \frac{1}{2\mu^*}\right) = \breve{c}_{Q} - \frac{3\gamma}{4}\left( 1- \frac{1}{\mu^*}\right),\\
\breve{c}_{Q_{n-1}^-} &=\breve{\vec{c}}_{Q} - \gamma + \frac{\gamma}{2} =\breve{\vec{c}}_{Q} - \frac{\gamma}{2}, \text{ and }\\
\breve{c}_{Q_{n-1}^+} &=\breve{\vec{c}}_{Q} +\gamma - \frac{3\gamma}{2} =\breve{\vec{c}}_{Q} - \frac{\gamma}{2}.\\
\end{align*} This means that $\breve{c}_Q - \breve{c}_{Q_i^-} = 3\gamma/4$, $\breve{c}_Q - \breve{c}_{Q_i^+} = \frac{\gamma}{4}\left(3 - \frac{1}{\mu^*}\right)$, and $\breve{c}_Q - \breve{c}_{Q_{n-1}^-} = \breve{c}_Q - \breve{c}_{Q_{n-1}^+} = \gamma/2$. Therefore, $\mu\score_1(Q,i) + (1-\mu)\score_2(Q,i) =  \frac{\mu\gamma}{4}\left(3 - \frac{1}{\mu^*}\right) + \frac{3\gamma(1-\mu)}{4} = \frac{3\gamma}{4} -  \frac{\mu\gamma}{4\mu^*}$ and $\mu\score_1(Q,n-1) + (1-\mu)\score_2(Q,n-1) = \gamma/2.$ This means that $\mu\score_1(Q,i) + (1-\mu)\score_2(Q,i) = \frac{3\gamma}{4} -  \frac{\mu\gamma}{4\mu^*} < \gamma/2 = \mu\score_1(Q,n-1) + (1-\mu)\score_2(Q,n-1)$ so long as $\mu > \mu^*$.

The next node branch-and-bound will explore is $Q_{n-1}^-$. The vector $\breve{\vec{x}}_{Q_{n-1}^-}$ has fractional values only in positions $i$ and $n-2$. Branching on $i$, we again have that $\breve{c}_{Q_{n-1}^-} - \breve{c}_{Q_{n-1,i}^{-,-}} = 3\gamma/4$ and $\breve{c}_{Q_{n-1}^-} - \breve{c}_{Q_{n-1,i}^{-,+}} = \frac{\gamma}{4}\left(3 - \frac{1}{\mu^*}\right)$.
Branching on $n-2$, $Q_{n-1, n-2}^{-,-}$ is infeasible, so $\breve{c}_{Q_{n-1}^-} - \breve{c}_{Q_{n-1,n-2}^{-,-}}$ equals some large number $B \geq ||\vec{c}||_1$. Next, $\breve{\vec{x}}_{Q_{n-1, n-2}^{-,+}} = \breve{\vec{x}}_{Q_{n-1}^-} + \frac{1}{2}\vec{e}_{n-2} - \frac{1}{2}\vec{e}_n,$ so $\breve{c}_{Q_{n-1, n-2}^{-,+}} = \breve{c}_{Q_{n-1}^-} - \gamma.$ Therefore, $\mu\score_1\left(Q_{n-1}^-,i\right) + (1-\mu)\score_2\left(Q_{n-1}^-,i\right) = \frac{3\gamma}{4} -  \frac{\mu\gamma}{4\mu^*}$ and \begin{align*}\mu\score_1\left(Q_{n-1}^-,n-2\right) + (1-\mu)\score_2\left(Q_{n-1}^-,n-2\right) &= \mu\gamma + (1 - \mu)B\\
&= B + \mu\left(\gamma-B\right)\\
&= B - \frac{\mu\gamma}{4\mu^*} + \mu\left(\gamma + \frac{\gamma}{4\mu^*} -B\right)\\
& > B - \frac{\mu\gamma}{4\mu^*} + \frac{3\gamma/4 - B}{\gamma + \frac{\gamma}{4\mu^*} -B}\left(\gamma + \frac{\gamma}{4\mu^*} -B\right)\\
&= B - \frac{\mu\gamma}{4\mu^*} + 3\gamma/4 - B\\
&= \frac{3\gamma}{4} -  \frac{\mu\gamma}{4\mu^*},
\end{align*} where the final inequality holds because $\mu < 1 < \frac{B - 3\gamma/4}{B - \left(\gamma + \gamma/\left(4\mu^*\right)\right)}$. Therefore, $x_{n-2}$ will be branched on next.

Since $Q_{n-1, n-2}^{-,-}$ is infeasible, the next node branch-and-bound will explore is $Q_{n-1, n-2}^{-,+}$. The vector $\breve{\vec{x}}_{Q_{n-1, n-2}^{-, +}}$ has fractional values only in positions $i$ and $n$. Since both MILP instances $Q_{n-1, n-2, n}^{-, +, -}$ and $Q_{n-1, n-2, n}^{-, +, +}$ are infeasible, so $\mu\score_1\left(Q_{n-1, n-2}^{-, +},n\right) + (1-\mu)\score_2\left(Q_{n-1, n-2}^{-, +},n\right) = B$ whereas $\mu\score_1\left(Q_{n-1, n-2}^{-, +},i\right) + (1-\mu)\score_2\left(Q_{n-1, n-2}^{-, +},i\right) = \frac{3\gamma}{4} -  \frac{\mu\gamma}{4\mu^*} < B$. Therefore, branch-and-bound will branch on $x_n$ and fathom both children.

The next node branch-and-bound will explore is $Q_{n-1}^+$. The vector $\breve{\vec{x}}_{Q_{n-1}^+}$ has fractional values only in positions $i$ and $n$. Branching on $i$, we again have that $\breve{c}_{Q_{n-1}^-} - \breve{c}_{Q_{n-1,i}^{-,-}} = 3\gamma/4$ and $\breve{c}_{Q_{n-1}^-} - \breve{c}_{Q_{n-1,i}^{-,+}} = \frac{\gamma}{4}\left(3 - \frac{1}{\mu^*}\right)$.
Branching on $x_n$, $\breve{\vec{x}}_{Q_{n-1, n}^{+,-}} = \breve{\vec{x}}_{Q_{n-1}^+} - \frac{1}{2}\vec{e}_{n} + \frac{1}{2}\vec{e}_{n-2},$ so $\breve{c}_{Q_{n-1, n}^{+,-}} = \breve{c}_{Q_{n-1}^+} - \gamma.$
Meanwhile, $Q_{n-1, n}^{+,+}$ is infeasible, so $\breve{c}_{Q_{n-1}^+} - \breve{c}_{Q_{n-1, n-2}^{+,+}}$ equals some large number $B \geq ||\vec{c}||_1$.
 Therefore, $\mu\score_1\left(Q_{n-1}^+,i\right) + (1-\mu)\score_2\left(Q_{n-1}^+,i\right) = \frac{3\gamma}{4} -  \frac{\mu\gamma}{4\mu^*}$ and $\mu\score_1\left(Q_{n-1}^+,n\right) + (1-\mu)\score_2\left(Q_{n-1}^+,n\right) = \mu\gamma + (1 - \mu)B > \frac{3\gamma}{4} -  \frac{\mu\gamma}{4\mu^*}$. Therefore, $x_n$ will be branched on next.

The next node branch-and-bound will explore is $Q_{n-1, n}^{+,-}$. The vector $\breve{\vec{x}}_{Q_{n-1, n}^{+,-}}$ has fractional values only in positions $i$ and $n-2$. Since both MILP instances $Q_{n-1, n, n-2}^{+, -, -}$ and $Q_{n-1, n, n-2}^{+,-,-}$ are infeasible, so $\mu\score_1\left(Q_{n-1, n}^{-, +},n-2\right) + (1-\mu)\score_2\left(Q_{n-1, n}^{-, +},n-2\right) = B$ whereas $\mu\score_1\left(Q_{n-1, n-2}^{-, +},i\right) + (1-\mu)\score_2\left(Q_{n-1, n-2}^{-, +},i\right) = \frac{3\gamma}{4} -  \frac{\mu\gamma}{4\mu^*} < B$, as before. Therefore, branch-and-bound will branch on $x_{n-2}$ and fathom both children.

At this point, all children have been fathomed, so branch-and-bound will terminate.
\end{proof}

\begin{lemma}\label{lem:mu_small_2}
Suppose $\mu < \mu^*$. Then for any MILP $Q_{\gamma, n} \in \mathcal{G}_{n, \mu^*}$, branch-and-bound with the scoring rule $\mu \score_1 + (1-\mu)\score_2$ will create a tree of depth at least $2^{(n-5)/4}$.
\end{lemma}

\begin{proof}
Let $i =\left\lceil (n-3)/2 \right\rceil$. We first prove two useful claims.

\begin{claim}\label{claim:even}
Let $j$ be an even number such that $2 \leq j \leq i - 2$ and let $J = \left\{x_{i-j/2}, \dots, x_{i+j/2-1}\right\}$. Suppose that B\&B has branched on exactly the variables in $J$ and suppose that the number of variables set to 1 equals the number of variables set to 0. Then B\&B will next branch on the variable $x_{i+j/2}$. Similarly, suppose $J = \left\{x_{i-j/2 + 1}, x_{i-j/2+2}, \dots, x_{i+j/2}\right\}$. Suppose that B\&B has branched on exactly the variables in $J$ and suppose that the number of variables set to 1 equals the number of variables set to 0. Then B\&B will next branch on the variable $x_{i-j/2}$.
\end{claim}

\begin{proof}
Let $Q$ be the MILP contained in the node at the end of the path. This proof has two cases.
\paragraph{Case 1: $J = \left\{x_{i-j/2}, x_{i-j/2+1}, \dots, x_{i+j/2-1}\right\}$.} In this case, there is a set $J_< = \left\{x_1, \dots, x_{i - j/2 - 1}\right\}$ of $i-\frac{j}{2} - 1$ variables smaller than $x_i$ that have not yet been branched on and there is a set $J_> = \left\{x_{i + j/2}, \dots, x_{n-3}\right\}$ of $n - 3 - \left(i+ \frac{j}{2} -1\right)  = 2i - 1  - \left(i+ \frac{j}{2} -1\right) = i - \frac{j}{2}$ variables in $\{x_{i+1}, \dots, x_{n-3}\}$ that have not yet been branched on. Since the number of variables set to 1 in $J$ equals the number of variables set to 0, the LP relaxation will set the $i - \frac{j}{2} - 1$ variables in $J_<$ to 0, the $i - \frac{j}{2} - 1$ variables in $J_> \setminus \left\{x_{i+j/2}\right\}$ to 1, and $x_{i+j/2}$ to $\frac{1}{2}.$ It will also set $x_{n-2} = 0$, $x_{n-1} = \frac{1}{2}$, and $x_n = 1$. Therefore, the two fractional variables are $x_{i+j/2}$ and $x_{n-1}$. Branching on $x_{i+j/2}$, we have
$\breve{\vec{x}}_{Q_{i + j/2}^-} = \breve{\vec{x}}_Q - \frac{1}{2}\vec{e}_{i + j/2} + \frac{1}{2}\vec{e}_{i - j/2 - 1}$ and
$\breve{\vec{x}}_{Q_{i + j/2}^+} = \breve{\vec{x}}_Q + \frac{1}{2}\vec{e}_{i + j/2} - \frac{1}{2}\vec{e}_{i + j/2 + 1}$. Branching on $x_{n-1}$, we have that
$\breve{\vec{x}}_{Q_{n-1}^-} = \breve{\vec{x}}_Q - \frac{1}{2}\vec{e}_{n-1} + \frac{1}{2}\vec{e}_{n-2}$ and
$\breve{\vec{x}}_{Q_{n-1}^+} = \breve{\vec{x}}_Q + \frac{1}{2}\vec{e}_{n-1} - \frac{1}{2}\vec{e}_{n}$. Therefore, \begin{align*}
\breve{c}_{Q_{i + j/2}^-} &= \breve{c}_Q - \frac{\gamma}{2}\left(3 - \frac{1}{2\mu^*}\right)\\
\breve{c}_{Q_{i + j/2}^+} &= \breve{c}_Q\\
\breve{c}_{Q_{n-1}^-} &= \breve{c}_Q - 1 + \frac{1}{2} = \breve{c}_Q - \frac{\gamma}{2}\\
\breve{c}_{Q_{n-1}^+} &= \breve{c}_Q + 1 - \frac{3}{2} = \breve{c}_Q - \frac{\gamma}{2}
\end{align*}

This means that $\breve{c}_Q - \breve{c}_{Q_{i + j/2}^-}
= \frac{\gamma}{2}\left(3 - \frac{1}{2\mu^*}\right)$, $\breve{c}_Q - \breve{c}_{Q_{i + j/2}^+} = 0$, and $\breve{c}_Q - \breve{c}_{Q_{n-1}^-} = \breve{c}_Q - \breve{c}_{Q_{n-1}^+} = \frac{\gamma}{2}$. Therefore, $\mu\score_1(Q, i + j/2) + (1-\mu)\score_2(Q, i + j/2) = \frac{\gamma\left(1 - \mu\right)}{2}\left(3 - \frac{1}{2\mu^*}\right)$ and $\mu\score_1(Q, n-1) + (1-\mu)\score_2(Q, n-1) = \frac{\gamma}{2}$. Since $\mu < \mu^*$ and $\mu^* \in \left(\frac{1}{3}, \frac{1}{2}\right)$, we have that \begin{align*}
\mu\score_1(Q, i + j/2) + (1-\mu)\score_2(Q, i + j/2) &= \frac{\gamma(1 - \mu)}{2}\left(3 - \frac{1}{2\mu^*}\right)\\
& \geq \frac{\gamma\left(1 - \mu^*\right)}{2}\left(3 - \frac{1}{2\mu^*}\right)\\
& \geq \frac{\gamma}{2}.
\end{align*}
Therefore, $x_{i + j/2}$ will be branched on next.

\paragraph{Case 2: $J = \left\{x_{i-j/2 + 1}, x_{i-j/2+2}, \dots, x_{i+j/2}\right\}$.} 
In this case, there is a set $J_< = \left\{x_1, \dots, x_{i- j/2}\right\}$ of $i-\frac{j}{2}$ variables smaller than $x_i$ that have not yet 
been branched on and a set $J_> = \left\{x_{i + j/2 + 1}, \dots, x_{n-3}\right\}$ of $n - 3 - \left(i+ \frac{j}{2}\right)  = 2i - 1  - \left(i+ \frac{j}{2}\right) = i - \frac{j}{2} - 1$ variables in $\{x_{i+1}, \dots, x_{n-3}\}$ that have not yet been branched on. Since the number 
of variables set to 1 in $J$ equals the number of variables set to 0, the LP 
relaxation will set the $i - \frac{j}{2} - 1$ variables in $J_<\setminus\left\{x_{i-j/2}\right\}$ to 0, the $i - \frac{j}{2} - 1$ variables in $J_> \setminus \left\{x_{i+j/2}\right\}$ to 1, and $x_{i-j/2}$ to $\frac{1}{2}.$ It will also set $x_{n-2} = 0$, 
$x_{n-1} = \frac{1}{2}$, and $x_n = 1$. Therefore, the two fractional variables are 
$x_{i-j/2}$ and $x_{n-1}$. Branching on $x_{i-j/2}$, we have
$\breve{\vec{x}}_{Q_{i - j/2}^-} = \breve{\vec{x}}_Q - \frac{1}{2}\vec{e}_{i - j/2} + \frac{1}{2}\vec{e}_{i - j/2 - 1}$ and
$\breve{\vec{x}}_{Q_{i - j/2}^+} = \breve{\vec{x}}_Q + \frac{1}{2}\vec{e}_{i - j/2} - \frac{1}{2}\vec{e}_{i + j/2 + 1}$. Branching on $x_{n-1}$, we have that
$\breve{\vec{x}}_{Q_{n-1}^-} = \breve{\vec{x}}_Q - \frac{1}{2}\vec{e}_{n-1} + \frac{1}{2}\vec{e}_{n-2}$ and
$\breve{\vec{x}}_{Q_{n-1}^+} = \breve{\vec{x}}_Q + \frac{1}{2}\vec{e}_{n-1} - \frac{1}{2}\vec{e}_{n}$. Therefore, \begin{align*}
\breve{c}_{Q_{i - j/2}^-} &= \breve{c}_Q\\
\breve{c}_{Q_{i - j/2}^+} &=  \breve{c}_Q - \frac{\gamma}{2}\left(3 - \frac{1}{2\mu^*}\right)\\
\breve{c}_{Q_{n-1}^-} &= \breve{c}_Q - \gamma + \frac{\gamma}{2} = \breve{c}_Q - \frac{\gamma}{2}\\
\breve{c}_{Q_{n-1}^+} &= \breve{c}_Q + \gamma - \frac{3\gamma}{2} = \breve{c}_Q - \frac{\gamma}{2}
\end{align*}

This means that $\breve{c}_Q - \breve{c}_{Q_{i-j/2}^-}
= \frac{\gamma}{2}\left(3 - \frac{1}{2\mu^*}\right)$, $\breve{c}_Q - \breve{c}_{Q_{i-j/2}^+}= 0$, and $\breve{c}_Q - \breve{c}_{Q_{n-1}^-} = \breve{c}_Q - \breve{c}_{Q_{n-1}^+} = \frac{\gamma}{2}$ as in the previous case, so $x_{i - j/2}$ will be branched on next. 
\end{proof}

\begin{claim}\label{claim:odd}
Suppose that $J$ is the set of variables branched on along a path of depth $2 \leq j \leq i - 2$  where $j$ is odd and let $j' = \lfloor j/2\rfloor$. Suppose that $J = \left\{x_{i-j'}, x_{i-j'+1}, \dots, x_{i+j'}\right\}$. Moreover, suppose that the number of variables set to 1 in $J$ equals the number of variables set to 0, plus or minus 1. Then B\&B will either branch on $x_{i-j' - 1}$ or $x_{i+j'+1}$.
\end{claim}

\begin{proof}
Let $Q$ be the MILP contained a the end of the path. There are $i - j' - 1$ variables $J_< = \left\{x_1, \dots, x_{i-j'-1}\right\}$ that are smaller than $x_i$ that have not yet been branched on, and $n-3 - (i+j') = 2i - 1 - (i + j') = i-j'-1$ variables $J_> = \left\{x_{i+j' + 1}, \dots, x_{n-3}\right\}$ in $\left\{x_{i+1}, \dots, x_{n-3}\right\}$ that have not yet been branched on. Let $z$ be the number of variables in $J$ set to 0 and let $o$ be the number of variables set to 1. This proof has two cases:
\paragraph{Case 1: $z = o + 1$.} Since $z + o = j$, we know that $z = j' + 1$ and $o = j'$. Therefore, the LP relaxation will set the variables in $J_< \setminus \left\{x_{i-j' - 1}\right\}$ to zero for a total of $\left|J_< \setminus \left\{x_{i-j' - 1}\right\}\right| + z = i - j' - 2 + j' + 1 = i - 1$ zeros, it will set the variables in $J_>$ to one for a total of $\left|J_>\right| + o = i - j' - 1 + j' = i-1$ ones, and it will set $x_{i-j' - 1}$ to $\frac{1}{2}$. It will also set $x_{n-2} = 0$, $x_{n-1} = \frac{1}{2}$, and $x_n = 1$. Therefore, the two fractional variables are $x_{i-j' - 1}$ and $x_{n-1}$. Branching on $x_{i-j' - 1}$, we have
$\breve{\vec{x}}_{Q_{i -j' - 1}^-} = \breve{\vec{x}}_Q - \frac{1}{2}\vec{e}_{i -j' - 1} + \frac{1}{2}\vec{e}_{i -j' - 2}$ and
$\breve{\vec{x}}_{Q_{i -j' - 1}^+} = \breve{\vec{x}}_Q + \frac{1}{2}\vec{e}_{i -j' - 1} - \frac{1}{2}\vec{e}_{i + j' + 1}$. Branching on $x_{n-1}$, we have that
$\breve{\vec{x}}_{Q_{n-1}^-} = \breve{\vec{x}}_Q - \frac{1}{2}\vec{e}_{n-1} + \frac{1}{2}\vec{e}_{n-2}$ and
$\breve{\vec{x}}_{Q_{n-1}^+} = \breve{\vec{x}}_Q + \frac{1}{2}\vec{e}_{n-1} - \frac{1}{2}\vec{e}_{n}$. Therefore, \begin{align*}
\breve{c}_{Q_{i -j' - 1}^-} &= \breve{c}_Q\\
\breve{c}_{Q_{i -j' - 1}^+} &= \breve{\vec{x}}_Q - \frac{\gamma}{2}\left(3 - \frac{1}{2\mu^*}\right)\\
\breve{c}_{Q_{n-1}^-} &= \breve{c}_Q - \gamma + \frac{\gamma}{2} = \breve{c}_Q - \frac{\gamma}{2}\\
\breve{c}_{Q_{n-1}^+} &= \breve{c}_Q + \gamma - \frac{3\gamma}{2} = \breve{c}_Q - \frac{\gamma}{2}
\end{align*}

This means that $\breve{c}_Q - \breve{c}_{Q_{i - j' - 1}^-} =0$, $\breve{c}_Q - \breve{c}_{Q_{i - j' - 1}^+} = \frac{\gamma}{2}\left(3 - \frac{1}{2\mu^*}\right)$, and $\breve{c}_Q - \breve{c}_{Q_{n - 1}^-}= \breve{c}_Q - \breve{c}_{Q_{n - 1}^+} = \frac{\gamma}{2}$. Therefore, $\mu\score_1(Q, i -j' - 1) + (1-\mu)\score_2(Q, i -j' - 1) = \frac{\gamma(1 - \mu)}{2}\left(3 - \frac{1}{2\mu^*}\right)$ and $\mu\score_1(Q, n-1) + (1-\mu)\score_2(Q, n-1) = \frac{\gamma}{2}$. Since $\mu < \mu^*$ and $\mu^* \in \left(\frac{1}{3}, \frac{1}{2}\right)$, we have that \begin{align*}
\mu\score_1(Q, i -j' - 1) + (1-\mu)\score_2(Q, i -j' - 1) &= \frac{\gamma(1 - \mu)}{2}\left(3 - \frac{1}{2\mu^*}\right)\\
& \geq \frac{\gamma\left(1 - \mu^*\right)}{2}\left(3 - \frac{1}{2\mu^*}\right)\\
& \geq \frac{\gamma}{2}.
\end{align*}
Therefore, $x_{i -j' - 1}$ will be branched on next.

\paragraph{Case 1: $z = o - 1$.} Since $z + o = j$, we know that $z = j'$ and $o = j' + 1$. Therefore, the LP relaxation will set the variables in $J_<$ to zero for a total of $\left|J_<\right| + z = i - j' - 1 + j'= i - 1$ zeros, it will set the variables in $J_> \setminus \left\{x_{i + j' + 1}\right\}$ to one for a total of $\left|J_>\setminus \left\{x_{i + j' + 1}\right\}\right| + o = i - j' - 2 + j' +1 = i-1$ ones, and it will set $x_{i + j' + 1}$ to $\frac{1}{2}$. It will also set $x_{n-2} = 0$, $x_{n-1} = \frac{1}{2}$, and $x_n = 1$. Therefore, the two fractional variables are $x_{i+j' + 1}$ and $x_{n-1}$. Branching on $x_{i+j' + 1}$, we have
$\breve{\vec{x}}_{Q_{i + j' + 1}^-} = \breve{\vec{x}}_Q - \frac{1}{2}\vec{e}_{i + j' + 1} + \frac{1}{2}\vec{e}_{i -j' - 1}$ and
$\breve{\vec{x}}_{Q_{i + j' + 1}^+} = \breve{\vec{x}}_Q + \frac{1}{2}\vec{e}_{i + j' + 1} - \frac{1}{2}\vec{e}_{i + j' + 2}$. Branching on $x_{n-1}$, we have that
$\breve{\vec{x}}_{Q_{n-1}^-} = \breve{\vec{x}}_Q - \frac{1}{2}\vec{e}_{n-1} + \frac{1}{2}\vec{e}_{n-2}$ and
$\breve{\vec{x}}_{Q_{n-1}^+} = \breve{\vec{x}}_Q + \frac{1}{2}\vec{e}_{n-1} - \frac{1}{2}\vec{e}_{n}$. Therefore, \begin{align*}
\breve{c}_{Q_{i + j' + 1}^-} &= \breve{c}_Q - \frac{\gamma}{2}\left(3 - \frac{1}{2\mu^*}\right)\\
\breve{c}_{Q_{i + j' + 1}^+} &= \breve{\vec{x}}_Q\\
\breve{c}_{Q_{n-1}^-} &= \breve{c}_Q - \gamma + \frac{\gamma}{2} = \breve{c}_Q - \frac{\gamma}{2}\\
\breve{c}_{Q_{n-1}^+} &= \breve{c}_Q + \gamma - \frac{3\gamma}{2} = \breve{c}_Q - \frac{\gamma}{2}
\end{align*}

This means that $\breve{c}_Q - \breve{c}_{Q_{i + j' + 1}^-}
 =\frac{\gamma}{2}\left(3 - \frac{1}{2\mu^*}\right)$, $\breve{c}_Q - \breve{c}_{Q_{i + j' + 1}^+} = 0$, and $\breve{c}_Q - \breve{c}_{Q_{n - 1}^-} = \breve{c}_Q - \breve{c}_{Q_{n - 1}^+} = \frac{\gamma}{2}$. Therefore, $\mu\score_1(Q, i + j' + 1) + (1-\mu)\score_2(Q, i +j' + 1) = \frac{\gamma(1 - \mu)}{2}\left(3 - \frac{1}{2\mu^*}\right)$ and $\mu\score_1(Q, n-1) + (1-\mu)\score_2(Q, n-1) = \frac{\gamma}{2}$. As in the previous case, this means that $x_{i +j' + 1}$ will be branched on next.
\end{proof}

We now prove by induction that there are $2^{\left(\lceil(n-3)/2\rceil) - 1\right)/2} \geq 2^{(n-5)/4}$ paths in the B\&B tree of length at least $i-2$. Therefore, the size of the tree is at least $2^{(n-5)/4}$.

\paragraph{Inductive hypothesis.} Let $j$ be an arbitrary integer between 1 and $i-2$. If $j$ is even, then there exist at least $2^{j/2}$ paths in the B\&B tree from the root to nodes of depth $j$ such that if $J$ is the set indices branched on along a given path, then $J = \left\{x_{i-j/2}, x_{i-j/2+1}, \dots, x_{i+j/2-1}\right\}$ or $J = \left\{x_{i-j/2 + 1}, x_{i-j/2+2}, \dots, x_{i+j/2}\right\}$. Moreover, the number of variables set to 0 in $J$ equals the number of variables set to 1. Meanwhile, if $j$ is odd, let $j' = \lfloor j/2\rfloor$. There exist at least $2^{(j+1)/2}$ paths in the B\&B tree from the root to nodes of depth $j$ such that if $J$ is the set indices branched on along a given path, then $J = \left\{x_{i-j'}, x_{i-j'+1}, \dots, x_{i+j'}\right\}$. Moreover, the number of variables set to 0 in $J$ equals the number of variables set to 1, plus or minus 1.

\paragraph{Base case.} To prove the base case, we need to show that B\&B first branches on $x_i$. We saw that this will be the case in Lemma~\ref{lem:mu_big_2} so long as $\mu < \mu^*$.

\paragraph{Inductive step.} Let $j$ be an arbitrary integer between 1 and $i-3$. There are two cases, one where $j$ is even and one where $j$ is odd. First, suppose $j$ is even. From the inductive hypothesis, we know that there exist at least $2^{j/2}$ paths in the B\&B tree from the root to nodes of depth $j$ such that if $J$ is the set varibles branched on along a given path, then $J = \left\{x_{i-j/2}, x_{i-j/2+1}, \dots, x_{i+j/2-1}\right\}$ or $J = \left\{x_{i-j/2 + 1}, x_{i-j/2+2}, \dots, x_{i+j/2}\right\}$. Moreover, the number of variables set to 0 in $J$ equals the number of variables set to 1. From Claim~\ref{claim:even}, we know that in the first case, $x_{i+j/2}$ will be the next node B\&B will branch on. This will create two new paths: $\left\{x_{i-j/2}, x_{i-j/2+1}, \dots, x_{i+j/2}\right\}$ will be the set of variables branched along each path, and the number of variables set to 0 will equal the number of variables set to 1, plus or minus 1. Also from Claim~\ref{claim:even}, we know that in the second case, $x_{i-j/2}$ will be the next node B\&B will branch on. This will also create two new paths: $\left\{x_{i-j/2}, x_{i-j/2+1}, \dots, x_{i+j/2}\right\}$ will be the set of variables branched along each path, and the number of variables set to 0 will equal the number of variables set to 1, plus or minus 1. Since this is true for all $2^{j/2}$ paths, this leads to a total of $2^{j/2 + 1} = 2^{(j+2)/2}$ paths, meaning the inductive hypothesis holds.

Next, suppose $j$ is odd and let $j' = \lfloor j/2\rfloor$. From the inductive hypothesis, we know that there exist at least $2^{(j+1)/2}$ paths in the B\&B tree from the root to nodes of depth $j$ such that if $J$ is the set variables branched on along a given path, then $J = \left\{x_{i-j'}, x_{i-j'+1}, \dots, x_{i+j'}\right\}$. Moreover, the number of variables set to 0 in $J$ equals the number of variables set to 1, plus or minus 1. From Claim~\ref{claim:even}, we know that B\&B will either branch on $x_{i-j' - 1}$ or $x_{i+j'+1}$. Suppose the number of variables set to 0 in $J$ is 1 greater than the number of variables set to 1. If B\&B branches on $x_{i-j' - 1}$, we can follow the path where $x_{i-j' - 1} = 1$, and this will give us a new path where \[\left\{x_{i-j' - 1}, x_{i-j'}, \dots, x_{i+j'}\right\} = \left\{x_{i-(j + 1)/2}, x_{i-(j + 1)/2+1}, \dots, x_{i+(j + 1)/2 - 1}\right\}\] are the variables branched on and the number of variables set to 0 equals the number of variables set to 1. If B\&B branches on $x_{i+j'+1}$, we can follow the path where $x_{i+j'+1} = 1$, and this will give us a new path where \[\left\{x_{i-j'}, x_{i-j' + 1}, \dots, x_{i+j'}, x_{i+j'+1}\right\} = \left\{x_{i-(j + 1)/2 +1}, x_{i-(j + 1)/2+1}, \dots, x_{i+(j + 1)/2}\right\}\] are the variables branched on and the number of variables set to 0 equals the number of variables set to 1. A symmetric argument holds if the number of variables set to 0 in $J$ is 1 less than the number of variables set to 1. Therefore, all $2^{(j+1)/2}$ paths can be extended by one edge, so the statement holds.
\end{proof}
\end{proof}

\begin{theorem}\label{thm:NSP}
Let $Q$ be an infeasible MILP, let NSP and NSP' be two node selection policies, and let $\score$ be a path-wise scoring rule. The tree $\tree$ B\&B builds using NSP and $\score$ equals the tree $\tree'$ it builds using NSP' and $\score$.
\end{theorem}

\begin{proof}
For a contradiction, suppose $\tree \not= \tree'$. There must be a node $Q_0$ in $\tree$ where if $\tree_{Q_0}$ is the path from the root of $\tree$ to ${Q_0}$, then $\tree_{Q_0}$ is a rooted subtree of $\tree'$, but either:
\begin{enumerate}
\item In $\tree$, the node ${Q_0}$ is fathomed but in $\tree'$, ${Q_0}$ is not fathomed, or
\item In $\tree$, the node ${Q_0}$ is not fathomed, but for all children ${Q_0}'$ of ${Q_0}$ in $\tree$, if $\tree_{{Q_0}'}$ is the path from the root of $\tree$ to ${Q_0}'$, $\tree_{{Q_0}'}$ is not a rooted subtree of $\tree'$.
\end{enumerate}
We will show that neither case is possible, thus arriving at a contradiction. First, we know that since $Q$ is infeasible, B\&B will only fathom a node if it is infeasible. Therefore, the first case is impossible: if $Q_0$ is fathomed in $\tree$, it must be infeasible, so it will also be fathomed in $\tree'$, and vice versa. Therefore, we know that B\&B must branch on $Q_0$ in both $\tree$ and $\tree'$. Let $\bar{\tree}$ be the state of the tree B\&B has built using NSP and $\score$ by the time it branches on $Q_0$ and let $\bar{\tree}'$ be the state of the tree B\&B has built using NSP' and $\score$ by the time it branches on $Q_0$. Since $\tree_{Q_0}$ is a rooted subtree of both $\bar{\tree}$ and $\bar{\tree'}$, we know that for all variables $x_i$, $\score(\bar{\tree}, Q_0, i) = \score(\tree_Q, Q_0, i) = \score(\bar{\tree}', Q_0, i)$. Therefore, B\&B will branch on the same variable in both $\tree$ and $\tree'$, which is a contradiction, since this means that for all children ${Q_0}'$ of ${Q_0}$ in $\tree$, if $\tree_{{Q_0}'}$ is the path from the root of $\tree$ to ${Q_0}'$, $\tree_{{Q_0}'}$ is a rooted subtree of $\tree'$.
\end{proof}

\begin{lemma}[\citet{Shalev14:Understanding}]\label{lem:log_ineq}
Let $a \geq 1$ and $b > 0$. Then $x < a\log x + b$ implies that $x < 4a \log (2a) + 2b$.
\end{lemma}

\claimIntervals*
\begin{proof}We prove this claim by induction.

\bigskip
\noindent\textbf{Inductive hypothesis.} For $i \in \{1,\dots, n\}$, there are $T \leq 2^{i(i-1)/2}n^i$ intervals $I_1, \dots, I_T$ partitioning $[0,1]$ where for any interval $I_j$ and any two parameters $\mu, \mu' \in I_j$, if $\tree$ and $\tree'$ are the trees $A'$ builds using the scoring rules $\mu\score_1 + (1-\mu)\score_2$ and $\mu'\score_1 + (1-\mu')\score_2$, respectively, then $\tree[i] = \tree'[i]$.

\bigskip

\noindent\textbf{Base case.} Before branching on any variables, the branch-and-bound tree $\tree_0$ consists of a single root node $Q$. Given a parameter $\mu$, $A'$ will branch on variable $x_k$ so long as \[k = \argmax_{\ell \in [n]} \left\{\mu\score_1(\tree_0, Q, \ell) + (1-\mu)\score_2(\tree_0, Q, \ell)\right\}.\]
Since $\mu\score_1(\tree_0, Q, \ell) + (1-\mu)\score_2(\tree_0, Q, \ell)$ is a linear function of $\mu$ for each $\ell \in [n]$, we know that for any $k \in [n]$, there is at most one interval $I$ of the parameter space $[0,1]$
where $k = \argmax_{\ell \in [n]} \left\{\mu\score_1 + (1-\mu)\score_2\right\}$. Thus, there are $T \leq n = 2^{1\cdot(1-1)/2}n^1$ intervals $I_1, \dots, I_T$ partitioning $[0,1]$ where for any interval $I_j$, $A'$ branches on the same variable at the root node using the scoring rule $\mu\score_1 + (1-\mu)\score_2$ across all $\mu \in I_j$.

\bigskip

\noindent\textbf{Inductive step.} Let $i \in \{2, \dots, n\}$ be arbitrary. From the inductive hypothesis, we know that there are $T \leq 2^{(i-2)(i-1)/2}n^{i-1}$ intervals $I_1, \dots, I_T$ partitioning $[0,1]$ where for any interval $I_j$ and any two parameters $\mu, \mu' \in I_j$, if $\tree$ and $\tree'$ are the trees $A'$ builds using the scoring rules $\mu\score_1 + (1-\mu)\score_2$ and $\mu'\score_1 + (1-\mu')\score_2$, respectively, then $\tree[i-1] = \tree'[i-1]$. Consider an arbitrary node $\node$ in $\tree[i-1]$ (or equivalently, $\tree'[i-1]$) at depth $i-1$. If $Q$ is integral or infeasible, then it will be fathomed no matter which parameter $\mu\in I_j$ the algorithm $A'$ uses. Otherwise, for all $\mu \in I_j$, let $\tree_{\mu}$ be the state of the search tree $A'$ builds using the scoring rule $\mu\score_1 + (1-\mu)\score_2$ at the point when it branches on $Q$. By the inductive hypothesis, we know that across all $\mu \in I_j$, the path from the root to $Q$ in $\tree_{\mu}$ is invariant, and we refer to this path as $\tree_Q$. Given a parameter $\mu \in I_j$, the variable $x_k$ will be branched on at node $\node$ so long as $k= \argmax_{\ell}\left\{\mu\score_1(\tree_{\mu},\node,\ell) + (1-\mu)\score_2(\tree_{\mu},\node,\ell)\right\},$ or equivalently, so long as $k= \argmax_{\ell}\left\{\mu\score_1(\tree_Q,\node,\ell) + (1-\mu)\score_2(\tree_Q,\node,\ell)\right\}$. In other words, the decision of which variable to branch on is determined by a convex combination of the constant values $\score_1(\tree_Q,\node,\ell)$ and $\score_2(\tree_Q,\node,\ell)$ no matter which parameter $\mu \in I_j$ the algorithm $A'$ uses. Here, we critically use the fact that the scoring rule is path-wise.

Since $\mu\score_1(\tree_Q,\node,\ell) + (1-\mu)\score_2(\tree_Q,\node,\ell)$ is a linear function of $\mu$ for all $\ell$, there are at most $n$ intervals subdividing the interval $I_j$ such that the variable branched on at node $\node$ is fixed. Moreover, there are at most $2^{i-1}$ nodes at depth $i-1$, and each node similarly contributes a subpartition of $I_j$ of size $n$. If we merge all $2^{i-1}$ partitions, we have $T' \leq 2^{i-1}(n-1)+1$ intervals $I_1', \dots, I_{T'}'$ partitioning $I_j$ where for any interval $I_p'$ and any two parameters $\mu, \mu' \in I_p'$, if $\tree$ and $\tree'$ are the trees $A'$ builds using the scoring rules $\mu\score_1 + (1-\mu)\score_2$ and $\mu'\score_1 + (1-\mu')\score_2$, respectively, then $\tree[i] = \tree'[i]$. We can similarly subdivide each interval $I_1, \dots, I_T$ for a total of \[\bar{T} \leq 2^{(i-1)(i-2)/2}n^{i-1}\left(2^{i-1}(n-1)+1\right) \leq 2^{(i-1)(i-2)/2}n^{i-1}\left(2^{i-1}n\right) = 2^{i(i-1)/2}n^{i}\] intervals $\bar{I}_1, \dots, \bar{I}_{\bar{T}}$ partitioning $[0,1]$ such that for any interval $\bar{I}_t$, across all $\mu \in \bar{I}_t$ and any two parameters $\mu, \mu' \in \bar{I}_t$, if $\tree$ and $\tree'$ are the trees $A'$ builds using the scoring rules $\mu\score_1 + (1-\mu)\score_2$ and $\mu'\score_1 + (1-\mu')\score_2$, respectively, then $\tree[i] = \tree'[i]$.
\end{proof}

\subsection{Empirical Risk Minimization Algorithm} \label{app:ermalg}

In this section we describe an empirical risk minimization algorithm capable of
finding the best mixture of two variable selection scoring rules for a given set
of problem instances $\Pi^{(1)}, \dots, \Pi^{(m)}$. We modify the tree search
algorithm so that given a problem instance $\Pi$ and a mixing paramter $\mu \in
[0,1]$, the search algorithm keeps track of the largest interval $I \subset
[0,1]$ such that the behavior of the algorithm is identical to the current run
when run with any parameter $\mu' \in I$. With this, we can enumerate all
possible behaviors of the algorithm for a single instance $\Pi$ by running the
algorithm with $\mu = 0$, followed by the smallest value of $\mu$ that will give
a different outcome, and so on, until we have covered the entire interval
$[0,1]$. This procedure results in running the tree search algorithm on the
instance $\Pi$ exactly once for each possible behavior achievable across all
values of the parameter $\mu \in [0,1]$. By applying this algorithm to each
problem $\Pi^{(i)}$ for $i = 1, \dots, m$, we discover how the tree search
algorithm would perform on every instance for any value of the mixing parameter
$\mu$. This allows us to divide the interval $[0,1]$ into a finite number of
intervals on which $\cost$ is piecewise constant, and to compute the cost on
each interval.

To see why this additional book keeping is possible, suppose we are choosing
which variable to branch on in node $\node$ of tree $\tree$. We have two scoring
rules $\score_1$ and $\score_2$ that each rank the candidate variables in
$\node$, and when we run the algorithm with parameter $\mu$, we combine these
two scores as $(1-\mu)\score_1(\tree, \node, i) + \mu \score_2(\tree, \node,
i)$. For any parameter $\mu$ which results in the same variable having the
highest score, the variable chosen for branching in this node will be identical.
Let $i^*$ be the variable chosen by the algorithm when run with parameter $\mu$.
The set of all $\mu'$ for which $i^*$ is the variable of the highest score is an
interval (and its end points can be found by solving a linear equation to
determine the value of $\mu'$ for which some other variable overtakes $i^*$
under the mixed score). Also, for every parameter $\mu'$ outside of
this interval, the algorithm would indeed branch on a different node, resulting
in a different outcome of the tree search algorithm. By taking the intersections
of these intervals across all branching variable choices, we find the largest
subset of $[0,1]$ for which the algorithm would behave exactly the same, and
this subset is an interval. The overhead of this book keeping is only linear in
the number of candidate branch variables.

Psuedo-code for the ERM algorithm is given in Algorithm~\ref{alg:ermalg}.

\begin{algorithm}[tb]
\caption{ERM Algorithm}\label{alg:ermalg}
\begin{algorithmic}[1]
\Require Problem instances $\Pi^{(1)}, \dots, \Pi^{(m)}$, variable scoring rules $\score_1, \score_2$.
\State For each problem instance $\Pi^{(i)}$, compute the piecewise constant
$\cost$, as a function of the mixing parameter $\mu$. (See discussion in Section~\ref{app:ermalg}).
\State Compute the point-wise average of the resulting piecewise constant
functions.
\Ensure The $\mu^*$ in the interval minimizing the average cost.
\end{algorithmic}
\end{algorithm}

\section{Additional information about experiments}\label{app:experiments}

\paragraph{Facility location.}
Suppose there is a set $I$ of customers
and a set $J$ of facilities that have not yet been built. The facilities each
produce the same good, and each consumer demands one unit of that good. Consumer
$i$ can obtain some fraction $y_{ij}$ of the good from facility $j$, which costs
them $d_{ij}y_{ij}$. Moreover, it costs $f_j$ to construct facility $j$. The
goal is to choose a subset of facilities to construct while minimizing total
cost. We can model this
problem as a MILP by assigning a binary variable $x_j$ to each facility, which
represents whether or not it is built. The optimization is then over the
variables $x_j$ and $y_{ij}$, as follows.
\[
\begin{array}{lll}
\textnormal{minimize} & \sum_{j \in J} f_jx_j + \sum_{j \in J, i \in I} d_{ij}y_{ij} &\\
\textnormal{s.t.} & \sum_{j \in J} y_{ij} = 1 &\forall i \in I\\
&y_{ij} \leq x_j &\forall i \in I, j \in J\\
& x_j \in \{0,1\} &\forall j \in J\\
& y_{ij} \in [0,1] &\forall i \in I, j \in J.
\end{array}
\]

\paragraph{Clustering.} We can formulate $k$-means clustering as a MILP by
assigning a binary variable $x_i$ to each point $p_i$ where $x_i = 1$ if and
only if $p_i$ is a center, as well as a binary variable $y_{ij}$ for each pair
of points $p_i$ and $p_j$, where $y_{ij} = 1$ if and only if $p_j$ is a center
and $p_j$ is the closest center to $p_i$. We want to solve the following
problem:
\[
\begin{array}{lll}
\textnormal{min} & \sum_{i,j \in [n]} d\left(p_i,p_j\right)y_{ij} &\\
\textnormal{s.t.} & \sum_{i = 1}^n x_i = k &\\
& \sum_{j = 1}^n y_{ij} = 1 &\forall i \in [n]\\
& y_{ij} \leq x_j & \forall i,j \in [n]\\
& x_i \in \{0,1\} &\forall i \in [n]\\
& y_{ij} \in \{0,1\} &\forall i,j \in [n].
\end{array}
\]

\paragraph{Agnostically learning linear separators.} We can formulate this
problem as a MILP as follows. Let $M > \max \norm{\vec{p}_i}_1$.
\[
\begin{array}{lll}
\textnormal{min} & \sum_{i = 1}^n x_i &\\
\textnormal{s.t.} & z_i \left\langle\vec{p}_i, \vec{w}\right\rangle > -Mx_i &\forall i \in [n]\\
& w[i] \in [-1,1] &\forall i \in [n]\\
& x_i \in \{0,1\} &\forall i \in [n].
\end{array}
\]
Since $\left| \left \langle \vec{p}_i, \vec{w}\right\rangle\right| < M$, the
inequality $z_i \left\langle\vec{p}_i, \vec{w}\right\rangle > -Mx_i$ ensures
that if $z_i \left\langle\vec{p}_i, \vec{w}\right\rangle > 0$, then $x_i$ will
equal 0, but if $z_i \left\langle\vec{p}_i, \vec{w}\right\rangle \leq 0$, then
$x_i$ must equal 1.\footnote{In practice, we implement this constraint by
enforcing that $z_i \left\langle\vec{p}_i, \vec{w}\right\rangle \geq -Mx_i +
\gamma$ for some tiny $\gamma > 0$ since MILP solvers cannot enforce strict
inequalities.}

\begin{theorem}\label{thm:MILP_WCrad}
Let $\cost$ be a tree-constant cost function, let $\score_1$ and $\score_2$ be two path-wise scoring rules, and let $\sample$ be a set of $m$ problem instances over $n$ binary variables. Let $\mathcal{C}$ be the set of functions $\left\{\cost\left(\cdot, \mu\score_1 + (1-\mu)\score_2\right) : \mu \in [0,1]\right\}$. Then \[\erad(\mathcal{C}) \leq \kappa \sqrt{\frac{1}{m}\left(n^2 + 2n\log n + 2\log m\right)}.\]
\end{theorem}

\begin{proof} To prove this theorem, we rely on Lemma~\ref{lem:induction}, which tells us that for any problem instance $\Pi \in \sample$, there are $T \leq 2^{n(n-1)/2}n^n$ intervals $I_1, \dots, I_T$ partitioning $[0,1]$ where for any interval $I_j$, across all $\mu \in I_j$, the scoring rule $\mu\score_1 + (1-\mu)\score_2$ results in the same search tree.
Theorem~\ref{thm:MILP_WCrad} then follows from this lemma and Massart's lemma~\citep{Massart00:Some}, which tells us that for any class $\fclass$ of functions with range $[0, c]$, if $N= |\fclass(\sample)|$, then $\erad(\fclass) \leq c \sqrt{\frac{2\log N}{m}}.$ Lemma~\ref{lem:induction} allows us to apply Massart's lemma with $N \leq \left|2^{n^2/2}n^nm\right|$ and $c = \kappa$, since so long as the search tree is fixed for all $\mu$ in an interval $I$, we know that $\cost\left(\Pi, \mu\score_1 + (1-\mu)\score_2\right)$ is constant for all $\mu \in I$.\end{proof}
\section{Tree search}\label{app:TS}
\begin{algorithm}[t]
\caption{Tree search}\label{alg:TS}
\begin{algorithmic}[1]
   \Require A problem instance $\Pi = (X, D, f, g)$.
   \State Let $\tree$ be a tree that consists of a single node containing the empty partial solution $(\top, \dots, \top)$.
   \While {there remains an unfathomed leaf in $\tree$}
   \parState {Use a \emph{node selection policy} to select a leaf of the tree $\tree$. Let $\vec{y}$ be the partial solution contained in that leaf.}\label{step:begin_while_TS}
\State Use a \emph{variable selection policy} to choose a variable $x_i \in X$ to branch on at that leaf.
\State For all $j \in D_i$, let $\vec{y}^{(j)}_i$ be the partial solution $\vec{y}$ except with the component $y[i] = j$.
\parState {Create $|D_i|$ children of the node containing the partial solution $\vec{y}$, where the $j^{th}$ child contains the partial solution $\vec{y}^{(j)}_i$. Let $\tree'$ be the resulting search tree.}
\For {$j \in D_i$}\label{step:fathom_loop}
	\If {$\loc(\Pi, \vec{y}^{(j)}_i) = \fathom$}
		\State Update $\tree'$ so that the leaf containing the partial solution $\vec{y}^{(j)}_i$ is fathomed.
	\ElsIf {$\glo(\Pi, \vec{y}^{(j)}_i, \tree') = \fathom$}\label{step:if}
		\State Update $\tree'$ so that the leaf containing the partial solution $\vec{y}^{(j)}_i$ is fathomed.\label{step:then}
	\EndIf
\EndFor
\State Set $\tree = \tree'$.\label{step:end_while_TS}
  \EndWhile
 \Ensure The best known feasible solution $\vec{y}^*$, if one exists. Otherwise, return $\texttt{Null}$.
\end{algorithmic}
\end{algorithm}

A \emph{tree search algorithm} takes as input a tuple $\Pi = (X, D, f, g)$, where $X = \left\{x_1, \dots, x_n\right\}$ is a set of variables, $D = \{D_1, \dots, D_n\}$ is a set of domains where $D_i$ is the finite set of values variable $x_i$ can take on, $f: D_1\times \cdots \times D_n \to \{0,1\}$ is a feasibility function, and $g:D_1\times \cdots \times D_n \to \R$ is an objective function (if the problem is a satisfiability problem, rather than an optimization problem, we set $g$ to be the constant zero function). We use $\vec{y} \in (D_1 \cup \{\top\}) \times \cdots \times (D_n \cup \{\top\})$ to denote a partial solution to the problem instance $\Pi$, where $y[i]$ is an assignment of the variable $x_i$ and if $y[i] = \top$, it means the variable $x_i$ has not yet been assigned a value.

Tree search builds a tree of partial solutions to $\Pi$ until it finds the optimal solution.
We define two fathoming functions tree search can use to prune branches of the search tree, $\loc$ and $\glo$. The function $\loc(\Pi, \vec{y}) \in \{\fathom, \explore\}$ takes as input an instance $\Pi$ and a partial solution $\vec{y}$ and determines whether or not to fathom the node containing the partial solution $\vec{y}$. Its output is only based on the local information contained in the partial solution $\vec{y}$, not the remainder of the search tree.
For example, in MIP, $\loc(\Pi, \vec{y}) = \fathom$ if the LP relaxation of the MIP given the partial solution $\vec{y}$ is integral or infeasible.
 The function $\glo(\Pi, \vec{y}, \tree) \in \{\fathom, \explore\}$ takes as input an instance $\Pi$, a partial solution $\vec{y}$, and a partial search tree $\tree$ and determines whether or not to fathom the node in $\tree$ containing the partial solution $\vec{y}$.
 For example, in MIP, the function $\glo$ covers the case where a node is fathomed because LP relaxation's objective value given the partial solution contained in that node is no better than the objective value evaluated on the best known integral solution. In a bit more detail, suppose there is a fathomed leaf node in $\tree$ containing a partial solution $\vec{y}^*$ such that the LP relaxation of the MIP given the partial solution $\vec{y}^*$ is integral, and let $c^*$ be the objective value. Let $c$ be the objective value of the LP relaxation of the MIP given the partial solution $\vec{y}$. We know that $\glo(\Pi, \vec{y}, \tree) = \fathom$ if $c \leq c^*$.

See Algorithm~\ref{alg:TS} for the tree search pseudo-code.

\subsection{Problem statement}\label{sec:TS_statement}
The problem statement for general tree search is nearly identical to that in Section~\ref{sec:statement}. We state it hear for clarity's sake.

Let $\dist$ be a distribution over problem instances $\Pi$. Let $\score_1, \dots, \score_d$ be a set of variable selection scoring rules, such as those in Section~\ref{sec:CSP_VSP}. Our goal is to learn a convex combination $\mu_1\score_1+ \cdots + \mu_d\score_d$ of the scoring rules that is nearly optimal in expectation over $\dist$. More formally, let $\cost$ be an abstract cost function that takes as input a problem instance $\Pi$ and a scoring rule $\score$ and returns some measure of the quality of tree search using $\score$ on input $\Pi$.
We say that an algorithm $(\epsilon, \delta)$-learns a convex combination of the
$d$ scoring rules $\score_1, \dots, \score_d$ if for any distribution $\dist$,
with probability at least $1-\delta$ over the draw of a sample
$\left\{\Pi_1, \dots, \Pi_m\right\} \sim \dist^m$, the algorithm returns
a convex combination $\score = \hat{\mu}_1\score_1+ \cdots +
\hat{\mu}_d\score_d$ such that $\E_{\Pi \sim \dist}\bigl[\cost(\Pi,
\score)\bigr] - \E_{\Pi \sim \dist}\bigl[\cost(\Pi, \score^*)\bigr] \leq
\epsilon$, where $\score^*$ is the convex combination of $\score_1, \dots,
\score_d$ with minimal expected cost. In this work, we prove that only a small
number of samples is sufficient to ensure $(\epsilon, \delta)$-learnability.

We assume that the problem instances in the support of $\dist$ are over $n$ $D$-ary variables, for some $n, D \in \N$.\footnote{A variable is $D$-ary if it can take on at most $D$ distinct values.}

Our results hold for cost functions that are \emph{tree-constant}, which means that for any problem instance $\Pi$, so long as the scoring rules $\score_1$ and $\score_2$ result in the same search tree, $\cost(\Pi, \score_1) = \cost(\Pi, \score_2)$.
For example, the size of the search tree is tree-constant.

\subsection{Path-wise scoring rules}
We now slightly tweak the definition of \emph{score-based variable selection policies} and \emph{path-wise scoring rules} so that they apply to tree search more generally. The only difference is that a scoring rule will now be defined in terms of a partial solution, rather than a MILP.

\begin{definition}[Score-based variable selection policy]
Let $\score$ be a function that takes as input a partial search tree $\tree$, a partial solution $\vec{y}$ contained in a leaf of $\tree$, and an index $i$ and returns a real value ($\score(\tree, Q, i) \in \R$). For a partial solution $\vec{y}$ contained in a leaf of a tree $\tree$, let $N_{\tree, \vec{y}}$ be the set of variables that have not yet been branched on along the path from the root of $\tree$ to the leaf. A score-based variable selection policy selects the variable $\argmax_{x_j \in N_{\tree, \vec{y}}} \{\score(\tree, \vec{y}, j)\}$ to branch on at the node $Q$.
\end{definition}

\begin{definition}[Path-wise scoring rule]
Suppose $\vec{y}$ is a partial solution contained in the node of a search tree $\tree$. We say that $\score(\tree, \vec{y}, i)$ is a path-wise scoring rule if the value of $\score(\tree, \vec{y}, i)$ depends only on the node $\node$, the variable $x_i$, and the path from the root of $\tree$ to the node containing $\vec{y}$, which we denote as $\tree_{\vec{y}}$. Specifically, $\score(\tree, \vec{y}, i) = \score(\tree_{\vec{y}}, \vec{y}, i)$.
\end{definition}

The following lemma parallels Lemma~\ref{lem:induction}.

\begin{lemma}\label{lem:TS_induction}
Let $\cost$ be a tree-constant cost function, let $\score_1$ and $\score_2$ be two path-wise scoring rules, and let $\Pi$ be an arbitrary problem instance over $n$ $D$-ary variables. There are $T \leq D^{n(n-1)/2}n^n$ intervals $I_1, \dots, I_T$ partitioning $[0,1]$ where for any interval $I_j$, across all $\mu \in I_j$, the scoring rule $\mu\score_1 + (1-\mu)\score_2$ results in the same search tree.
\end{lemma}

\begin{proof}
We prove this lemma first by considering the actions of an alternative algorithm $TS'$ which runs exactly like Algorithm~\ref{alg:TS} except it does not use the function $\glo$, but rather skips Steps \ref{step:if} and \ref{step:then} of Algorithm~\ref{alg:TS}. We then relate the behavior of $TS'$ to the behavior of Algorithm~\ref{alg:TS} to prove the lemma.

First, we prove the following bound on the number of search trees $TS'$ will build on a given instance over the entire range of parameters. Note that this bound matches that in the lemma statement.

\begin{restatable}{claim}{claimIntervals}\label{claim:simple_TS_intervals}
There are $T \leq D^{n(n-1)/2}n^n$ intervals $I_1, \dots, I_T$ partitioning $[0,1]$ where for any interval $I_j$, the search tree $TS'$ builds using the scoring rule $\mu\score_1 + (1-\mu)\score_2$ is invariant across all $\mu \in I_j$.
\end{restatable}

\begin{proof}We prove this claim by induction.

\bigskip
\noindent\textbf{Inductive hypothesis.} For $i \in \{1,\dots, n\}$, there are $T \leq D^{i(i-1)/2}n^i$ intervals $I_1, \dots, I_T$ partitioning $[0,1]$ where for any interval $I_j$ and any two parameters $\mu, \mu' \in I_j$, if $\tree$ and $\tree'$ are the trees $TS'$ builds using the scoring rules $\mu\score_1 + (1-\mu)\score_2$ and $\mu'\score_1 + (1-\mu')\score_2$, respectively, then $\tree[i] = \tree'[i]$.

\bigskip

\noindent\textbf{Base case.} Before branching on any variables, the seasrch tree $\tree_0$ consists of a single root node containing the empty partial solution $\vec{y} = (\top, \dots, \top)$. Given a parameter $\mu$, $TS'$ will branch on variable $x_k$ so long as \[k = \argmax_{\ell \in [n]} \left\{\mu\score_1(\tree_0, \vec{y}, \ell) + (1-\mu)\score_2(\tree_0, \vec{y}, \ell)\right\}.\]
Since $\mu\score_1(\tree_0, \vec{y}, \ell) + (1-\mu)\score_2(\tree_0, \vec{y}, \ell)$ is a linear function of $\mu$ for each $\ell \in [n]$, we know that for any $k \in [n]$, there is at most one interval $I$ of the parameter space $[0,1]$
where $k = \argmax_{\ell \in [n]} \left\{\mu\score_1 + (1-\mu)\score_2\right\}$. Thus, there are $T \leq n = D^{1\cdot(1-1)/2}n^1$ intervals $I_1, \dots, I_T$ partitioning $[0,1]$ where for any interval $I_j$, $TS'$ branches on the same variable at the root node using the scoring rule $\mu\score_1 + (1-\mu)\score_2$ across all $\mu \in I_j$.

\bigskip

\noindent\textbf{Inductive step.} Let $i \in \{2, \dots, n\}$ be arbitrary. From the inductive hypothesis, we know that there are $T \leq D^{(i-2)(i-1)/2}n^{i-1}$ intervals $I_1, \dots, I_T$ partitioning $[0,1]$ where for any interval $I_j$ and any two parameters $\mu, \mu' \in I_j$, if $\tree$ and $\tree'$ are the trees $TS'$ builds using the scoring rules $\mu\score_1 + (1-\mu)\score_2$ and $\mu'\score_1 + (1-\mu')\score_2$, respectively, then $\tree[i-1] = \tree'[i-1]$. Consider an arbitrary node containing a partial solution $\vec{y}$ in $\tree[i-1]$ (or equivalently, $\tree'[i-1]$) at depth $i-1$. If $\loc(\Pi, \vec{y}) = \fathom$, then it will be fathomed no matter which parameter $\mu\in I_j$ the algorithm $TS'$ uses. Otherwise, for all $\mu \in I_j$, let $\tree_{\mu}$ be the state of the search tree $TS'$ builds using the scoring rule $\mu\score_1 + (1-\mu)\score_2$ at the point when it branches on the node containing $\vec{y}$. By the inductive hypothesis, we know that across all $\mu \in I_j$, the path from the root to this node in $\tree_{\mu}$ is invariant, and we refer to this path as $\tree_{\vec{y}}$. Given a parameter $\mu \in I_j$, the variable $x_k$ will be branched on at this node so long as $k= \argmax_{\ell}\left\{\mu\score_1(\tree_{\mu},\vec{y},\ell) + (1-\mu)\score_2(\tree_{\mu},\vec{y},\ell)\right\},$ or equivalently, so long as $k= \argmax_{\ell}\left\{\mu\score_1(\tree_{\vec{y}},\vec{y},\ell) + (1-\mu)\score_2(\tree_{\vec{y}},\vec{y},\ell)\right\}$. In other words, the decision of which variable to branch on is determined by a convex combination of the constant values $\score_1(\tree_{\vec{y}},\vec{y},\ell)$ and $\score_2(\tree_{\vec{y}},\vec{y},\ell)$ no matter which parameter $\mu \in I_j$ the algorithm $TS'$ uses. Here, we critically use the fact that the scoring rule is path-wise.

Since $\mu\score_1(\tree_{\vec{y}},\vec{y},\ell) + (1-\mu)\score_2(\tree_{\vec{y}},\vec{y},\ell)$ is a linear function of $\mu$ for all $\ell$, there are at most $n$ intervals subdividing the interval $I_j$ such that the variable branched on at the node containing $\vec{y}$ is fixed. Moreover, there are at most $D^{i-1}$ nodes at depth $i-1$, and each node similarly contributes a subpartition of $I_j$ of size $n$. If we merge all $D^{i-1}$ partitions, we have $T' \leq D^{i-1}(n-1)+1$ intervals $I_1', \dots, I_{T'}'$ partitioning $I_j$ where for any interval $I_p'$ and any two parameters $\mu, \mu' \in I_p'$, if $\tree$ and $\tree'$ are the trees $TS'$ builds using the scoring rules $\mu\score_1 + (1-\mu)\score_2$ and $\mu'\score_1 + (1-\mu')\score_2$, respectively, then $\tree[i] = \tree'[i]$. We can similarly subdivide each interval $I_1, \dots, I_T$ for a total of \[\bar{T} \leq D^{(i-1)(i-2)/2}n^{i-1}\left(D^{i-1}(n-1)+1\right) \leq D^{(i-1)(i-2)/2}n^{i-1}\left(D^{i-1}n\right) = D^{i(i-1)/2}n^{i}\] intervals $\bar{I}_1, \dots, \bar{I}_{\bar{T}}$ partitioning $[0,1]$ such that for any interval $\bar{I}_t$, across all $\mu \in \bar{I}_t$ and any two parameters $\mu, \mu' \in \bar{I}_t$, if $\tree$ and $\tree'$ are the trees $TS'$ builds using the scoring rules $\mu\score_1 + (1-\mu)\score_2$ and $\mu'\score_1 + (1-\mu')\score_2$, respectively, then $\tree[i] = \tree'[i]$.
\end{proof}

Next, we explicitly relate the behavior of Algorithm~\ref{alg:TS} to $TS'$, proving that the search tree Algorithm~\ref{alg:TS} builds is a rooted subtree of the search tree $TS'$ builds.

\begin{claim}\label{claim:simple_TS_subtree}
Given a parameter $\mu \in [0,1]$, let $\tree$ and $\tree'$ be the trees Algorithm~\ref{alg:TS} and $TS'$ build, respectively, using the scoring rule $\mu\score_1 + (1-\mu)\score_2$. For an arbitrary node of $\tree$, let $\vec{y}$ be the partial solution contained in that node and let $\tree_{\vec{y}}$ be the path from the root of $\tree$ to the node. Then $\tree_{\vec{y}}$ is a rooted subtree of $\tree'$.
\end{claim}

\begin{proof}[Proof of Claim~\ref{claim:simple_TS_subtree}]
Note that the path $\tree_{\vec{y}}$ can be labeled by a sequence of indices from $\{1, \dots, D\}$ and a sequence of variables from $\{x_1, \dots, x_n\}$ describing which variable is branched on and which value it takes on along the path $\tree_{\vec{y}}$. Let $((j_1, x_{i_1}), \dots, (j_t, x_{i_t}))$ be this sequence of labels, where $t$ is the number of edges in $\tree_{\vec{y}}$. We can similarly label every edge in $\tree'$. We claim that there exists a path beginning at the root of $\tree'$ with the labels $((j_1, x_{i_1}), \dots, (j_t, x_{i_t}))$.

For a contradiction, suppose no such path exists. Let $(j_{\tau}, x_{i_{\tau}})$ be the earliest label in the sequence $((j_1, x_{i_1}), \dots, (j_t, x_{i_t}))$ where there is a path beginning at the root of $\tree'$ with the labels $((j_1, x_{i_1}), \dots, (j_{\tau-1}, x_{i_{\tau - 1}}))$, but there is no way to continue the path using an edge labeled $(j_{\tau}, x_{i_{\tau}})$. There are exactly two reasons why this could be the case:
\begin{enumerate}
\item The node at the end of the path with labels $((j_1, x_{i_1}), \dots, (j_{\tau-1}, x_{i_{\tau - 1}}))$ was fathomed by $TS'$.
\item The algorithm $TS'$ branched on a variable other than $x_{i_{\tau}}$ at the end of the path labeled $((j_1, x_{i_1}), \dots, (j_{\tau-1}, x_{i_{\tau - 1}}))$.
\end{enumerate}

Let $\vec{y}'$ be the partial solution contained in the node at the end of the path with labels \[((j_1, x_{i_1}), \dots, (j_{\tau-1}, x_{i_{\tau - 1}})).\] We refer to this path as $\tree_{\vec{y}'}$.
In the first case, since $TS'$ only fathoms the node containing the partial solution $\vec{y}'$ if $\loc(\Pi, \vec{y}') = \fathom$, we know that Algorithm~\ref{alg:TS} will also fathom this node. However, this is not the case since Algorithm~\ref{alg:TS} next branches on the variable $x_{i_\tau}$.

The second case is also not possible since the scoring rules are both path-wise. Let $\bar{\tree}$ (respectively, $\bar{\tree}'$) be the state of the search tree Algorithm~\ref{alg:TS} (respectively, $A$') has built at the point it branches on the node containing $\vec{y}'$. We know that $\tree_{\vec{y}'}$ is the path from the root this node in both of the trees $\bar{\tree}$ and $\bar{\tree}'$. Therefore, for all variables $x_k$, $\mu \score_1(\bar{\tree}, \vec{y}', k) + (1-\mu)\score_2(\bar{\tree}, \vec{y}', k) = \mu \score_1(\tree_{\vec{y}'}, \vec{y}', k) + (1-\mu)\score_2(\tree_{\vec{y}'}, \vec{y}', k) = \mu \score_1(\bar{\tree}', \vec{y}', k) + (1-\mu)\score_2(\bar{\tree}', \vec{y}', k)$. This means that Algorithm~\ref{alg:TS} and $TS'$ will choose the same variable to branch on at the node containing the partial solution $\vec{y}'$.

 Therefore, we have reached a contradiction, so the claim holds.
\end{proof}

Next, we use Claims~\ref{claim:simple_TS_intervals} and \ref{claim:simple_TS_subtree} to prove Lemma~\ref{lem:TS_induction}. Let $I_1, \dots, I_T$ be one of the interval guaranteed to exist by Claim~\ref{claim:simple_TS_intervals} and let $I_t$ be an arbitrary one of the intervals. Let $\mu'$ and $\mu''$ be two arbitrary parameters from $I_t$. We will prove that the scoring rules $\mu'\score_1 + (1-\mu')\score_2$ and $\mu''\score_1 + (1-\mu'')\score_2$ result in the same search tree. For a contradiction, suppose that this is not the case. 
Consider the first iteration where of Algorithm~\ref{alg:TS} using the scoring rule $\mu'\score_1 + (1-\mu')\score_2$ differs from Algorithm~\ref{alg:TS} using the scoring rule $\mu''\score_1 + (1-\mu'')\score_2$, where
  an iteration corresponds to lines~\ref{step:begin_while_TS} through \ref{step:end_while_TS} of Algorithm~\ref{alg:TS}. Up until this iteration, Algorithm~\ref{alg:TS} has built the same partial search tree $\tree$. Since the node selection policy does not depend on $\mu'$ or $\mu''$, Algorithm~\ref{alg:TS} will choose the same leaf of the search tree to branch on no matter which scoring rule it uses. Let $\vec{y}$ be the partial solution contained in this leaf.

Suppose Algorithm~\ref{alg:TS} chooses different variables to branch on depending on the scoring rule. Let $\tree_{\vec{y}}$ be the path from the root of $\tree$ to the node containing the partial solution $\vec{y}$. By Claim~\ref{claim:simple_TS_intervals}, we know that the algorithm $TS'$ builds the same search tree using the two scoring rules. Let $\bar{\tree}'$ (respectively, $\bar{\tree}''$) be the state of the search tree $TS'$ has built using the scoring rule $\mu'\score_1 + (1-\mu')\score_2$ (respectively, $\mu''\score_1 + (1-\mu'')\score_2$) by the time it branches on the node containing the partial solution $\vec{y}$. By Claims~\ref{claim:simple_TS_intervals} and \ref{claim:simple_TS_subtree}, we know that $\tree_{\vec{y}}$ is the path of from the root to the node containing the partial solution $\vec{y}$ of both $\bar{\tree}'$ and $\bar{\tree}''$. By Claim~\ref{claim:simple_TS_intervals},
we know that $TS'$ will branch on the same variable $x_i$ at this node in both the trees $\bar{\tree}'$ and $\bar{\tree}''$, so $i = \argmax_{j} \left\{\mu'\score_1(\bar{\tree}', \vec{y}, j)  + (1-\mu')\score_2(\bar{\tree}', \vec{y}, j)\right\}$, or equivalently, \begin{equation}i = \argmax_{j} \left\{\mu'\score_1(\tree_{\vec{y}}, \vec{y}, j)  + (1-\mu')\score_2(\tree_{\vec{y}}, \vec{y}, j)\right\},\label{eq:one_prime_TS}\end{equation} and $i = \argmax_{j} \left\{\mu''\score_1(\bar{\tree}'', \vec{y}, j)  + (1-\mu'')\score_2(\bar{\tree}'', \vec{y}, j)\right\}$, or equivalently, \begin{equation}i = \argmax_{j} \left\{\mu''\score_1(\tree_{\vec{y}}, \vec{y}, j)  + (1-\mu'')\score_2(\tree_{\vec{y}}, \vec{y}, j)\right\}.\label{eq:two_primes_TS}\end{equation} Returning to the search tree $\tree$ that Algorithm~\ref{alg:TS} is building, Equation~\eqref{eq:one_prime_TS} implies that \[i = \argmax_{j} \left\{\mu'\score_1(\tree, \vec{y}, j)  + (1-\mu')\score_2(\tree, \vec{y}, j)\right\}\] and Equation~\eqref{eq:two_primes_TS} implies that $i = \argmax_{j} \left\{\mu''\score_1(\tree, \vec{y}, j)  + (1-\mu'')\score_2(\tree, \vec{y}, j)\right\}$.
Therefore, Algorithm~\ref{alg:TS} will branch on $x_i$ at the node containing the partial solution $\vec{y}$ no matter which scoring rule it uses.

Finally, since $\loc$ and $\glo$ do not depend on the parameter $\mu$, whether or not Algorithm~\ref{alg:TS} fathoms any of the nodes in Steps~\ref{step:fathom_loop} through \ref{step:then} does not depend on $\mu'$ or $\mu''$.

We have reached a contradiction by showing that the two iterations of Algorithm~\ref{alg:TS} are identical. Therefore, the lemma holds.
\end{proof}

\subsubsection{General scoring rules}\label{app:general_TS}

In this section, we provide generalization guarantees that apply to learning convex combinations of any set of scoring rules, as in Section~\ref{sec:general_scoring}. The following lemma corresponds to Lemma~\ref{lem:TS_induction} for this setting.

\begin{lemma}\label{lem:induction_general_TS}
Let $\cost$ be a tree-constant cost function, let $\score_1, \dots, \score_d$ be $d$ arbitrary scoring rules, and let $\Pi$ be an arbitrary problem instance over $n$ $D$-ary variables. Suppose we limit Algorithm~\ref{alg:TS} to producing search trees of
  size $\bar{\kappa}$. There is a set $\mathcal{H}$ of at most $n^{2(\bar{\kappa} + 1)}$ hyperplanes such that for any connected component $R$ of $[0,1]^d \setminus \mathcal{H}$, the search tree Algorithm~\ref{alg:TS} builds using the scoring rule $\mu_1\score_1 + \cdots + \mu_d\score_d$ is invariant across all $(\mu_1, \dots, \mu_d) \in R$.
\end{lemma}

\begin{proof}
  The proof has two steps. In Claim~\ref{claim:sequences_TS}, we show that there are at most $n^{\bar{\kappa}}$
  different search trees that Algorithm~\ref{alg:TS} might produce for the instance $\Pi$ as we
  vary the mixing parameter vector $(\mu_1, \dots, \mu_d)$. In Claim~\ref{claim:hyperplanes_TS}, for each of the possible search trees
  $\mathcal{T}$ that might be produced, we show that the set of parameter values
  $(\mu_1, \dots, \mu_d)$ which give rise to that tree lie in the intersection of $n^{\bar{\kappa} + 2}$ halfspaces. These facts together prove
  the lemma.
  
  \begin{claim}\label{claim:sequences_TS}
  There are only $n^{\bar{\kappa}}$ different search trees that can
  be achieved by varying the parameter vector $(\mu_1, \dots, \mu_d)$.
  \end{claim}

\begin{proof}[Proof of Claim~\ref{claim:sequences_TS}]
Fix any $d$ mixing parameters $(\mu_1, \dots, \mu_d)$ and
  let $v_1, \dots, v_{\bar{\kappa}} \in [n]$ be the sequence of branching variables
  chosen by Algorithm~\ref{alg:TS} run with scoring rule $\mu_1\score_1 + \cdots + \mu_d\score_d$, ignoring which node of
  the tree each variable was chosen for. That is, $v_1$ is the variable branched
  on at the root, $v_2$ is the variable branched on at the next unfathomed node
  chosen by the node selection policy, and so on. If Algorithm~\ref{alg:TS} with scoring rule
  $\mu_1\score_1 + \cdots + \mu_d\score_d$ produces a tree of size $k < \bar{\kappa}$, then define $v_t = 1$ for
  all $t \geq k$ (we are just padding the sequence $v_1, v_2, \dots$ so that it
  has length $\bar{\kappa}$). We will show that whenever two sets of mixing parameters $(\mu_1, \dots, \mu_d)$
  and $(\mu_1', \dots, \mu_d')$ give rise to the same sequence of branching variable selections,
  they in fact produce identical search trees. This will imply that the number
  of distinct trees that can be produced by Algorithm~\ref{alg:TS} with scoring rules of the
  form $\mu_1\score_1 + \cdots + \mu_d\score_d$ is at most $n^{\bar{\kappa}}$, since there are only $n^{\bar{\kappa}}$
  distinct sequences of $\bar{\kappa}$ variables $v_1, \dots, v_{\bar{\kappa}} \in [n]$.

  Let $(\mu_1, \dots, \mu_d)$
  and $(\mu_1', \dots, \mu_d')$ be two sets of mixing parameters, and suppose running Algorithm~\ref{alg:TS} with
  $\mu_1\score_1 + \cdots + \mu_d\score_d$ and $\mu_1'\score_1 + \cdots + \mu_d'\score_d$ both results in the sequence of branching
  variable decisions being $v_1$, \dots, $v_{\bar{\kappa}}$. We prove that the resulting
  search trees are identical by induction on the iterations of the algorithm, where
  an iteration corresponds to lines~\ref{step:begin_while_TS} through \ref{step:end_while_TS} of Algorithm~\ref{alg:TS}. Our base case is before the first iteration when the two trees are trivially
  equal, since they both contain just the root node. Now suppose that up until
  the beginning of iteration $t$ the two trees were identical. Since the two trees
  are identical, the node selection policy will choose the same node to branch
  on in both cases. In both trees, the
  algorithm will choose the same variable to branch on, since the sequence of
  branching variable choices $v_1, \dots, v_{\bar{\kappa}}$ is shared. Finally, if any of the children are fathomed, they will be
  fathomed in both trees, since the they are identical. It follows that
  all steps of Algorithm~\ref{alg:TS} maintain equality between the two trees, and the claim
  follows. Also, whenever the sequence of branching variables differ,
  then the search tree produced will not be the same. In particular, on the
  first iteration where the two sequences disagree, the tree built so far will be
  identical up to that point, but the next variable branched on will be
  different, leading to different trees.
\end{proof}

  Next, we argue that for any given search tree $\mathcal{T}$ produced by
  Algorithm~\ref{alg:TS}, the set of mixing parameters $(\mu_1, \dots, \mu_d)$ giving rise to $\mathcal{T}$ is defined by the intersection of $n^{\bar{\kappa} + 2}$ halfspaces.

\begin{claim}\label{claim:hyperplanes_TS}
For a given search tree $\tree$, there are at most $n^{\bar{\kappa} + 2}$ halfspaces such that Algorithm~\ref{alg:TS} using the scoring rule $\mu_1\score_1 + \cdots + \mu_d\score_d$ builds the tree $\tree$ if and only if $(\mu_1, \dots, \mu_d)$ lies in the intersection of those halfspaces.
\end{claim}  
  
  \begin{proof}[Proof of Claim~\ref{claim:hyperplanes_TS}]
  Let $v_1, \dots, v_{\bar{\kappa}}$ be the sequence of branching variable
  choices that gives rise to tree $\mathcal{T}$. We will prove the claim by
  induction on iterations completed by Algorithm~\ref{alg:TS}. Let $\tree_t$ be the state of Algorithm~\ref{alg:TS} after $t$ iterations.
  
  \medskip
\textbf{Induction hypothesis.} For a given index $t \in [\bar{\kappa}]$, there are at most $tn^2$ halfspaces such that Algorithm~\ref{alg:TS} using the scoring rule $\mu_1\score_1 + \cdots + \mu_d\score_d$ builds the partial tree $\tree_t$ after $t$ iterations if and only if $(\mu_1, \dots, \mu_d)$ lies in the intersection of those halfspaces.

\medskip
\textbf{Base case.}
    In the base case, before
  the first iteration, the set of parameters that will produce the partial search
  tree consisting of just the root is the entire set of parameters, which vacuously is the intersection of zero hyperplanes.
  
  \medskip\textbf{Inductive step.}
  For the
  inductive step, let $t < \bar{\kappa}$ be an arbitrary tree size. By the inductive
  hypothesis, we know that there exists a set $\mathcal{B}$ of at most $tn^2$ halfspaces such that Algorithm~\ref{alg:TS} using the scoring rule $\mu_1\score_1 + \cdots + \mu_d\score_d$ builds the partial tree $\tree_t$ after $t$ iterations if and only if $(\mu_1, \dots, \mu_d)$ lies in the intersection of those halfspaces. Let $\vec{y}$ be the partial solution contained in the next node that Algorithm~\ref{alg:TS} will
  branch on given $\tree_t$. We know that Algorithm~\ref{alg:TS} will choose to branch on
  variable $v_{t+1}$ at node $\node$ if and only if
  \begin{align*}
  &\mu_1\score_1(\tree, \vec{y},v_{t+1}) + \cdots + \mu_d\score_d(\tree, \vec{y},v_{t+1})\\
  > &\max_{v' \neq v_{t+1}}\left\{\mu_1\score_1(\tree, \vec{y},v') + \cdots + \mu_d\score_d(\tree, \vec{y},v')\right\}
  .\end{align*}
  Since these functions are linear in $(\mu_1, \dots, \mu_d)$, there are at most $n^2$ halfspaces defining the region where $v_{t + 1} = \argmax\left\{\mu_1\score_1(\tree, \vec{y},v') + \cdots + \mu_d\score_d(\tree, \vec{y},v')\right\}$. Let $\mathcal{B}'$ be this set of halfspaces. Algorithm~\ref{alg:TS} using the scoring rule $\mu_1\score_1 + \cdots + \mu_d\score_d$ builds the partial tree $\tree_t$ after $t$ iterations if and only if $(\mu_1, \dots, \mu_d)$ lies in the intersection of the $(t+1)n^2$ halfspaces in the set $\mathcal{B} \cup \mathcal{B}'$.
  \end{proof}
\end{proof}

\begin{theorem}\label{thm:MILP_WCpdim_arbitrary}
Let $\cost$ be a tree-constant cost function and let $\score_1, \dots, \score_d$ be $d$ arbitrary scoring rules. Suppose we limit Algorithm~\ref{alg:TS} to producing search trees of
  size $\bar{\kappa}$. Let $\mathcal{C}$ be the set of functions $\left\{\cost\left(\cdot, \mu\score_1 + \cdots + \mu_d\score_d\right) : (\mu_1, \dots, \mu_d) \in [0,1]^d\right\}$. Then $\pdim(\mathcal{C}) = O\left(d\bar{\kappa}\log n + d\log d\right).$
\end{theorem}

\begin{proof} Suppose that $\pdim(\mathcal{C}) = m$ and let $\sample = \left\{\Pi_1, \dots, \Pi_m\right\}$ be a shatterable set of problem instances. We know there exists a set of targets $r_1, \dots, r_m \in \R$ that witness the shattering of $\sample$ by $\mathcal{C}$. This means that for every $\sample' \subseteq \sample$, there exists a parameter vector $\left(\mu_{1,\sample'}, \dots, \mu_{d, \sample'}\right)$ such that if $\Pi_i \in \sample'$, then $\cost\left(\Pi_i, \mu_{1,\sample'}\score_1 + \cdots + \mu_{d, \sample'}\score_d\right) \leq r_i$. Otherwise \[\cost\left(\Pi_i, \mu_{1,\sample'}\score_1 + \cdots + \mu_{d, \sample'}\score_d\right) > r_i.\] Let $M = \left\{\left(\mu_{1,\sample'}, \dots, \mu_{d, \sample'}\right) : \sample' \subseteq \sample\right\}$. We will prove that  $|M| = O(dm^dn^{2d(\bar{\kappa} + 1)})$, and since $2^m = |M|$, this means that $\pdim(\mathcal{C}) = m = O\left(d\bar{\kappa}\log n + d\log d\right)$ (see Lemma~\ref{lem:log_ineq} in Appendix~\ref{app:theory}).

To prove that $|M| \leq mn^{\bar{\kappa}} + 1$, we rely on Lemma~\ref{lem:induction_general_TS}, which tells us that for any problem instance $\Pi$, there is a set $\mathcal{H}$ of at most $T \leq n^{2(\bar{\kappa} + 1)}$ hyperplanes such that for any connected component $R$ of $[0,1]^d \setminus \mathcal{H}$, the search tree Algorithm~\ref{alg:TS} builds using the scoring rule $\mu_1\score_1 + \cdots + \mu_d\score_d$ is invariant across all $(\mu_1, \dots, \mu_d) \in R$. If we merge all $T$ hyperplanes for all samples in $\sample$, we are left with a set $\mathcal{H}'$ of $T' \leq mn^{2(\bar{\kappa} + 1)}$ hyperplanes where for any connected component $R$ of $[0,1]^d \setminus \mathcal{H}'$ and any $\Pi_i \in \sample$, the search tree Algorithm~\ref{alg:TS} builds using the scoring rule $\mu_1\score_1 + \cdots + \mu_d\score_d$ given as input $\Pi_i$ is invariant across all $(\mu_1, \dots, \mu_d) \in R$. Therefore, at most one element of $M$ can come from each connected component, of which there are $O(d|\mathcal{H}'|^d) = O(dm^dn^{2d(\bar{\kappa} + 1)})$. Therefore, $|M| = O(dm^dn^{2d(\bar{\kappa} + 1)})$, as claimed.
\end{proof}

Naturally, if the function $\cost$ measures the size of the search tree Algorithm~\ref{alg:TS} returns capped at some value $\kappa$ as described in Section~\ref{sec:statement}, we obtain the following corollary.

\begin{cor}
Let $\cost$ be a tree-constant cost function, let $\score_1, \dots, \score_d$ be $d$ arbitrary scoring rules, and let $\kappa \in \N$ be an arbitrary tree size bound. Suppose that for any problem instance $\Pi$, $\cost(\Pi, \mu\score_1 + \cdots + \mu_d\score_d)$ equals the minimum of the following two values: 1) The number of nodes in the search tree Algorithm~\ref{alg:TS} builds using the scoring rule $\mu\score_1 + \cdots + \mu_d\score_d$ on input $\Pi$; and 2) $\kappa$. For any distribution $\dist$ over problem instances $\Pi$ with at most $n$ variables, with probability at least $1-\delta$ over the drawn $\{\Pi^{(1)}, \dots, \Pi^{(m)}\} \sim \dist^m,$ for any $\mu \in [0,1]$, \begin{align*}&\left|\E_{\Pi \sim \dist} [\cost(\Pi, \mu\score_1 + \cdots + \mu_d\score_d)] - \frac{1}{m} \sum_{i = 1}^m \cost(\Pi^{(i)}, \mu\score_1 + \cdots + \mu_d\score_d)\right|\\
= \text{ }&O\left(\sqrt{\frac{d\kappa^2(\kappa\log n + \log d)}{m}} + \kappa \sqrt{\frac{\ln(1/\delta)}{m}}\right).\end{align*}
\end{cor}

\end{document}